\documentclass[11pt,letterpaper]{article}
\usepackage[numbers]{natbib}
\usepackage[letterpaper,top=1in,bottom=1in,left=1in,right=1in,marginparwidth=1in]{geometry}
\usepackage{bm}
\usepackage{amsmath,amsfonts,amssymb,amsthm}

\usepackage{graphicx}
\usepackage{xspace}
\usepackage[colorlinks=true, allcolors=blue]{hyperref}
\hypersetup{hypertexnames=false}
\usepackage{bbm}
\usepackage{outlines}
\usepackage[noend]{algpseudocode}
\usepackage{algorithm}
\usepackage{cleveref}
\usepackage{xcolor}
\usepackage{enumitem}
\usepackage{subcaption}
\usepackage{microtype}
\usepackage{thmtools}
\usepackage{thm-restate}
\usepackage{soul}
\allowdisplaybreaks
\usepackage{multirow}
\usepackage{array}
\usepackage{wrapfig}
\usepackage{adjustbox}
\usepackage{appendix}
\usepackage[utf8]{inputenc}
\usepackage[T1]{fontenc}
\usepackage{url}
\usepackage{booktabs}
\usepackage{nicefrac}
\usepackage{float}
\setlist[enumerate]{itemsep=0.25\baselineskip, parsep=0pt}
\usepackage{tikz}
\usetikzlibrary{shapes.geometric, arrows}

 \theoremstyle{plain}
 \newtheorem{theorem}{Theorem}
 \newtheorem*{theorem*}{Theorem}
 
 \newtheorem{lemma}{Lemma}
 \newtheorem{proposition}{Proposition}

 \theoremstyle{definition}
 \newtheorem{definition}{Definition}
\newtheorem{assumption}{Assumption}
 
\newtheorem{property}{Property}

\theoremstyle{remark}
\newtheorem{remark}{\textit{Remark}}
\newtheorem*{remark*}{Remark}

\providecommand{\hcref}[1]{%
  \texorpdfstring{\protect\NoHyper\Cref{#1}\protect\endNoHyper}{#1}%
}

\newcommand{\norm}[1]{\left\lVert#1\right\rVert}
\newcommand{\normplain}[1]{\lVert#1\rVert}
\newcommand{\normbig}[1]{\big\lVert#1\big\rVert}
\newcommand{\normBig}[1]{\Big\lVert#1\Big\rVert}
\newcommand{\E}[1]{\mathbb{E}\left[#1\right]}
\newcommand{\Eplain}[1]{\mathbb{E}[#1]}
\newcommand{\Ebig}[1]{\mathbb{E}\big[#1\big]}
\newcommand{\EBig}[1]{\mathbb{E}\Big[#1\Big]}

\newcommand{\given}{\,\middle|\,}
\newcommand{\givenplain}{\,|\,}
\newcommand{\givenbig}{\,\big|\,}
\newcommand{\givenBig}{\,\Big|\,}

\newcommand{\Var}[1]{\text{Var}{\left[#1\right]}}

\newcommand{\Probbig}[1]{\mathbb{P}\big[#1\big]}
\newcommand{\ProbBig}[1]{\mathbb{P}\Big[#1\Big]}
\newcommand{\Prob}[1]{\mathbb{P}\left[#1\right]}
\newcommand{\R}{\mathbb{R}}

\newcommand{\indibrac}[1]{\mathbbm{1}\!\left\{#1\right\}}

\newcommand{\abs}[1]{\left\lvert#1\right\rvert}
\newcommand{\absplain}[1]{\lvert#1\rvert}
\newcommand{\absbig}[1]{\big\lvert#1\big\rvert}
\newcommand{\absBig}[1]{\Big\lvert#1\Big\rvert}

\newcommand{\ceil}[1]{\left\lceil#1\right\rceil}
\newcommand{\floor}[1]{\left\lfloor#1\right\rfloor}

\newcommand{\ve}[1]{\bm{#1}}

\newcommand{\vone}{\mathbbm{1}}

\newcommand{\veS}{\ve{S}}

\newcommand{\veA}{\ve{A}}
\newcommand{\sspa}{\mathbb{S}}
\newcommand{\aspa}{\mathbb{A}}
\newcommand{\sumN}{\sum_{i\in[N]}}
\newcommand{\sumsa}{\sum_{s\in\sspa, a\in\aspa}}
\newcommand{\sums}{\sum_{s\in\sspa}}

\newcommand{\rel}{\textup{rel}} %
\newcommand{\rmax}{r_{\max}}

\newcommand{\ravg}{R}
\newcommand{\rsysn}{\ravg(\pi, \veS_0)}
\newcommand{\rliminf}{\ravg^{-}(\pi, \veS_0)}
\newcommand{\rlimsup}{\ravg^{+}(\pi, \veS_0)}
\newcommand{\ropt}{\ravg^*(N, \veS_0)}
\newcommand{\rrel}{\ravg^\rel}

\newcommand{\btotal}{\alpha N}
\newcommand{\bsub}{B}

\newcommand{\pibar}{{\sysbar{\pi}}}
\newcommand{\pibs}{{\pibar^*}}
\newcommand{\lppriority}{\textup{LP-Priority}\xspace}

\newcommand{\frcontrol}{\textup{$\pibs$-Fixed-Ratio Control}\xspace}

\newcommand{\syshat}[1]{\widehat{#1}}
\newcommand{\sysbar}[1]{\bar{#1}}

\newcommand{\simplex}{\Delta}
\newcommand{\threshbar}{\overline{\eta}}

\newcommand{\sempty}{S^{\emptyset}}

\newcommand{\costvec}{c_\pibs}

\newcommand{\md}{m_d}
\newcommand{\Md}[1]{N\md(x)}

\newcommand{\statdist}{\mu^*}

\newcommand{\wmat}{W}
\newcommand{\umat}{U}

\newcommand{\rhoFinal}{\rho_1}

\newcommand{\errfe}{\epsilon_N^{\text{fe}}}
\newcommand{\errtol}{\epsilon_N^{\text{rd}}}

\newcommand{\hw}{h_\wmat}

\newcommand{\hu}{h_\umat}

\newcommand{\lamw}{\lambda_\wmat}
\newcommand{\lamu}{\lambda_\umat}
\newcommand{\rhow}{\rho_{w}}
\newcommand{\rhou}{\rho_{u}}

\newcommand{\KhwOne}{K_{W,1}}
\newcommand{\KhwTwo}{K_{W,2}}
\newcommand{\Chw}{C_{W}}
\newcommand{\KhuOne}{K_{U,1}}
\newcommand{\KhuTwo}{K_{U,2}}
\newcommand{\Chu}{C_{U}}
\newcommand{\KnsOne}{K_{D,1}}
\newcommand{\KnsTwo}{K_{D,2}}
\newcommand{\Cns}{C_{D}}
\newcommand{\KscOne}{K_{C,1}}
\newcommand{\KscTwo}{K_{C,2}}

\newcommand{\Da}{D^{\pibs}}
\newcommand{\Db}{D^{\textup{LP}}}
\newcommand{\Dtemp}{\Db_{\textup{temp}}}

\newcommand{\Fullstate}{\Sigma}
\newcommand{\fullstate}{\sigma}
\newcommand{\fullssp}{\mathbb{X}}
\newcommand{\loworder}{O_N}
\newcommand{\Vds}{V_1} %
\newcommand{\Vll}{V_2} %
\newcommand{\rhat}{\hat{r}}
\newcommand{\gvec}{\bm{g}}

\newcommand{\KVds}{K_1}  %
\newcommand{\KVdst}{K_2}
\newcommand{\CVds}{C}  %
\newcommand{\Lds}{L_1}
\newcommand{\Lvll}{L_2}
\newcommand{\locapprox}{\Psi} %

\newcommand{\drftup}{u}
\newcommand{\uone}{u_1}  %
\newcommand{\utwo}{u_2}  %

\newcommand{\lamq}{\lambda_{Q}}
\newcommand{\sneu}{\tilde{s}}

\newcommand{\vlam}{f^*}
\newcommand{\qlam}{Q^*}
\newcommand{\rlam}{r_{\lambda^*}}
\newcommand{\rewardgap}{\epsilon_0}
\newcommand{\slk}{\delta}
\newcommand{\slkb}{\delta}

\newcommand{\disy}{d_{\textup{IC}}}
\newcommand{\gapus}{\epsilon_1} %
\newcommand{\gapusmax}{\epsilon_2}
\newcommand{\gapst}{\epsilon_3}
\newcommand{\radiusus}{\eta'}
\newcommand{\noiseortho}{K_{\textup{noise}}}

\newcommand{\subus}{+}%
\newcommand{\subst}{-}%
\newcommand{\Jus}{J_{\subus}}
\newcommand{\Jst}{J_{\subst}}
\newcommand{\Ius}{I_{\subus}}
\newcommand{\Ist}{I_{\subst}}

\newcommand{\Uus}{U_{\subus}}
\newcommand{\Ust}{U_{\subst}}
\newcommand{\spus}{\mathcal{V}_{\subus}}
\newcommand{\spst}{\mathcal{V}_{\subst}}
\newcommand{\hus}{h_{\subus}}
\newcommand{\hst}{h_{\subst}}
\newcommand{\lamus}{\lambda_{\subus}}
\newcommand{\lamst}{\lambda_{\subst}}
\newcommand{\Knorm}{K_{\textup{norm}}}
\newcommand{\goodevent}{\mathcal{E}}

\newcommand{\hlb}{h_{\textup{lb}}}
\newcommand{\ratiolb}{L_{\textup{lb}}}

\newcommand{\ximax}{\xi_{\max}}
\newcommand{\ximin}{\xi_{\min}}
\newcommand{\vonescaled}{\vone_{\textup{scaled}}}

\newcommand{\leftsub}{\text{L}}
\newcommand{\rightsub}{\text{R}}

\floatstyle{ruled}
\newfloat{subroutine}{t!}{log}
\floatname{subroutine}{Subroutine}
\Crefname{subroutine}{Subroutine}{Subroutines}
\crefname{subroutine}{subroutine}{subroutines}
\Crefname{assumption}{Assumption}{Assumptions}
\crefname{assumption}{assumption}{assumptions}

\usetikzlibrary{shapes.geometric, arrows}

\tikzstyle{lemma} = [rectangle, rounded corners, minimum width=3cm, minimum height=1cm, text centered, draw=black, fill=white]
\tikzstyle{theorem} = [ellipse, minimum width=3cm, minimum height=1cm, text centered, draw=black, fill=white]
\tikzstyle{arrow} = [thick,->,>=stealth]

\newcommand{\Rey}{\mathrm{Re}\,y}
\newcommand{\Imy}{\mathrm{Im}\,y}

\author{%
   Yige Hong$^1$ \quad Qiaomin Xie$^2$ \quad Yudong Chen$^3$ \quad  Weina Wang$^4$   \medskip\\
   $^1$H.\ Milton Stewart School of Industrial and Systems Engineering\\
   Georgia Institute of Technology\\
   $^2$Department of Industrial and Systems Engineering, University of Wisconsin--Madison\\
   $^3$Department of Computer Sciences, University of Wisconsin--Madison\\
   $^4$Computer Science Department, Carnegie Mellon University\smallskip\\
   $^1$\texttt{yhong320@gatech.edu}\\
   $^2$\texttt{qiaomin.xie@wisc.edu}\\
   $^3$\texttt{yudongchen@cs.wisc.edu}\\
   $^4$\texttt{weinaw@cs.cmu.edu}
}
\title{Achieving Exponential Asymptotic Optimality in Average-Reward Restless Bandits without Global Attractor Assumption}
\date{}

\newcommand{\paperabstract}{%
    We consider the infinite-horizon average-reward restless bandit problem.
    We propose a novel \emph{two-set policy} that maintains two dynamic subsets of arms:
    one subset of arms has a nearly optimal state distribution and takes actions according to an \lppriority routine; the other subset of arms is driven towards the optimal state distribution and gradually merged into the first subset.
    We show that our two-set policy is asymptotically optimal with an $O(\exp(-C N))$ optimality gap for an $N$-armed problem, under the assumptions of aperiodic-unichain, non-degeneracy, and local stability.
    Our policy is the first to achieve \emph{exponential asymptotic optimality} under the above set of easy-to-verify assumptions, whereas prior work either requires a \emph{global attractor} assumption or only achieves an $O(1/\sqrt{N})$ optimality gap.
    We further discuss obstacles in weakening the assumptions by demonstrating examples where exponential asymptotic optimality is not achievable when any of the three assumptions is violated.
    Notably, we prove a lower bound for a large class of locally unstable restless bandits, showing that local stability is particularly fundamental for exponential asymptotic optimality.
    Finally, we use simulations to demonstrate that the two-set policy outperforms previous policies on certain restless bandit problems and performs competitively overall.%
}

\begin{document}
\maketitle
\begin{abstract}
\paperabstract
\end{abstract}

\section{Introduction}\label{sec:intro}

The restless bandit (RB) \citep{Whi_88_rb} problem is a stochastic sequential decision-making problem consisting of multiple arms coupled by a constraint.
Each arm operates as a Markov decision process (MDP) with two possible actions: activating or keeping passive.
At each time step, the decision maker observes the states of the arms and chooses a fixed fraction of arms to activate, aiming to maximize the expected reward collected from the arms.
The RB problem has a long history and diverse applications. We refer readers to the recent survey paper \citep{Nin_23} for a comprehensive review of the literature.

In this paper, we study infinite-horizon RBs in discrete time, with the long-run average-reward objective.
Since computing the exact optimal policy for an RB problem is intractable when the number of arms $N$ is large \citep{PapTsi_99_pspace}, we focus on developing asymptotically optimal policies that can be efficiently computed.
Specifically, the optimality gap of a policy is defined as the difference between the average reward per arm under a policy and the optimal reward, and a policy is \emph{asymptotically optimal} if the optimality gap is $o(1)$ in the asymptotic limit $N\to\infty$.\footnote{In this paper, we use standard Bachmann–Landau notation and focus on the asymptotics with respect to $N$.}

The asymptotic optimality of average-reward RB in the large $N$ regime has been studied for several decades, starting with the seminal papers on the renowned Whittle index policy \citep{Whi_88_rb,WebWei_90}.
Since then, researchers have been weakening the assumptions for asymptotic optimality \citep{Ver_16_verloop,HonXieCheWan_23,HonXieCheWan_24,Yan_24_multichain} or improving the order of the optimality gaps \citep{GasGauYan_23_exponential,GasGauYan_23_whittles}.
A recent milestone is the \emph{exponential asymptotic optimality} established in \citep{GasGauYan_23_exponential,GasGauYan_23_whittles}, where a class of policies called LP-Priority is proved to achieve an optimality gap  $O(\exp(-CN))$ for a constant $C$, under some assumptions.
This exponential gap is particularly noteworthy because it ``beats'' the Central Limit Theorem (CLT), an intriguing theoretical property that is not a priori obvious to be possible.
In particular, at a high level, all asymptotic optimality results in the RB literature are based on the concentration of the empirical distribution of the arms' states around a certain optimal distribution. Therefore, an $O(1/\sqrt{N})$ bound owing to CLT was believed to be fundamental.
An important observation made in \citep{GasGauYan_23_exponential,GasGauYan_23_whittles} is that if the system has a certain notion of \emph{local linearity},  the optimality gap only depends on the distance between the optimal distribution and the \emph{expected} empirical state distribution, allowing one to go beyond CLT.

Despite the above intuition,
it remains unclear what is the fundamental mechanism leading to exponential asymptotic optimality. Although a simple assumption called non-degeneracy (or non-singularity) suffices to ensure the local linearity property described above, non-degeneracy alone is not enough to guarantee exponential optimality of the LP-Priority policies considered in \citep{GasGauYan_23_exponential,GasGauYan_23_whittles}.
In particular, another crucial assumption is needed, namely the Uniform Global Attractor Property (UGAP).
UGAP is a stronger version of the global attractor property (GAP) --- the latter is assumed in all asymptotic analyses of LP-Priority policies, without which an LP-Priority may even have a constant optimality gap  \citep[see][for concrete examples]{WebWei_90,GasGauYan_20_whittles,HonXieCheWan_23}.
Unfortunately, UGAP is hard to interpret and verify, as it concerns the global convergence of a certain non-linear difference equation.

In light of the difficulty caused by UGAP, very recent work has studied new policies that are asymptotically optimal without assuming UGAP \citep{HonXieCheWan_23,HonXieCheWan_24,Yan_24_multichain}.
Unlike LP-Priority and Whittle index policies, these new policies actively control the empirical distribution of the states of the arms, driving the distribution to \emph{globally converge} towards a certain optimal distribution, without relying on extraneous assumptions.
However, while these results conclude that UGAP is not needed for asymptotic optimality, it remains unclear whether UGAP is fundamental for exponential asymptotic optimality --- the best optimality gap proved in this line of work is $O(1/\sqrt{N})$.

Given these developments, it is natural to ask if removing UGAP necessarily comes at the cost of degrading the optimality gap.
Is it possible to efficiently find policies that achieve \emph{exponential asymptotic optimality} without assuming UGAP?

\subsection*{Our contributions}
The primary contribution of the paper is to design the first policy, the \emph{two-set policy}, that achieves exponential asymptotic optimality without UGAP.
The two-set policy maintains two dynamic subsets of arms and applies two different subroutines to the subsets:
one subset of arms has a nearly optimal state distribution and takes actions according to a certain \lppriority; the other subset of arms is driven to the optimal state distribution by another subroutine, \frcontrol,
and gradually merged into the first subset.
Intuitively, the \lppriority in the first subset induces local linearity, while the second subset enables global convergence.
While the ideas in each subroutine are not completely new, the two-set policy leverages them in a novel way to achieve the best of both worlds.
The theoretical guarantee of the two-set policy is informally summarized in the theorem below.

\begin{theorem*}[Informal version of Theorem~\ref{thm:two-set:achievability}]
    Suppose each arm is unichain, aperiodic, non-degenerate, and locally stable.
    Then the two-set policy $\pi$ satisfies
    \begin{align}
        \rrel - \rsysn = O(\exp(-CN)),
    \end{align}
    where $\rrel$ is an upper bound on the optimal long-run average reward, and $\rsysn$ is the long-run average reward under the policy $\pi$ given the vector of initial states $\veS_0$.
\end{theorem*}
We will make precise the definitions of non-degeneracy and local stability. We emphasize that we only need aperiodic unichain under a particular policy rather than all policies.
The assumptions of Theorem~\ref{thm:two-set:achievability} are a subset of, and hence strictly weaker than, the assumptions made in prior work \citep{GasGauYan_23_whittles,GasGauYan_23_exponential} for exponential asymptotic optimality. Moreover, all our assumptions pertain to a linear program defined by problem primitives and are thus much easier to verify than UGAP.

The algorithmic ideas of the two-set policy share some common traits with those in the ``set-expansion'' policy in \citep{HonXieCheWan_24} and the ``align and steer'' policy in \citep{Yan_24_multichain}.
However, these two policies have only been shown to have $O(1/\sqrt{N})$ and $o(1)$ optimality gaps, respectively.

The proof of Theorem~\ref{thm:two-set:achievability} employs a novel multivariate Lyapunov function technique, which generalizes the focus-set approach in \citep{HonXieCheWan_24}. This multivariate Lyapunov function enables us to decouple the complex dynamics under the two-set policy, where the states of the arms and the two dynamic subsets are coupled and change simultaneously.

We next elucidate the roles of our assumptions in achieving exponential optimality. 
The importance of unichain, aperiodicity, and non-degeneracy is well recognized in the literature; we will complement the discussion with concrete examples.
We further establish the fundamental importance of local stability by proving the following theorem.

\begin{theorem*}[Informal version of Theorem~\ref{thm:instability-lower-bound}]
    For regular unstable RBs, every policy $\pi$ satisfies
    \begin{align}
        \rrel - \rsysn = \Omega(1/\sqrt{N}).
    \end{align}
\end{theorem*}
We will formally define the class of regular unstable RBs and provide concrete examples. Note that all known efficient policies rely on the upper bound $\rrel$.
For these policies, \Cref{thm:instability-lower-bound} shows that local stability is crucial for achieving exponential gap to the upper bound $\rrel$ for a natural class of RB instances.
The proof of \Cref{thm:instability-lower-bound} uses Lyapunov analysis to reveal an interesting trade-off between two competing needs when local stability fails: to keep the states of the arms near the optimal distribution, and to take the optimal actions when the states are near the optimal distribution.

\subsection*{Other related work} Apart from the average-reward RBs, there is also a rich body of literature on RBs with the finite-horizon total-reward criterion \citep{HuFra_17_rb_asymptotic,ZayJasWan_19_rb,BroSmi_19_rb,ZhaFra_21,BroZha_22,DaeGhoGri_23,GhoNagJaiTam_23_finite_discount,BroZha_23_ftva_and_reopt,GasGauYan_23_exponential,GasGauYan_24_reopt,CheWanYin_25_RB_Gaussian_approx} or the infinite-horizon discounted-reward criterion \citep{BroSmi_19_rb,ZhaFra_22_discounted_rb,GhoNagJaiTam_23_finite_discount,BroZha_23_ftva_and_reopt}.
Among these papers, \citep{ZhaFra_21,BroZha_23_ftva_and_reopt,CheWanYin_25_RB_Gaussian_approx} achieve $O(1/N)$ optimality gaps, and \citep{GasGauYan_23_exponential,GasGauYan_24_reopt} achieve exponential optimality gaps.
The approaches in these settings do not directly apply to the average-reward setting, as their computational complexities scale with the (effective) time horizon.

Building on prior work and this work, \citet{WuZurCheWanXie_26_wcmdp_learn} studied the learning setting in which the transition kernels and reward functions are unknown. Under slightly stronger assumptions, they show that the optimality gap of the learned policy is bounded by the sum of an $O(\exp(-CN))$ term and a finite-sample learning-error term.

\subsection*{Paper organization}
This paper is organized as follows. 
We first describe the setup of the restless bandit problem in \Cref{sec:setup}.
Then we introduce our assumptions and the two-set policy in \Cref{sec:preliminaries}.
Next, we state the optimality guarantee in \Cref{sec:results-policy-and-opt-gap} and discuss what violations of the assumptions imply for the gap to the LP-relaxation upper bound in \Cref{sec:results-necessity}.
We prove the optimality guarantee in \Cref{sec:analysis} and the lower bound in \Cref{sec:lower-bound-proof}, then present numerical experiments in \Cref{sec:experiments}.
We conclude in \Cref{sec:conclusion}. Appendix~A proves that UGAP implies local stability, formally relating our assumptions to those in prior work. Other auxiliary results, examples, and experiment details are provided in the online companion, included at the end of this version.

\section{Problem setup}
\label{sec:setup}
Consider the discrete-time average-reward restless bandit problem with $N$ homogeneous arms, referred to as the $N$-armed problem. Index the arms by $[N] \triangleq \{1,2,\dots, N\}.$
To avoid confusion with the Whittle index, we refer to $i$ as the \emph{ID} of Arm~$i$. Each arm is associated with a Markov decision process (MDP) defined by  $(\sspa, \aspa, P, r),$ which is called the single-armed MDP.
Here $\sspa$ is a finite state space;
$\aspa = \{0, 1\}$ is the action space, where action $1$ is activating/pulling the arm;
$P:\sspa\times\aspa\times\sspa \to [0,1]$ is the transition kernel;
$r: \sspa \times \aspa \to \mathbb{R}$ is the reward function.
Let $\rmax = \max_{s\in\sspa,a\in\aspa} \abs{r(s,a)}$.
The RB problem has a \emph{budget constraint}, given by a constant $\alpha\in(0,1)$, which requires that $\alpha$ fraction of arms must be pulled at every time step. We assume that $\alpha N$ is an integer.
We focus on the setting where all the model parameters, $\sspa, \aspa, P, r, \alpha$, are known.

A policy $\pi$ for the $N$-armed problem chooses the action for each of the $N$ arms in each time step.
The policy can be randomized and history-dependent.
Under $\pi$, the \emph{state vector} $\ve{S}_t^\pi \triangleq (S_t^\pi(i))_{i\in[N]} \in \sspa^N$ represents the states of all arms at time $t$, and the \emph{action vector}  $\ve{A}_t^\pi \triangleq (A_t^\pi(i))_{i\in[N]} \in \aspa^N$ denotes the actions applied to each arm.
For each policy $\pi,$ consider the following two quantities: the limsup average reward
$\rlimsup \triangleq \limsup_{T\to\infty} \frac{1}{T} \sum_{t=0}^{T-1} \frac{1}{N} \sumN \E{r(S_t^\pi(i), A_t^\pi(i))};$ and the liminf average reward $\rliminf \triangleq \liminf_{T\to\infty} \frac{1}{T} \sum_{t=0}^{T-1} \frac{1}{N} \sumN \E{r(S_t^\pi(i), A_t^\pi(i))}$.
When the limsup and the liminf average rewards are equal, the \emph{long-run average reward} is defined as:
\[
    \rsysn \triangleq \lim_{T\to\infty} \frac{1}{T} \sum_{t=0}^{T-1} \frac{1}{N} \sumN \E{r(S_t^\pi(i), A_t^\pi(i))}.
\]

Our goal is solving the optimization problem:
\begin{align}
    \label{eq:N-arm-formulation} \tag{RB}
    \underset{\text{policy } \pi}{\text{maximize}} & \quad \rliminf  \qquad
    \text{subject to}
    \quad  \sumN A_t^\pi(i) = \alpha N,\quad \forall t\ge 0.
\end{align}
Let $\ropt$ be the optimal value of \eqref{eq:N-arm-formulation}, termed as the \emph{optimal reward}.
Note that $\ropt = \sup_{\pi'} R^-(\pi', \veS_0) = \sup_{\pi'} R^+(\pi', \veS_0)$ because \eqref{eq:N-arm-formulation} is an MDP with finite state and action spaces \citep[][Theorem~9.1.6]{Put_05}.
For any policy $\pi$, we define its optimality gap as $\ropt - \rliminf$. Consistent with the literature, we call a policy $\pi$ \emph{asymptotically optimal} if its optimality gap vanishes as $N\to\infty$, i.e., $\ropt - \rliminf = o(1)$ \cite[][Definition 4.11]{Ver_16_verloop}.

In the rest of the paper, we restrict ourselves to policies whose long-run average reward $\rsysn$ is well-defined. Such policies include all stationary Markovian policies \citep[][Proposition 8.1.1]{Put_05}, as well as any stationary policy that makes decisions based on an augmented system state with a finite state space, because they all induce finite-state Markov chains.
Note that focusing on such policies is sufficient because there always exists a stationary Markovian policy that achieves the optimal reward \citep[][Theorem~9.1.8]{Put_05}.
Consequently, we can refer to $\rsysn$ as the objective function of \eqref{eq:N-arm-formulation} and write the optimality gap as $\ropt - \rsysn$.

\textbf{Scaled state-count vector.~~}
We introduce an alternative representation of the $N$-armed system state,
which is used extensively throughout the paper.
For each subset $D\subseteq[N]$, we define the \emph{scaled state-count vector on $D$} as $X_t^\pi(D) = (X_t^\pi(D, s))_{s\in\sspa}$, where
$
    X_t^\pi(D, s) \triangleq \frac{1}{N} \sum_{i\in D} \indibrac{S_t^\pi(i) = s}
$ is the number of arms in $D$ in state $s$ scaled by $1/N.$
When $D = [N]$ is the set of all arms, we simply call $X_t^\pi([N])$ the \emph{scaled state-count vector}.
Sometimes we view $X_t^\pi(D)$ as a vector-valued function of $D \subseteq [N]$. We refer to this function $X_t$ as the \emph{system state} at time $t$. Note that the system state $X_t^\pi$ contains the same information as the state vector $\veS_t^\pi$.
We also define the \emph{scaled state-action-count vector} as $Y_t^\pi = (Y_t^\pi(s, a))_{s\in\sspa, a\in\aspa}$, where
$Y_t^\pi(s,a) \triangleq \frac{1}{N} \sumN \indibrac{S_t^\pi(i)=s, A_t^\pi(i) = a}$ is the fraction of arms with state $s$ and action $a$ at time $t$.

\textbf{Additional notation.~~} Let $[0,1]_N = \{0,1/N, 2/N,\dots, 1\}$. For each subset $D\subseteq[N]$, $m(D) = |D| / N$ denotes the fraction of arms contained in $D$, thus $m(D) \in [0,1]_N.$
Let $\Delta(\sspa)$ be probability simplex on  $\sspa$. For an integer $k>0$, let $I_{k}$ be the $k$-by-$k$ identity matrix.
We treat each distribution $v\in\Delta(\sspa)$ as a row vector.
Recall that $\pi$ denotes a policy for the $N$-armed problem. When the context is clear, we drop superscript $\pi$ from $\veS_t^\pi$, $\veA_t^\pi$, and $X_t^\pi$.

\section{Two-Set Policy}\label{sec:preliminaries}
In this section, we describe our main policy, the two-set policy.
We first discuss a linear programming (LP) relaxation of \eqref{eq:N-arm-formulation} in \Cref{sec:lp-relaxation} and state assumptions.
The optimal solution of the LP is then used to define two subroutines, \frcontrol in \Cref{sec:optimal-single-armed}, and \lppriority in \Cref{sec:LP-Priority}. We then use these two subroutines to construct the two-set policy in \Cref{sec:two-set}.
The ideas in the two subroutines are not new; our novelty lies in how we build them into the two-set policy.

\subsection{LP relaxation}\label{sec:lp-relaxation}

We consider the linear programming (LP) relaxation of the $N$-armed problem \eqref{eq:N-arm-formulation}, which is crucial for the design and analysis of RB policies.
\begin{align}
    \label{eq:lp-single} \tag{LP} \underset{\{y(s, a)\}_{s\in\sspa,a\in\aspa}}{\text{maximize}} \mspace{6mu}&\sum_{s\in\sspa,a\in\aspa} r(s, a) y(s, a) \\
    \text{subject to}\mspace{6mu}
    &\mspace{15mu}\sum_{s\in\sspa} y(s, 1) = \alpha, \label{eq:expect-budget-constraint}\\
    & \sum_{s'\in\sspa, a\in\aspa} \!\!\! y(s', a) P(s', a, s) = \sum_{a\in\aspa} y(s,a),\, \forall s; 
    \label{eq:flow-balance}\\
    &\sum_{s\in\sspa, a\in\aspa} \!\!\! y(s,a) = 1;
    \;
     y(s,a) \geq 0, \; \forall s,a. \label{eq:non-negative-constraint}
\end{align}
To see why \eqref{eq:lp-single} is a relaxation of \eqref{eq:N-arm-formulation}, for any stationary Markovian policy $\pi$, consider
$
    y^\pi(s,a) = \lim_{T\to\infty} \frac{1}{T} \sum_{t=0}^{T-1} \Ebig{\frac{1}{N} \sumN \indibrac{S_t^\pi(i) = s, A_t^\pi(i) = a}}, \forall s\in\sspa, a\in\aspa.
$
It is not hard to see that $\rsysn = \sumsa r(s,a)y^\pi(s,a)$, and $(y^\pi(s,a))_{s\in\sspa, a\in\aspa}$ satisfies the constraints \eqref{eq:expect-budget-constraint}-\eqref{eq:non-negative-constraint}.
Therefore, letting $\rrel$ be the optimal value of \eqref{eq:lp-single}, one can show that $\rrel \geq \ropt$ \citep[][]{HonXieCheWan_24}.
Note that the LP relaxation can be efficiently solved. We can bound the optimality gap of any policy $\pi$ using the inequality $\ropt - \rliminf \leq \rrel - \rliminf$.
This bound has been widely used in prior work \citep{WebWei_90,Ver_16_verloop,GasGauYan_23_exponential,GasGauYan_23_whittles,HonXieCheWan_23}.

Later in this section, we will make assumptions and construct our policy based on an optimal solution to \eqref{eq:lp-single}, $y^* = \{y^*(s, a)\}_{s\in\sspa,a\in\aspa}$. When there are multiple optimal solutions, we stick to an arbitrary one that satisfies all the assumptions.

\textbf{Optimal single-armed policy.} For a fixed optimal solution to \eqref{eq:lp-single}, $y^* = \{y^*(s, a)\}_{s\in\sspa,a\in\aspa}$, we consider the conditional distribution $\pibs = (\pibs(a|s))_{s\in\sspa, a\in\aspa}$ given by:
\begin{equation}\label{eq:single-arm-opt-def}
    \pibs(a | s) =
    \begin{cases}
        y^*(s, a) / (y^*(s, 0) + y^*(s,1)), & \text{if } y^*(s, 0) + y^*(s,1) > 0,  \\
        1/2, &  \text{if } y^*(s, 0) + y^*(s,1) = 0.
    \end{cases}
    \quad \text{for $s\in\sspa$, $a\in\aspa$.}
\end{equation}
We call $\pibs$ the \emph{optimal single-armed policy} associated with $y^*$, since it can be viewed as the optimal policy for the single-armed MDP $(\sspa, \aspa, P, r)$ with a relaxed budget constraint \citep[][]{HonXieCheWan_24}.

Let $P_{\pibs}$ be the transition matrix induced by the policy $\pibs$ on the single-armed MDP, i.e., $P_{\pibs}(s,s') = \sum\nolimits_{a\in\aspa} \pibs(a|s) P(s,a,s')$.
Since $P_\pibs$ is a stochastic matrix, its largest eigenvalue is equal to $1$. For the other eigenvalues of $P_\pibs$, we consider the following assumption.

\begin{assumption}[Aperiodic unichain]\label{assump:aperiodic-unichain}
    The transition probability matrix $P_\pibs$ has a simple eigenvalue $1$; the other eigenvalues all have a modulus strictly smaller than $1$. In other words, $P_\pibs$ defines an aperiodic unichain on $\sspa$.
    \footnote{A finite-state Markov chain is a unichain if it has a single-recurrent class, with possibly some transient states.}
\end{assumption}

Most literature assume certain versions of the aperiodic unichain condition \citep{WebWei_90,Ver_16_verloop,GasGauYan_23_exponential,GasGauYan_23_whittles,HonXieCheWan_23,HonXieCheWan_24}, which are either stronger than or equivalent to \Cref{assump:aperiodic-unichain} (See the appendices of \citep{HonXieCheWan_24} for a detailed discussion).
With \Cref{assump:aperiodic-unichain}, the transition matrix $P_\pibs$ induces a unique stationary distribution, $\statdist \triangleq (\statdist(s))_{s\in\sspa}$, where $\statdist(s) = y^*(s,1) + y^*(s,0)$, as one can verify using the definition of $P_\pibs$.
We call $\statdist$ the \emph{optimal stationary distribution}. We define $\sempty \triangleq \{s\in\sspa\colon y^*(s,1) = y^*(s,0) = 0\},$ the set of transient states under policy $\pibs.$

Our next assumption imposes a non-degenerate condition on the optimal solution to \eqref{eq:lp-single}.

\begin{assumption}[Non-degenerate]\label{assump:non-degeneracy}
    For the fixed optimal solution to \eqref{eq:lp-single}, $y^*$,
    there exists a unique \emph{fluid neutral state} (or simply \emph{neutral state}), i.e., a state $\sneu\in\sspa$ s.t. $y^*(\sneu, 1) > 0$ and $y^*(\sneu, 0) > 0$.
\end{assumption}

\Cref{assump:non-degeneracy} is considered in all prior analysis of RB policies that achieves exponential asymptotic optimality \citep{GasGauYan_23_exponential,GasGauYan_23_whittles}.
Note that one can always find an optimal solution of \eqref{eq:lp-single} such that there is at most one neutral state, but it is possible that every optimal solution has no neutral state.
The existence of such an LP solution has been proved for the case without transient states in \citet[][Section~6.1]{GasGauYan_23_exponential}, and the general case can be found in \citet[Theorem~3.8]{Alt_99_constrained_mdp}.

We note that this notion of non-degeneracy is different from $y^*$ being a non-degenerate basic feasible solution of the LP relaxation \eqref{eq:lp-single}: the two notions coincide when $\sempty=\emptyset$ but are mutually exclusive when $\sempty\neq\emptyset$. See \Cref{subsec:rb:non-degeneracy-validity} for the comparison.

\begin{assumption}[Local stability]\label{assump:local-stability}
    Given the non-degenerate condition,
    define $\Phi$ as a $|\sspa|$-by-$|\sspa|$ matrix:
    \begin{equation}\label{eq:phi-def}
        \Phi \triangleq P_\pibs - (\costvec-\alpha\vone)^\top  (P_1(\sneu) - P_0(\sneu)) - \vone^\top \statdist,
    \end{equation}
    where $\costvec \triangleq (\pibs(1|s))_{s\in\sspa}$ and $P_a(\sneu) \triangleq (P(\sneu,a, s))_{s\in\sspa}$ are both row vectors; $\vone$ is the all-one row vector.
    We assume that the moduli of all eigenvalues of $\Phi$ are strictly less than $1$.
\end{assumption}

The matrix $\Phi$ captures the transition dynamics induced by the subroutine \lppriority (\Cref{lem:lp-priority-and-phi}).
\Cref{assump:local-stability} is equivalent to the statement that right-multiplication by $\Phi_0 \triangleq P_{\pibs} - (\costvec-\alpha\vone)^\top  (P_1(\sneu) - P_0(\sneu))$ is a contraction on the probability simplex.
Intuitively, $\Phi$ and $\Phi_0$ agree on the probability simplex up to translation; subtracting $\vone^\top\statdist$ from $\Phi_0$ replaces the trivial eigenvalue $1$ of $\Phi_0$ with $0$ without affecting the other eigenvalues, which lets us state local stability as the single spectral condition that all eigenvalues of $\Phi$ have modulus less than $1$. See \Cref{prop:local-stability-equivalence} for a formal statement and proof.
We will give a more detailed interpretation of $\Phi_0$ in the remark after \Cref{lem:lp-priority-and-phi}.

\begin{remark}[Local stability and UGAP]\label{remark:UGAP}
Local stability replaces UGAP assumed in \citep{GasGauYan_23_exponential,GasGauYan_23_whittles}. 
UGAP is concerned with the deterministic iteration given by $x_{t+1} = \phi(x_t)$, where $\phi(x) \triangleq \E{X_{1}([N]) \givenplain X_0([N]) = x}$ denotes the mean transition map on $\simplex(\sspa)$ under an LP-Priority policy, which we introduce in \Cref{sec:LP-Priority};
UGAP requires $x_t$ to converge to $\statdist$ uniformly for any initial point $x_0$. 
We can show that UGAP implies local stability, using the fact that the matrix $\Phi$ in \eqref{eq:phi-def} is the local linearization of $\phi$ at $\statdist$ (\Cref{lem:lp-priority-and-phi}); the formal statement and proof are provided in \Cref{app:ugap}.
Moreover, UGAP is strictly stronger than local stability since it rules out other global failure modes, such as having multiple attractors or limit cycles; it is known to fail on instances where local stability still holds \citep[see, e.g.][Section~8]{HonXieCheWan_24}.
\end{remark}

\begin{remark}[UGAP and ergodicity]
UGAP should not be confused with ergodicity of the $N$-armed Markov chain: ergodicity regulates the long-run behavior of the stochastic chain $X_t([N])$, whereas UGAP is about the mean-field dynamics $x_t$. Ergodicity is needed for the steady-state to be well-defined, whereas UGAP is used to show the tight concentration of $\{X_\infty([N])\}_t$ around $\statdist$ \citep[Appendix C]{GasGauYan_23_whittles}.
\end{remark}

\begin{remark}\label{remark:assump-verifiability}
    Assumptions~\ref{assump:aperiodic-unichain}, \ref{assump:non-degeneracy} and \ref{assump:local-stability} are easy to verify as they pertain to either the \emph{single-armed}  MDP primitives ($P, r, \alpha$) or a linear program define by these problem primitives.
    In contrast, verifying UGAP is significantly harder due to the dependency on the \emph{non-linear $N$-armed} dynamics.
\end{remark}

\subsection{Overview of our main policy}
Before starting to introduce elements for constructing our main policy, the two-set policy (\Cref{alg:two-set}), we first give an overview.
The two-set policy utilizes two subroutines, \frcontrol (\Cref{alg:follow-pibs}) and \lppriority (\Cref{alg:follow-lp-priority}); the latter is a local version of LP-Priority policies.
At each time step, the two-set policy partitions $N$ arms into three subsets: the first subset $\Db_t$ has nearly optimal state distribution and takes actions according to the \lppriority; the second subset $\Da_t$ is driven towards the optimal state distribution by the \frcontrol and is gradually merged into $\Db_t$; the third subset $(\Db_t  \cup \Da_t)^c$ chooses actions to meet the budget constraint $\sumN A_t(i) = \alpha N$, and will be gradually merged into $\Da_t$.
In the long run, $\Db_t$ will contain all arms in the system and remain so for the majority of the time; the \lppriority guarantees that the average reward is close to $\rrel$.

\begin{subroutine}[H]
\caption{\frcontrol}\label{alg:follow-pibs}
\hspace*{\algorithmicindent} \hspace{-0.8em} \textbf{Input}: A set of arms with number of arms in each state $(z(s))_{s\in\sspa}$,
optimal single-armed policy $\pibs$
\begin{algorithmic}[1]
    \For{$s\in\sspa$}
        \State Activate $\pibs(1|s) z(s)$ expected number of arms in state $s$ through randomized rounding
    \EndFor
\end{algorithmic}
\end{subroutine}

\begin{subroutine}[H]
\caption{\lppriority}\label{alg:follow-lp-priority}
\hspace*{\algorithmicindent} \textbf{Input}: A set of $n$ arms with number of arms in each state $(z(s))_{s\in\sspa}$, budget parameter $\alpha$, \\
\hspace*{\algorithmicindent} \hspace{0.37in} optimal single-armed policy $\pibs$, neutral state $\sneu$ \\
\hspace*{\algorithmicindent} \textbf{Assert} \Cref{assump:non-degeneracy} and conditions in \eqref{eq:lp-priority-condition-1}
\begin{algorithmic}[1]
    \State Let the target budget $\bsub=\floor{\alpha n}$ or $\ceil{\alpha n}$ with suitable probability such that $\E{B} = \alpha n$
    \State Activate all arms in state $s\in S^+$\label{alg:lp-priority:non-sneu-actions-0}
    \State Activate no arms in state $s\in S^-$\label{alg:lp-priority:non-sneu-actions-1}
    \State Activate $z(s)/2$ expected number of arms in state $s$ through randomized rounding for $s\in S^{\emptyset}$\label{alg:lp-priority:non-sneu-actions-2}
    \State Activate the arms with state $\sneu$ using the remaining budget out of $\bsub$ \label{alg:lp-priority:sneu-actions}
\end{algorithmic}
\end{subroutine}

\subsection{Subroutine 1: \frcontrol}\label{sec:optimal-single-armed}

We now describe the \frcontrol, as presented in \Cref{alg:follow-pibs}. The subroutine takes the optimal single-armed policy $\pibs$ defined in \eqref{eq:single-arm-opt-def} and a subset of arms $D$ from the RB system as the input. The subroutine decides the actions for all arms in $D$ to match with $\pibs.$ Specifically, let the number of arms from $D$ in state $s\in \sspa$ be $z(s).$ The subroutine aims to activate $\pibs(1|s)$ fraction of the $z(s)$ arms. To handle the integer effect, it performs a \emph{randomized rounding}: it activates $\floor{\pibs(1|s) z(s)}$ or $\ceil{\pibs(1|s) z(s)}$ arms with properly chosen probability such that the expected proportion is $\pibs(1|s)$.

If all $N$ arms follow \frcontrol, $\E{X_{t+1}([N])} = \E{X_t([N])}P_\pibs$ (\Cref{lem:pibs-transition} in \Cref{app:subroutine-transition-lemmas}), so $\E{X_t([N])}$ converges to $\statdist$ as $t\to\infty$. As a result, the expected fraction of arms in state $s$ taking action $a$ converges to $y^*(s,a)$;
the \emph{expected budget usage} of the arms converges to $\alpha N$;
the long-run average reward achieves the optimum of \eqref{eq:lp-single}.

\begin{remark}[Alternative to fixed-ratio control]
    An alternative to \frcontrol is independently sampling each arm's action using $\pibs$, an approach taken by \citep{HonXieCheWan_23,HonXieCheWan_24}. 
    Both approaches activate the same expected number of arms in each state, but \frcontrol is cleaner for analysis as its activation count deviates from the mean by at most $1$.
\end{remark}

\subsection{Subroutine 2: \lppriority}\label{sec:LP-Priority}

Next, we describe \Cref{alg:follow-lp-priority}, \lppriority, another building block for constructing our policy. 
In the literature, \lppriority refers to \emph{a class of} priority policies \citep[see, e.g.][]{GasGauYan_23_exponential} satisfying certain a consistency condition with the optimal LP solution $y^*$. This consistency condition makes each instance of \lppriority act identically when the empirical state distribution of the arms is sufficiently close to $\statdist$. Below we define this identical behavior in the neighborhood of $\statdist$, and simply use ``\lppriority'' to refer to a generic instance of \lppriority restricted to this neighborhood.

Given a set of $n$ arms, the \lppriority starts with a target budget $B$ generated using the randomized rounding satisfying $\E{B} = \alpha n$. Then it partitions the state space $\sspa$ into four subsets based on $y^*$: $S^+ = \{s\in \sspa \colon y^*(s,1) > 0, y^*(s,0) = 0\}$,  $S^- = \{s\in \sspa \colon y^*(s,1) = 0, y^*(s,0) > 0 \}$,  $S^{\emptyset} = \{s\in \sspa \colon y^*(s,1) = 0, y^*(s,0) = 0\}$,  and $S^0 = \{s\in \sspa \colon y^*(s,1) > 0, y^*(s,0) > 0\}$.
Note that by \Cref{assump:local-stability}, $S^0 = \{\sneu\}$. \lppriority tries to activate all arms in $S^+$, keep arms in $S^-$ passive, activate half of the arms in state $s$ for each $s\in S^{\emptyset}$, and activate a suitable number of arms in $\sneu$ to meet the target budget $\bsub$.
Equivalently, the subroutine aims to activate a $\pibs(1|s)$ fraction of the $z(s)$ arms from $D$ in state $s$ for each non-neutral state $s\in \sspa \backslash \{\sneu\}$, and activate arms in $\sneu$ according to $B$.
To ensure that the target budget can be met, we consider the following conditions for the input subset $D$:
\begin{equation}
        \label{eq:lp-priority-condition-1}
        \sum\nolimits_{s\neq \sneu} \pibs(1|s) z(s) \leq  \alpha n - |\sempty|-1, \quad \mbox{and} \quad
         \sum\nolimits_{s\neq \sneu} \pibs(0|s) z(s) \leq (1-\alpha)n - |\sempty|-1.
\end{equation}
By the definition of $\pibs$ in \eqref{eq:single-arm-opt-def}, it is not hard to see that given \eqref{eq:lp-priority-condition-1}, there hold $\sum_{s\neq \sneu} \ceil{\pibs(1|s) z(s)} \leq \floor{\alpha n}$ and $z(\sneu) + \sum_{s\neq \sneu} \floor{\pibs(1|s) z(s)} \geq  \ceil{\alpha n}$. Therefore, the budget requested by non-neural states $\sspa \backslash \{\sneu\}$ does not exceed $B$,
and the neural state ensures that exactly $\bsub$ unit of budget is used.

When all arms follow \Cref{alg:follow-lp-priority}, their expected scaled state-count vector evolves linearly, as stated in the following lemma. Its proof is given in \Cref{app:subroutine-transition-lemmas}.

\begin{restatable}{lemma}{followlppriority}\label{lem:lp-priority-and-phi}
    Under \Cref{assump:non-degeneracy}, suppose \eqref{eq:lp-priority-condition-1} holds for the $N$-armed system at time step $t$. We have
    \begin{equation}\label{eq:lp-priority-trans}
        \Eplain{X_{t+1}([N]) -\statdist \givenplain X_t, \text{all arms follow \Cref{alg:follow-lp-priority}}} = (X_t([N]) - \statdist) \Phi,
    \end{equation}
    where $\Phi$ is the $|\sspa|$-by-$|\sspa|$ matrix defined in \eqref{eq:phi-def}.
\end{restatable}

Thus under \Cref{assump:local-stability},  \Cref{alg:follow-lp-priority} induces locally stable dynamics. Specifically, if the RB system satisfies \eqref{eq:lp-priority-condition-1} for a sufficiently long time and we let all $N$ arms follow \Cref{alg:follow-lp-priority}, then the scaled state-count vector $X_t([N])$ would converge to $\statdist$. Consequently, the expected fraction of arms in state $s$ taking action $a$ would be close to $y^*(s,a)$ and the time-average of expected reward would be close to the upper bound $\rrel$---hence the name ``\lppriority'' for \Cref{alg:follow-lp-priority}.

\begin{remark}[Additional comment on $P_{\pibs}$ and $\Phi_0$]
\lppriority can be viewed as an \emph{adaptive} modification of \frcontrol: it keeps activating a $\pibs(1 | s)$ fraction of arms at every state $s\neq\sneu$, but adjusts this fraction at $\sneu$ in an $X_t([N])$-dependent way so that the per-step budget constraint is met exactly. 
This explains the relationship between $P_\pibs$ and $\Phi_0 = P_{\pibs} - (\costvec-\alpha\vone)^\top  (P_1(\sneu) - P_0(\sneu))$, where the additional rank-one term in $\Phi_0$ performs this adjustment.
\end{remark}

\subsubsection*{When is the condition in \eqref{eq:lp-priority-condition-1} true?}

We argue that the condition in \eqref{eq:lp-priority-condition-1} is true when the empirical distribution on $D$, $x(D) / m(D)$, is sufficiently close to $\statdist$, and $m(D)$ is not too close to zero.
Specifically, we measure distance using the $U$-weighted $L_2$ norm $\norm{v}_\umat = \sqrt{v \umat v^\top}$ for $v\in \R^{|\sspa|}$, with $U$ given by $\umat \triangleq \sum_{k=0}^\infty \Phi^k (\Phi^\top)^k$. This norm determines the subsets used by the two-set policy in \Cref{sec:two-set}. Its well-definedness and contraction properties are established later in \Cref{sec:weighted-l2-norms}, as part of the policy analysis.
In \Cref{lem:feasibility-ensuring} stated below, we show that there exists $\eta = \Theta(1)$ and $\errfe = \Theta(1/N)$ such that the condition in \eqref{eq:lp-priority-condition-1} is true if
\begin{equation}\label{eq:sufficient-cond-for-OL-assertion}
  \norm{x(D) - m(D) \statdist}_{\umat} \leq \eta m(D) - \errfe.
\end{equation}

\begin{restatable}[Feasibility-ensuring pair]{lemma}{feasibilityensuring}\label{lem:feasibility-ensuring}
    For any system state $x$ and $D\subseteq [N]$, if \eqref{eq:sufficient-cond-for-OL-assertion} holds with $\eta = |\sspa|^{-1/2}\min\{y^*(\sneu, 0), y^*(\sneu, 1)\}$ and $\errfe = |\sspa|^{-1/2}(|\sempty|+1)/N$, then \eqref{eq:lp-priority-condition-1} also holds.
\end{restatable}

The proof of \Cref{lem:feasibility-ensuring} is given in \Cref{app:proof-of-exist-fe}.
For a pair of $(\eta, \errfe)$ such that \eqref{eq:sufficient-cond-for-OL-assertion} implies the assertion, we refer to it as \emph{feasibility ensuring}.

We can equivalently write \eqref{eq:sufficient-cond-for-OL-assertion} as the non-negativity of a \textit{slack} function, $\slk(x, D)$, given by
\begin{equation}
    \label{eq:slkb-def}
    \slkb(x, D) = \eta m(D) -  \norm{x(D) - m(D) \statdist}_{\umat}  - \errfe.
\end{equation}

In our main policy, we also need a notion of approximate maximality for subsets satisfying \eqref{eq:sufficient-cond-for-OL-assertion}.

\begin{definition}[$\errtol$-maximal feasible set]
\label{def:maximal-feasible-set}
    Given the current system state $x$,
    a set of arms $D\subseteq[N]$ is $\errtol$-maximal feasible if the two conditions hold:
    (1) $\slk(x, D)\geq 0$ or $D=\emptyset$;
    (2) for any $D'$ such that $D \subseteq D'\subseteq [N]$ and $\slkb(x, D') \geq \errtol$, we have $m(D') \leq m(D) + \errtol$.
\end{definition}

To understand this definition, note that when $\errtol = 0$,
a $\errtol$-maximal feasible set is an element of the collection of sets $\mathcal{F}(x) = \{D'\subseteq[N] \colon \slk(x, D') \geq 0\}\cup \{\emptyset\}$ whose supersets are not in $\mathcal{F}(x)$.
Thus, this definition captures the ideas of letting as many arms follow the \lppriority as possible.
By letting $\errtol$ to be a small positive number, we tolerate some errors in the maximality, which allows a $\errtol$-maximal feasible set to be identified efficiently for any system state $x$.
In \Cref{app:exp-details:two-set}, we provide concrete algorithms to efficiently determine the $\errtol$-maximal feasible sets for some $\errtol = O(1/N)$ when discussing the implementations of our main algorithm.

\subsection{Constructing the Two-Set Policy}\label{sec:two-set}

We now define the \emph{two-set policy} for the $N$-armed system, whose pseudocode is given in \Cref{alg:two-set}.
On a high level, the two-set policy works by picking out two disjoint subsets of arms among the $N$ arms in each time step, and letting them follow two subroutines, \lppriority and \frcontrol, respectively.
When the two sets are picked properly, the policy pushes the scaled state-count vector $X_t([N])$ of the $N$-armed system towards $\statdist$. Ultimately, the system will stabilize in the region where all arms can follow \Cref{alg:follow-lp-priority} $\lppriority$.

Specifically, at time step $t$, the two-set policy first picks a subset $\Db_t$, where the arms will follow $\lppriority$.
The policy chooses $\Db_t$ to be a $\errtol$-maximal feasible set for some predetermined $\errtol=O(1/N)$, with some additional requirements as specified in Lines~\ref{alg:two-set:update-db-step-begin}--\ref{alg:two-set:update-db-step-end} of \Cref{alg:two-set}.
Roughly speaking, the two-set policy intends to maximize the cardinality of $\Db_t$ and let $\Db_t$ be a superset of $\Db_{t-1}$ whenever possible, under the constraint $\slkb(X_t, \Db_t) \geq 0$.
The constraint $\slkb(X_{t}, \Db_{t}) \geq 0$ ensures that arms in $\Db_t$ can follow the \lppriority. Moreover, due to nice properties of the $\umat$-weighted norm, the choice of $\Db_t$ is ``closed'' under the \lppriority, in the sense that $\slkb(X_{t+1}, \Db_{t})$ remains non-negative with high probability, and $\Db_{t+1}$ can usually be a superset of $\Db_t$ (See \Cref{lem:two-set:non-shrink}).

After choosing $\Db_t$, the two-set policy picks a subset $\Da_{t}$ that contains $\beta \triangleq \min(\alpha, 1-\alpha)$ fraction of the remaining arms, lower-rounded to the nearest integer.
The set $\Da_{t}$ should be such that $\Da_t \cap \Db_t = \emptyset$ and $\Da_{t} \supseteq \Da_{t-1} \backslash \Db_{t}$.

With $\Db_t$ and $\Da_t$ chosen, the two-set policy let arms in $\Db_t$ follow \Cref{alg:follow-lp-priority} \lppriority, let the arms in $\Da_t$ follow \Cref{alg:follow-pibs} \frcontrol, and chooses the actions of remaining arms to meet the budget constraint.
\Cref{lem:two-set:subroutine-conform} below shows that these action-selection steps are always feasible.

\begin{algorithm}[t]
\caption{Two-Set Policy}
\label{alg:two-set}
\hspace*{\algorithmicindent} \textbf{Input}: number of arms $N$, budget $\alpha N$, an optimal solution of LP-relaxation $y^*$, \\
\hspace*{\algorithmicindent} \hspace{0.37in} $\beta\triangleq \min(\alpha, 1-\alpha)$, a feasibility-ensuring pair of ($\eta$, $\errfe$), error tolerance $\errtol = O(1/N)$, \\
\hspace*{\algorithmicindent} \hspace{0.37in} initial system state $X_0$, initial state vector $\veS_0$, initial subsets $\Db_{-1}=\Da_{-1}=\emptyset$
\begin{algorithmic}[1]
    \For{$t=0,1,\dots$}
        \If{$\slk(X_t, [N]) \geq 0$}
        \label{alg:two-set:update-db-step-begin}
            \State Let $\Db_t = [N]$
        \ElsIf{$\slk(X_t, \Db_{t-1}) \geq 0$}
            \State Let $\Db_t$ be any $\errtol$-maximal feasible subset such that $\Db_t \supseteq \Db_{t-1}$
        \Else
            \State Let $\Db_t$ be any $\errtol$-maximal feasible subset
        \label{alg:two-set:update-db-step-end}
        \EndIf
        \State %
         Let $\Da_{t}$ be a subset of $[N] \backslash \Db_t$ with $|\Da_{t}| = \big\lfloor \beta (N -|\Db_{t}|)\big\rfloor$, s.t. either $\Da_{t}\supseteq \Da_{t-1} \backslash \Db_{t}$ or $\Da_{t} \subseteq \Da_{t-1} \backslash \Db_{t}$
        \State Set $A_t(i)$ for $i\in \Db_{t}$ using \Cref{alg:follow-lp-priority} \lppriority
        \label{alg:two-set:db-action}
        \State Set $A_t(i)$ for $i\in \Da_{t}$ using \Cref{alg:follow-pibs} \frcontrol
        \label{alg:two-set:da-action}
        \State Set $A_t(i)$ for $i\notin \Db_{t} \cup \Da_{t}$ such that $\sumN A_t(i) = \alpha N$
        \label{alg:two-set:remaining-action}
        \State Apply $A_t(i)$ for $i\in[N]$ and observe the new state vector $\veS_{t+1}$
    \EndFor
\end{algorithmic}
\end{algorithm}

\subsubsection*{Subroutine conformity.}

The following lemma completes the policy construction by establishing that its prescribed actions are well-defined and can satisfy the global budget constraint.

\begin{restatable}[Subroutine conformity]{lemma}{subroutineconformity}\label{lem:two-set:subroutine-conform}
    Assume Assumptions~\ref{assump:aperiodic-unichain}, \ref{assump:non-degeneracy} and \ref{assump:local-stability} hold.
    Consider the two-set policy described in \Cref{alg:two-set} with a pair of feasibility-ensuring $(\eta, \errfe)$.
    The steps in Lines~\ref{alg:two-set:db-action}-\ref{alg:two-set:remaining-action} are well-defined. In other words, we can always let the arms in $\Db_t$ follow the \lppriority, let the arms in $\Da_t$ follow the \frcontrol, and choose actions for the arms not in $\Db_t\cup\Da_t$ to satisfy the budget constraint $\sumN A_t(i) = \btotal$.
\end{restatable}

\begin{proof}[Proof of \Cref{lem:two-set:subroutine-conform}]
    We first argue that the arms in $\Db_t$ can follow the \lppriority.
    By the definition of $\Db_t$, if $\Db_t \neq \emptyset$, then $\slk(X_t, \Db_t) \geq 0$.
    Because $(\eta, \errfe)$ is feasibility ensuring, \eqref{eq:lp-priority-condition-1} holds for the arms in $\Db_t$, i.e.,
    \begin{align}
        \label{eq:two-set:db-condition-1}
        \sum_{s\neq \sneu} \pibs(1|s) X_t(\Db_t, s) &\leq  \alpha m(\Db_t) - \frac{|\sempty|+1}{N}, \\
        \label{eq:two-set:db-condition-2}
         \sum_{s\neq \sneu} \pibs(0|s) X_t(\Db_t, s) &\leq (1-\alpha)m(\Db_t) - \frac{|\sempty|+1}{N}.
    \end{align}
    which allows these arms to follow the \lppriority.

    It remains to show that $A_t(i)$ for $i\notin \Db_t \cup \Da_t$ can be chosen such that $\sumN A_t(i) = \btotal$, if the arms in $\Db_t$ and $\Da_t$ follow the \lppriority and the \frcontrol, respectively.

    By the definition of the \lppriority, $\sum_{i\in \Db_t} A_t(i) \in \{\lfloor\alpha |\Db_t|\rfloor, \lceil\alpha |\Db_t|\rceil\}$.
    Because $A_t(i)\in \{0,1\}$ for $i\in \Da_t$, we have
    \[
        \floor{\alpha |\Db_t|} \leq \sum_{i\in \Db_t\cup\Da_t} A_t(i)\leq \ceil{\alpha |\Db_t|} + |\Da_t|.
    \]
    Because we can choose $\sum_{i\notin \Db_t\cup\Da_t} A_t(i)$ to be anything integer between $0$ and $N - |\Db_t| - |\Da_t|$ to satisfy the budget constraint, it suffices to show that
    \begin{align}
        \label{eq:pf-subroutine-conf:interm-goal-1}
        \floor{\alpha |\Db_t|} + (N - |\Db_t| - |\Da_t|) &\geq \btotal    \\
        \label{eq:pf-subroutine-conf:interm-goal-2}
        \ceil{\alpha |\Db_t|} + |\Da_t| &\leq \btotal.
    \end{align}
    To show \eqref{eq:pf-subroutine-conf:interm-goal-1}, observe that $|\Da_t| = \big\lfloor \beta (N -|\Db_{t}|)\big\rfloor \leq (1-\alpha)(N - |\Db_t|)$, so
    \begin{equation*}
        \alpha |\Db_t| + (N - |\Db_t| - |\Da_t|)  \geq \alpha N.
    \end{equation*}
    Rounding both sides down to the nearest integers, and utilizing the fact that $N - |\Db_t| - |\Da_t|$ and $\alpha N$ are integers, we get \eqref{eq:pf-subroutine-conf:interm-goal-1}. Similarly, to show \eqref{eq:pf-subroutine-conf:interm-goal-2}, observe that $|\Da_t| \leq \alpha (N - |\Db_t|)$, so
    \[
        \alpha |\Db_t| + |\Da_t| \leq \alpha N.
    \]
    Rounding both sides up to the nearest integers, we get  \eqref{eq:pf-subroutine-conf:interm-goal-2}.
\end{proof}

\section{Main results}
\label{sec:results}

In this section, we present our main results.

\subsection{Optimality guarantee on Two-Set Policy}\label{sec:results-policy-and-opt-gap}

Our first main theorem states an upper bound on the optimality gap of the two-set policy.
\begin{theorem}\label{thm:two-set:achievability}
    Suppose Assumptions~\ref{assump:aperiodic-unichain}, \ref{assump:non-degeneracy} and \ref{assump:local-stability} hold. Let $\pi$ be the two-set policy in \Cref{alg:two-set} with a pair of feasibility-ensuring $(\eta, \errfe)$ such that $\eta=\Theta(1)$ and $\errfe=O(1/N)$, and with error tolerance $\errtol=O(1/N)$. Then for some $C > 0$ independent of $N$:
    \begin{equation}
        \rrel - \rsysn = O\big(\exp(-CN)\big).
    \end{equation}
\end{theorem}

Since $\ropt\leq\rrel$, \Cref{thm:two-set:achievability} establishes the \emph{first} exponential asymptotic optimality without UGAP.
In contrast, prior work either requires UGAP to achieve exponential asymptotic optimality or only has an $O(1/\sqrt{N})$ optimality gap.
We prove \Cref{thm:two-set:achievability} in \Cref{sec:analysis}.

\begin{remark}\label{remark:two-set-one-set}
We discuss the relationship between the two-set policy in this paper and a policy from prior work named the set-expansion policy \citep{HonXieCheWan_24}. The set-expansion policy achieves an $O(1/\sqrt{N})$ optimality gap under the aperiodic unichain assumption (\Cref{assump:aperiodic-unichain}); it works by letting arms in a specific subset follow $\pibs$ and gradually expanding the subset.
The two-set policy in this paper refines this approach by maintaining two judiciously chosen sets instead of one set.
Furthermore, we argue that the two-set policy can be implemented in a way such that when Assumption~\ref{assump:non-degeneracy} or \ref{assump:local-stability} fails, it achieves $O(1/\sqrt{N})$ optimality gap, because it becomes almost the same as the set-expansion policy.
To see this, note that when Assumption~\ref{assump:non-degeneracy} (non-degeneracy) fails, it holds that $\eta = 0$; when Assumption~\ref{assump:local-stability} (local stability) fails, the matrix $U$ becomes infinity; in either case, $\Db_t = \emptyset$.
By taking suitable actions for the arms in $(\Db_t\cup\Da_t)^c$, we can maintain a subset of arms following the \frcontrol that gradually expands in the same way as the set-expansion policy.
Using similar arguments as those in \citep{HonXieCheWan_24}, one can prove an $O(1/\sqrt{N})$ bound for the two-set policy.
\end{remark}

\subsection{Discussion on the necessities of the assumptions}
\label{sec:results-necessity}

In this subsection, we discuss the necessities of Assumptions~\ref{assump:aperiodic-unichain}--\ref{assump:local-stability} in \Cref{thm:two-set:achievability}.
In particular, we establish lower bounds of the following flavor: when any of Assumptions~\ref{assump:aperiodic-unichain}--\ref{assump:local-stability} does not hold, there exist RB problem instances such that for every policy $\pi$ and certain initial state $\veS_0$, it holds that
\begin{equation}
    \label{eq:lower-bound}
    \rrel - \rsysn = \Omega\Big(\frac{1}{\sqrt{N}}\Big).
\end{equation}
It is important to note that $\rrel - \rsysn$ is an upper bound on the optimality gap $\ropt - \rsysn$, so a lower bound on $\rrel - \rsysn$ does not necessarily imply a lower bound on the optimality gap.
Nevertheless, \emph{all} existing bounds on the optimality gap reply on $\rrel$ or its variants, and our lower bounds reveal insights on when it is impossible to get exponentially close to $\rrel$.

\subsubsection*{When \hcref{assump:aperiodic-unichain} (aperiodic  unichain) or \hcref{assump:non-degeneracy} (non-degeneracy) is violated.}
We construct RB instances where \eqref{eq:lower-bound} holds for every policy $\pi$ and certain initial states $\veS_0$; the details are provided in \Cref{app:examples_vio_assump}.

\subsubsection*{When \hcref{assump:local-stability} (local stability) is violated.}

Instead of finding one counterexample, we take a step further to establish the lower bound \eqref{eq:lower-bound} for a large class of instances.

\begin{definition}[Regular unstable RBs]
    \label{def:regular-unstable}
    For an RB instance satisfying \Cref{assump:aperiodic-unichain} and \Cref{assump:non-degeneracy}, we call it regular unstable if it has the following properties:
    \begin{enumerate}[leftmargin=2em, label=(\alph*)]
        \item \label{assump:unique-and-communicating}
        \eqref{eq:lp-single} has a unique optimal solution $y^*$, and $y^*(s,1) + y^*(s,0) > 0$ for all $s\in\sspa$.
        \item \label{assump:instability}
        The matrix $\Phi$ defined in \eqref{eq:phi-def} has at least one eigenvalue whose modulus is strictly larger than $1$, and the modulus of any eigenvalue of $\Phi$ does not equal $1$.
        \item \label{assump:orthogonal-noise}
        For each $\xi \in \R^{|\sspa|}$ with $\normplain{\xi}_2 = 1$ and $\statdist \xi^\top = 0$, for each large-enough $N$  and each time step $t$,
        \begin{equation*}
            \Ebig{\absbig{\big(X_{t+1}([N]) - \statdist - (X_t([N]) - \statdist)\Phi \big)  \xi^\top } \givenbig X_t, \veA_t} \geq \frac{\noiseortho}{\sqrt{N}} \quad \text{ if } \norm{X_t([N]) - \statdist}_1 \leq \radiusus,
        \end{equation*}
        where $\radiusus, \noiseortho > 0$ are some constants; $\veA_t$ follows the \lppriority (\Cref{alg:follow-lp-priority}), which is always feasible if $\norm{X_t([N]) - \statdist}_1 \leq \min\{y^*(\sneu, 0), y^*(\sneu, 1)\}-1/N$.
        \end{enumerate}
\end{definition}

Our next main theorem is a lower bound on the gap to the LP-relaxation upper bound for regular unstable RBs; it is proved in \Cref{sec:lower-bound-proof}.

\begin{theorem}\label{thm:instability-lower-bound}
    For every regular unstable RB, policy $\pi$ and initial state vector $\veS_0$, we have
    \[
        \rrel - \rsysn = \Omega\Big(\frac{1}{\sqrt{N}}\Big).
    \]
\end{theorem}

\Cref{thm:instability-lower-bound} implies that an exponential optimality gap compared to the LP relaxation upper bound $\rrel$ cannot be achieved for a significant subset of  RB instances violating local stability (\Cref{assump:local-stability}).
Therefore, in addition to the widely accepted non-degenerate condition, local stability is also fundamental for exponential  optimality.

We briefly comment on Property \ref{assump:orthogonal-noise} of regular unstable RB: this property requires the state-count vector to have sufficient randomness in most directions under the \lppriority when the system is near the optimal state distribution $\statdist$.
In \Cref{app:remark-noise-assumption}, we provide an intuitive sufficient condition for \ref{assump:orthogonal-noise}, and a weaker version of the condition under which \Cref{thm:instability-lower-bound} remains valid.

The proof of \Cref{thm:instability-lower-bound} is based on the following idea: By \Cref{lem:opt-gap-linear-comp}, arms need to closely follow the \lppriority to maintain a small gap to the LP-relaxation upper bound; however, if the \lppriority is locally unstable, the system will quickly drift away from the region where the \lppriority is applicable.
In the proof, we use Lyapunov functions to quantify how much the system drifts away from $\statdist$ in each step, based on how closely it follows the \lppriority.

The following corollary is proved by a concrete example of a regular unstable RB in \Cref{app:locally-unstable-example}.

\begin{restatable}{corollary}{existregularunstable}\label{cor:exist-regular-unstable}
    There exists an RB satisfying Assumptions \ref{assump:aperiodic-unichain} and \ref{assump:non-degeneracy} and violating \Cref{assump:local-stability} such that
    $
        \rrel - \rsysn \geq \Omega\Big(\frac{1}{\sqrt{N}}\Big).
     $
\end{restatable}

\section{Analysis of two-set policy (Proof of \hcref{thm:two-set:achievability})}
\label{sec:analysis}

We prove \Cref{thm:two-set:achievability} using a Lyapunov drift argument. We first present the general framework in \Cref{sec:lyapunov-framework}, then introduce the weighted norms in \Cref{sec:weighted-l2-norms} needed to construct our Lyapunov functions. \Cref{sec:two-set:warm-up} derives an $O(1/\sqrt{N})$ bound as a warm-up, following the approach in \citep{HonXieCheWan_24}; \Cref{sec:two-set:exponential-bound} strengthens the construction and proves the exponential bound of \Cref{thm:two-set:achievability}. The remaining subsections prove the four key lemmas used in this argument.

\subsection{Lyapunov framework}\label{sec:lyapunov-framework}

The two-set policy induces a finite-state Markov chain on the aggregate state $\Fullstate_t \triangleq  (X_t,\, \Db_t,\, \Da_t)$,
whose state space is a finite set $\fullssp$, the set of all $(X_t,\, \Db_t,\, \Da_t)$-tuples attained with positive probability under the policy $\pi$. We write $\fullstate = (x, \Db, \Da)$ for a generic element of $\fullssp$.
For ease of presentation, we assume the Markov chain $\fullstate_t$ is aperiodic. Then, for any fixed initial state $(X_0, \Db_0, \Da_0)$, the limiting distribution is well-defined. We denote $\Fullstate_\infty=(X_{\infty}, \Db_{\infty}, \Da_{\infty})$ as the random element distributed according to this limiting distribution.
The general case requires only notational changes: expectations with respect to the limiting distribution are replaced with long-run time averages. We omit the details for simplicity.

We define the one-step drift operator
\[
    \Delta f(\fullstate) \triangleq  \E{f(\Fullstate_{t+1}) \given \Fullstate_t = \fullstate} - f(\fullstate)
\]
for any function $f\colon \fullssp\to\R$. Taking $\fullstate=\Fullstate_{\infty}$, by the stationarity, we have $\E{\Delta f(\Fullstate_\infty)} = 0$.

We define the expected instantaneous reward under the two-set policy as
\[
    r^\pi(\fullstate) \triangleq  \sumsa r(s,a)\,\Ebig{Y_t(s,a) \givenbig \Fullstate_t=\fullstate}.
\]
The long-run average reward of the policy then satisfies $\rsysn = \E{r^\pi(\Fullstate_\infty)}$.

We seek a pair of functions $V, \drftup\colon\fullssp\to\R$, a constant $c>0$ independent of $N$, and a scalar $\loworder \geq 0$ satisfying the following two conditions:
\begin{enumerate}
    \item \emph{(Drift condition)} $\Delta V(\fullstate) \leq  -c\,\drftup(\fullstate) \;+\; \loworder$ for all $\fullstate\in\fullssp$;
    \item \emph{(Dominance condition)} $\drftup(\fullstate) \geq \rrel - r^\pi(\fullstate)$ for all $\fullstate\in\fullssp$.
\end{enumerate}
This pair of conditions is equivalent to the standard Lyapunov drift condition \citep[Ch.~14]{MeyTweGly_09}.
Here $V$ is the \emph{Lyapunov function}, whose value at $\fullstate$ quantifies the system's distance to convergence, and $\drftup$ bounds $V$'s expected rate of decrease, up to the constant factor $c$ and the residual $\loworder$.
The drift condition requires that $V$ decreases at a rate of at least $c\,$$\drftup(\fullstate) - \loworder$; the dominance condition ensures that $\drftup$ is in turn large enough to bound the reward gap $\rrel - r^\pi(\fullstate)$.
Taking expectations in both conditions under $\Fullstate_\infty$ and using $\E{\Delta V(\Fullstate_\infty)} = 0$ then gives
\begin{equation}\label{eq:sketch:gap-bound}
    \rrel - \rsysn = \rrel - \E{r^\pi(\Fullstate_\infty)} \leq \E{\drftup(\Fullstate_\infty)} \leq \frac{\loworder}{c}.
\end{equation}
Since $\ropt\leq\rrel$ and $c$ is independent of $N$, the order of the optimality gap is determined by the residual $\loworder$.
We apply this framework twice below: in the warm-up in \Cref{sec:two-set:warm-up}, we take $\drftup$ to be the distance from $\statdist$ and $\loworder=O(1/\sqrt{N})$; in \Cref{sec:two-set:exponential-bound}, we take $\drftup$ to be just $\rrel-r^\pi(\fullstate)$ and $\loworder=O(\exp(-CN))$.

\subsection{Weighted $L_2$ norms for quantifying the convergence of distributions}\label{sec:weighted-l2-norms}
To construct Lyapunov functions for the framework above, we need distances that contract under the two subroutines.
As discussed in \Cref{sec:optimal-single-armed} and \Cref{sec:LP-Priority}, if a set of arms follow the \frcontrol or the \lppriority for a long time, the distribution of their states converge to the optimal stationary distribution $\statdist$. 
In our analysis, it is convenient to define a proper metric to witness this convergence in each time step and quantify the convergence rate. 
For this purpose, we introduce two matrices, $\wmat$ and $\umat$, and study the $\wmat$-weighted ($\umat$-weighted) $L_2$ norm on $\sspa$. 
Note that the same matrix $\wmat$ and $\wmat$-weighted norm have also been used in \citep{HonXieCheWan_24}. 

The definition of $\wmat$ and $\umat$ are given as follows. 

\begin{restatable}{definition}{wudef}\label{def:w-and-u}
    Assume Assumption~\ref{assump:aperiodic-unichain}, \ref{assump:non-degeneracy} and \ref{assump:local-stability}. 
    Let $\wmat$ and $\umat$ be $|\sspa|$-by-$|\sspa|$ matrices given by  
    \begin{align}
        \label{eq:w-def}
        \wmat &= \sum_{k=0}^\infty (P_\pibs - \vone^\top \statdist)^k ((P_\pibs - \vone^\top \statdist)^\top)^k \\
        \umat &= \sum_{k=0}^\infty \Phi^k (\Phi^\top)^k,
    \end{align}
\end{restatable}

Next, we show that the matrices $\wmat$ and $\umat$ are well-defined and positive definite, and lower bound their eigenvalues by $1$. Then if we let $\lamw$ and $\lamu$ denote maximal eigenvalues of $\wmat$ and $\umat$, the eigenvalues of $\wmat$ and  $\umat$ are in the range $[1, \lamw]$ and $[1, \lamu]$, respectively.

\begin{lemma}\label{lem:W-U-well-defined}
    Assume \Cref{assump:aperiodic-unichain}, \ref{assump:non-degeneracy} and \ref{assump:local-stability}, the matrices $\wmat$ and $\umat$ given in \Cref{def:w-and-u} are well-defined. Moreover, their eigenvalues are lower bounded by $1$. 
\end{lemma}

\begin{proof}
    The well-definedness and eigenvalue lower bound for $\wmat$ has been proved in \citep{HonXieCheWan_24}. The main fact used in the proof is that all eigenvalues of $P_\pibs - \vone^\top \statdist$ have moduli strictly less than $1$ under \Cref{assump:aperiodic-unichain}. 
    Since we have assumed the same things for $\Phi$ in \Cref{assump:local-stability}, the proof for $\umat$ is analogous with $P_\pibs - \vone^\top \statdist$ replaced by  $\Phi$. 
\end{proof}

With matrices $\wmat$ and $\umat$, we can define some weighted $L_2$ norms. 
Let $\norm{\cdot}_\wmat$ denote the $\wmat$-weighted $L_2$ norm, given by $\norm{u}_\wmat = \sqrt{u \wmat u^\top}$ for any row vector $u\in \R^{|\sspa|}$; let $\norm{\cdot}_\umat$ denote the $\umat$-weighted $L_2$ norm, defined similarly. 

\Cref{lem:one-step-contraction-W-U} below shows that $P_\pibs$ and $\Phi$ are pseudo-contractions on the probability simplex $\simplex(\sspa)$, under the $\wmat$-weighted $L_2$ norm and the $\umat$-weighted $L_2$ norm, respectively.

\begin{restatable}[Pseudo-contraction under the weighted $L_2$ norms]{lemma}{wucontraction}\label{lem:one-step-contraction-W-U}
    Assume Assumption~\ref{assump:aperiodic-unichain}, \ref{assump:non-degeneracy} and \ref{assump:local-stability}. 
    For any distribution $v\in \Delta(\sspa)$, we have
    \begin{align}
        \label{eq:pibar-contraction}
        \norm{(v  - \statdist)P_\pibs}_\wmat &\leq \rhow \norm{v - \statdist}_\wmat \\
         \label{eq:lp-contraction}
        \norm{(v  - \statdist)\Phi}_\umat &\leq \rhou \norm{v - \statdist}_\umat.
    \end{align}
    where $\norm{\cdot}_\wmat$ is the $\wmat$-weighted $L_2$ norm, i.e., $\norm{u}_\wmat = \sqrt{u \wmat u^\top}$ for any row vector $u\in \R^{|\sspa|}$; $\norm{\cdot}_\umat$ is the $\umat$-weighted $L_2$ norm defined similarly; $\rhow = 1 - 1/(2\lamw)$; $\rhou = 1 - 1/(2\lamu)$.  
\end{restatable}

\begin{proof}
    The inequality \eqref{eq:pibar-contraction} has been proved in \citep{HonXieCheWan_24}. The proof of \eqref{eq:lp-contraction} is analogous. 
\end{proof}

\subsection{Warm-up: proving $O(1/\sqrt{N})$ optimality gap}
\label{sec:two-set:warm-up}
As a warm-up, we first bound the gap to the LP-relaxation upper bound by $O(1/\sqrt{N})$, which also bounds the optimality gap. The argument will be subsequently strengthened to yield an $O(\exp(-CN))$ bound.

We take $\uone(\fullstate) = K_0\,\norm{\statdist - x([N])}_\umat$, for a constant $K_0 > 0$ chosen so that the dominance condition
\[
    \rrel - r^\pi(\fullstate) \leq  K_0\,\norm{\statdist - x([N])}_\umat
\]
holds for every $\fullstate\in\fullssp$. Intuitively, such a $K_0$ exists because $r^\pi(\fullstate) = \rrel$ whenever $x([N]) = \statdist$, and $r^\pi$ changes at a bounded rate as $x([N])$ deviates from $\statdist$. A formal derivation can be read off from \Cref{lem:sketch:inst-reward}, which we state and use in the $O(\exp(-CN))$ analysis; for the warm-up we simply take the dominance condition as given.

To establish the drift condition with $V=\Vds$, it suffices to find a constant $c>0$ independent of $N$ such that
\begin{equation}
    \label{eq:sketch:vds-drift-goal}
    \Delta \Vds(\fullstate) \leq  -c\,\uone(\fullstate) \;+\; O\!\left(\frac{1}{\sqrt{N}}\right).
\end{equation}

To construct such a Lyapunov function, the key challenge lies in the complex dynamics of the two-set policy, which maintains two subsets, each running a different subroutine and could expand or shrink over time. 
To tackle this challenge, we use a novel \emph{multivariate Lyapunov function} that takes not only the system state $X_t$ but also the subsets $\Db_t$, $\Da_t$ as variables.

Specifically, we first define two functions, $\hw(x, D)$ and $\hu(x, D)$, for any system state $x$ and subset $D\subseteq[N]$:
\begin{align}
    \label{eq:hw-def}
    \hw(x, D) &= \norm{x(D) - m(D)\statdist}_\wmat \\
    \label{eq:hu-def}
    \hu(x, D) &= \norm{x(D) - m(D)\statdist}_\umat.
\end{align}
These two functions are Lyapunov functions that witness convergence of $X_t(D)$ to $m(D)\statdist$ if the arms in $D$ follow the \frcontrol or the \lppriority, and we refer to them as \emph{subset Lyapunov functions}.

Now we define $\Vds$, using $\hw(x, D)$ and $\hu(x, D)$ as building blocks:
\begin{equation}\label{eq:sketch:Vds}
    \Vds(\fullstate) = \hu(x, \Db) + \hw(x, \Da) + \Lds\,(1 - m(\Db)),
\end{equation}
where $\Lds \triangleq 2\lamu^{1/2} + 4\lamw^{1/2} - 2\lamw^{1/2}\beta$ is a large constant.
Intuitively, the first two terms of $\Vds(\fullstate)$ handles the state transitions of the arms in $\Db$ and $\Da$, whereas the third term handles the effect of changing $\Db$ and $\Da$. We have the following lemma for properties of $\Vds$:

\begin{restatable}[Properties of $\Vds$]{lemma}{vdsproperties}
    \label{lem:Vds-properties}
    Assume \Cref{assump:aperiodic-unichain,assump:non-degeneracy,assump:local-stability} hold and let $\threshbar=\eta-\errfe$ for a feasibility-ensuring pair $(\eta,\errfe)$. There exist $\rhoFinal\in(0,1)$ and $\KVds, \KVdst, \CVds>0$ independent of $N$, such that for all $t\geq 0$ and all $\fullstate=(x,\Db,\Da)\in\fullssp$,
    \begin{align}
        \label{eq:Vds-mean}
        \E{\Bigl(\Vds(\Fullstate_{t+1}) - \rhoFinal\,\Vds(\Fullstate_t) \Bigr)^{\!+} \given \Fullstate_t = \fullstate} &\leq \frac{\KVds}{\sqrt{N}} \\
        \label{eq:Vds-tail}
        \E{\Bigl(\Vds(\Fullstate_{t+1}) - \rhoFinal\,\Vds(\Fullstate_t) - \tfrac{(1-\rhoFinal)\threshbar}{2}\Bigr)^{\!+} \given \Fullstate_t = \fullstate} &\leq \KVdst\exp(-\CVds N),
    \end{align}
    and
    \begin{equation}\label{eq:Vds-dom}
        \Vds(\fullstate) \geq  \norm{x([N]) - \statdist}_\umat.
    \end{equation}
\end{restatable}

\Cref{lem:Vds-properties} is proved in \Cref{sec:two-set:pf-Vds-properties}.
Equation~\eqref{eq:Vds-dom} shows that $\Vds$ dominates the target distance $\norm{x([N])-\statdist}_\umat$.
Equations~\eqref{eq:Vds-mean} and \eqref{eq:Vds-tail} both establish a geometric decay of $\Vds$ by bounding $\Vds(\Fullstate_{t+1}) - \rhoFinal\,\Vds(\Fullstate_t)$ in a suitable probabilistic sense: \eqref{eq:Vds-mean} controls the mean and \eqref{eq:Vds-tail} controls the tail.
Together, \eqref{eq:Vds-dom} and \eqref{eq:Vds-mean} imply
\[
    \Delta\Vds(\fullstate) \leq  -(1-\rhoFinal)\norm{x([N]) - \statdist}_\umat + \frac{\KVds}{\sqrt{N}},
\]
Since $\uone(\fullstate)=K_0\norm{\statdist-x([N])}_\umat$, this is \eqref{eq:sketch:vds-drift-goal} with $c=(1-\rhoFinal)/K_0$. Thus, the framework in \Cref{sec:lyapunov-framework} applies directly with $V=\Vds$, $\drftup=\uone$, and $\loworder=\KVds/\sqrt{N}$. The dominance condition holds by the choice of $\uone$, so \eqref{eq:sketch:gap-bound} and $\ropt\leq\rrel$ yield an $O(1/\sqrt{N})$ optimality gap. The tail bound \eqref{eq:Vds-tail} will be used when we establish the stronger $O(\exp(-CN))$ bound for the optimality gap.

\subsection{Exponential optimality gap (Proof of \hcref{thm:two-set:achievability})}
\label{sec:two-set:exponential-bound}
To get a stronger optimality-gap bound, we apply the framework with $\utwo(\fullstate) \triangleq \rrel-r^\pi(\fullstate)$, so that the dominance condition holds with equality. We construct a second Lyapunov function $\Vll(\fullstate)$ whose drift has an exponentially small residual $O_N$; the key is to exploit the \emph{local linearity} of the instantaneous reward.

Specifically, when $x([N])$ is sufficiently close to $\statdist$, namely, when $\norm{x([N]) - \statdist}_\umat \leq \threshbar \triangleq \eta - \errfe$, $r^\pi(\fullstate)$ is fully determined by the \lppriority. The following lemma, proved in \Cref{sec:two-set:pf-inst-reward}, shows that $r^\pi$ is then an affine function of $x([N])$.

\begin{restatable}[Instantaneous reward under \lppriority]{lemma}{instrew}
\label{lem:sketch:inst-reward}
    Under the two-set policy, for any $\fullstate\triangleq (x, \Db, \Da)\in\fullssp$, if $\norm{x([N]) - \statdist}_\umat \leq \threshbar$, we have
    \[
        r^\pi(\fullstate) = \rhat(x([N])),
    \]
    where $\rhat\colon \simplex(\sspa)\to\R$ is the affine function
    \begin{equation}\label{eq:sketch:rhat}
        \rhat(v) \triangleq  \sum_{s \neq \sneu,\, a\in\aspa} \bigl(r(s,a) - r(\sneu,a)\bigr)\pibs(a|s)\,v(s)
        \;+\; \alpha\,r(\sneu,1) + (1-\alpha)\,r(\sneu,0),
    \end{equation}
    which can be written in vector form as $\rhat(v) = v\cdot\gvec + r_0$ for some $\gvec \in \R^{|\sspa|}$ and $r_0\in \R$.
\end{restatable}

Since \Cref{lem:sketch:inst-reward} implies $\rhat(\statdist)=\rrel$, the affine representation of $\rhat$ yields the following:
\[
\begin{aligned}
    \utwo(\fullstate)
    &= (\statdist-x([N]))\cdot\gvec
    && \text{if } \norm{x([N])-\statdist}_\umat\leq\threshbar,\\
    \utwo(\fullstate)
    &\leq 2\rmax
    && \text{if } \norm{x([N])-\statdist}_\umat>\threshbar.
\end{aligned}
\]
The exact linear representation of $\utwo$ in the local region motivates the \emph{local-approximation} term
\begin{equation}\label{eq:sketch:phi-term}
    \locapprox(\fullstate) \triangleq  (\statdist - x([N]))\,Q\,\gvec,
\end{equation}
where $Q \triangleq (I - \Phi)^{-1}$ is well-defined since $\rho(\Phi) < 1$ by \Cref{assump:local-stability}.
Let
\begin{equation}
    \label{eq:lamq-ef}
    \lamq \triangleq  \max\Bigl\{\,\max_{v\in\simplex(\sspa)}\norm{(v-\statdist)Q}_1,\;\max_{v\in\simplex(\sspa)}\norm{(v-\statdist)\Phi Q}_1\,\Bigr\}.
\end{equation}

Combining $\locapprox$ with a truncated version of $\Vds$, we define $\Vll$ as
\begin{equation}\label{eq:sketch:Vtil}
    \Vll(\fullstate) = \Bigl(\Vds(\fullstate) - \frac{\threshbar}{2}\Bigr)^{\!+} + \Lvll\,\locapprox(\fullstate),
\end{equation}
where $\Lvll \triangleq (1-\rhoFinal)\threshbar\,/\,(8\rmax\lamq+4\rmax)$.
The following two lemmas, proved in \Cref{sec:two-set:pf-drift-truncate,sec:two-set:pf-drift-local}, bound the drift of the two summands.

\begin{restatable}[Drift of truncated $\Vds$]{lemma}{drifttruncate}
    \label{lem:drift-Vds-truncate-term}
    For any $\fullstate,\fullstate'\in\fullssp$,
    \begin{align}
         \bigl(\Vds(\fullstate') - \tfrac{\threshbar}{2}\bigr)^+ - \bigl(\Vds(\fullstate) - \tfrac{\threshbar}{2}\bigr)^+
        &\leq -\frac{(1-\rhoFinal)\threshbar}{2} \indibrac{\Vds(\fullstate) > \threshbar}  \nonumber\\
        &\qquad + \Bigl(\Vds(\fullstate') - \rhoFinal\,\Vds(\fullstate) - \tfrac{(1-\rhoFinal)\threshbar}{2}\Bigr)^{\!+}.\label{eq:drift-Vds-truncate}
    \end{align}
\end{restatable}

\begin{restatable}[Drift of $\locapprox$]{lemma}{driftlocal}
    \label{lem:drift-local-approx-term}
    For any $\fullstate=(x,\Db,\Da)\in\fullssp$,
    \begin{equation}\label{eq:drift-locapprox}
        \Delta\locapprox(\fullstate) \leq  -\utwo(\fullstate) + (4\rmax\lamq+2\rmax)\,\indibrac{\norm{x([N]) - \statdist}_{\umat} > \threshbar},
    \end{equation}
    where $\utwo(\fullstate)=\rrel-r^\pi(\fullstate)$, $\locapprox(\fullstate)=(\statdist-x([N]))\,Q\,\gvec$, $Q=(I-\Phi)^{-1}$, and $\threshbar=\eta-\errfe$.
\end{restatable}

These bounds reveal the complementary roles of the two summands. The $\Lvll\locapprox$ term contributes the desired drift $-\Lvll\utwo$ together with an error term when $\norm{x([N]) - \statdist}_{\umat} > \threshbar$; the truncated $\Vds$ term supplies the negative drift when $\norm{x([N]) - \statdist}_{\umat} > \threshbar$.
With the reward representation of \Cref{lem:sketch:inst-reward} and the bounds in \Cref{lem:Vds-properties,lem:drift-Vds-truncate-term,lem:drift-local-approx-term}, we can now complete the proof of \Cref{thm:two-set:achievability}.

\begin{proof}[Proof of \Cref{thm:two-set:achievability}]
Setting $\fullstate=\Fullstate_t$, $\fullstate'=\Fullstate_{t+1}$ in \Cref{lem:drift-Vds-truncate-term} and taking the conditional expectation given $\Fullstate_t=\fullstate$,
\begin{align}
    \Delta\bigl(\Vds(\fullstate)-\tfrac{\threshbar}{2}\bigr)^+
    &\leq -\frac{(1-\rhoFinal)\threshbar}{2}\,\indibrac{\Vds(\fullstate) > \threshbar}
    + \E{\Bigl(\Vds(\Fullstate_{t+1}) - \rhoFinal\,\Vds(\Fullstate_t) - \tfrac{(1-\rhoFinal)\threshbar}{2}\Bigr)^{\!+} \bigg| \Fullstate_t=\fullstate} \nonumber\\
    &\leq -\frac{(1-\rhoFinal)\threshbar}{2}\,\indibrac{\norm{x([N])-\statdist}_\umat > \threshbar} + \KVdst\exp(-\CVds N),\label{eq:sketch:truncated-Vds-drift}
\end{align}
where the expectation term is bounded by $\KVdst\exp(-\CVds N)$ via \eqref{eq:Vds-tail} of \Cref{lem:Vds-properties}, and the indicator is relaxed using $\indibrac{\Vds(\fullstate)>\threshbar} \geq \indibrac{\norm{x([N])-\statdist}_\umat>\threshbar}$, which holds since $\Vds(\fullstate)\geq\norm{x([N])-\statdist}_\umat$ by \eqref{eq:Vds-dom}.
Adding $\Lvll$ times \eqref{eq:drift-locapprox} to \eqref{eq:sketch:truncated-Vds-drift},
\begin{align}
    \Delta\Vll(\fullstate)
    &\leq \Bigl(-\frac{(1-\rhoFinal)\threshbar}{2} + \Lvll(4\rmax\lamq+2\rmax)\Bigr)
          \indibrac{\norm{x([N])-\statdist}_\umat > \threshbar} \nonumber\\
    &\qquad - \Lvll\utwo(\fullstate) + \KVdst\exp(-\CVds N) \nonumber\\
    \label{eq:sketch:Vll-drift}
    &= -\Lvll\utwo(\fullstate) + \KVdst\exp(-\CVds N),
\end{align}
where the bracket in the first line is zero by the choice $\Lvll = (1-\rhoFinal)\threshbar\,/\,(8\rmax\lamq+4\rmax)$.
Thus, the framework in \Cref{sec:lyapunov-framework} applies with $V=\Vll$, $\drftup=\utwo$, $c=\Lvll$, and $\loworder=\KVdst\exp(-\CVds N)$. Since the dominance condition holds with equality, \eqref{eq:sketch:gap-bound} gives
\[
    \rrel - \rsysn \leq  \frac{\KVdst}{\Lvll}\,\exp(-\CVds N) = O\!\left(\exp(-\CVds N)\right),
\]
which completes the proof.
\end{proof}

It remains to prove \Cref{lem:Vds-properties,lem:sketch:inst-reward,lem:drift-Vds-truncate-term,lem:drift-local-approx-term}, which we do in the following four subsections. \Cref{fig:flowchart-two-set-proof-new} summarizes their roles in the argument and the supporting lemmas used to establish \Cref{lem:Vds-properties}.

\begin{figure}[!htbp]
    \begin{tikzpicture}[node distance=2cm and 3cm]
        \node (hwhu) [lemma] at (-6, 0) {\Cref{lem:hw-hu-properties}};
        \node (nonshrink) [lemma] at (-2, 0) {\Cref{lem:two-set:non-shrink}};
        \node (coverage) [lemma] at (2, 0) {\Cref{lem:two-set:sufficient-coverage}};
        \node (vds) [lemma] at (-6, -3) {\Cref{lem:Vds-properties}};
        \node (trunc) [lemma] at (-2, -3) {\Cref{lem:drift-Vds-truncate-term}};
        \node (instrew) [lemma] at (2, -3) {\Cref{lem:sketch:inst-reward}};
        \node (local) [lemma] at (6, -3) {\Cref{lem:drift-local-approx-term}};
        \node (thm) [theorem] at (0, -6) {\Cref{thm:two-set:achievability}};
        \draw [arrow] (hwhu) -- (nonshrink);
        \draw [arrow] (hwhu.north) to[out=30, in=150] (coverage.north);
        \draw [arrow] (hwhu) -- (vds);
        \draw [arrow] (nonshrink) -- (vds);
        \draw [arrow] (coverage) -- (vds);
        \draw [arrow] (instrew) -- (local);
        \draw [arrow] (vds) -- (thm);
        \draw [arrow] (trunc) -- (thm);
        \draw [arrow] (local) -- (thm);
    \end{tikzpicture}
    \centering
    \caption{Dependencies in the proof of \Cref{thm:two-set:achievability}. The middle row contains the four key lemmas used above; the top row contains supporting lemmas stated in \Cref{sec:two-set:pf-Vds-properties}. Their proofs are given there, except for the norm-properties lemma, \Cref{lem:hw-hu-properties}, proved in \Cref{app:proof-hw-hu}. Subroutine conformity (\Cref{lem:two-set:subroutine-conform}), established in \Cref{sec:two-set}, underlies the analysis and is omitted for readability.}
    \label{fig:flowchart-two-set-proof-new}
\end{figure}

\subsection{Proof of \hcref{lem:Vds-properties}}\label{sec:two-set:pf-Vds-properties}

The proof of \Cref{lem:Vds-properties} uses three supporting lemmas: one on the subset Lyapunov functions $\hw$, $\hu$, and two on the dynamics of $\Db_t$. We state them first, then prove \Cref{lem:Vds-properties}, and finally prove the two lemmas on $\Db_t$. 
The lemma on properties of $\hw$, $\hu$, \Cref{lem:hw-hu-properties} is proved in \Cref{app:proof-hw-hu}, as it follows straightforwardly from the properties of the norms.

\subsubsection{Statement of supporting lemmas}
Recall the subset Lyapunov functions
\begin{align}
    \tag{\ref{eq:hw-def}}
    \hw(x, D) &\triangleq \norm{x(D) - m(D)\statdist}_\wmat \\
    \tag{\ref{eq:hu-def}}
    \hu(x, D) &\triangleq \norm{x(D) - m(D)\statdist}_\umat,
\end{align}
defined for each system state $x$ and subset $D\subseteq[N]$. The next lemma records their drift, concentration, strength, and Lipschitz properties.

\begin{restatable}[Properties of $\hw$ and $\hu$]{lemma}{hwhuproperties}\label{lem:hw-hu-properties}
    Assume Assumptions~\ref{assump:aperiodic-unichain}, \ref{assump:non-degeneracy} and \ref{assump:local-stability} hold.
    For any time step $t$ and $D\subseteq [N]$, let $X_{t+1}'$ be the system state at time $t+1$ if the arms in $D$ follow the \frcontrol at time $t$, then we have
    \begin{align}
        \label{eq:hw:drift}
        \Ebig{(\hw(X_{t+1}', D) - \rhow \hw(X_t, D))^+ \givenbig X_t} &\leq \frac{\KhwOne}{\sqrt{N}} \quad a.s. \\
        \label{eq:hw:high-prob}
        \Probbig{\hw(X_{t+1}', D) > \rhow \hw(X_t, D) + r   \givenbig X_t} &\leq \KhwTwo \exp(-\Chw N r^2) \quad a.s. \quad \forall r \geq 0,
    \end{align}
    where $\rhow = 1 - 1/(2\lamw)$, and $\KhwOne, \KhwTwo, \Chw$ are positive constants independent of $N$.
    In addition, for any $D, D'\subseteq [N]$,
    \begin{align}
        \label{eq:hw:strength}
        \hw(x, D) &\geq \frac{1}{|\sspa|^{1/2}} \norm{x(D) - m(D)\statdist}_1 \\
        \label{eq:hw:lipschitz}
        \abs{\hw(x, D) - \hw(x, D')} &\leq 2\lamw^{1/2} (m(D'\backslash D) + m(D\backslash D')),
    \end{align}
    The same inequalities hold for $\hu$ with the \frcontrol and $W$ replaced by the \lppriority and $U$. We omit the details here.
\end{restatable}

The next two lemmas characterize how the set $\Db_t$ evolves. \Cref{lem:two-set:non-shrink} shows that the arms in $\Db_t$ are most likely to remain in $\Db_{t+1}$, so that $\Db_t$ is almost non-shrinking in the set-inclusive sense. \Cref{lem:two-set:sufficient-coverage} exploits the maximality of $\Db_t$ to show that the fraction of arms not in $\Db_t$ is small whenever $X_t(\Da_t)$ is close to $m(\Da_t)\statdist$ or $X_t([N])$ is close to $\statdist$. Together they imply that $\Db_t$ tends to gradually expand and cover all arms as the two subroutines drive the states of the arms in $\Db_t$ and $\Da_t$ to the optimal distribution.

\begin{restatable}[Almost non-shrinking]{lemma}{almostnonshrinking}\label{lem:two-set:non-shrink}
    Assume Assumptions~\ref{assump:aperiodic-unichain}, \ref{assump:non-degeneracy} and \ref{assump:local-stability} hold.
    Consider the two-set policy (\Cref{alg:two-set}) such that $(\eta, \errfe)$ is feasibility-ensuring, $\eta=\Theta(1)$ and $\errfe=O(1/N)$.
    For any $t\geq0$, we have
    \begin{align}
        \label{eq:two-set:non-shrink-drift}
        \Ebig{m(\Db_t\backslash \Db_{t+1}) \givenbig \Fullstate_t} &\leq \KnsOne\, m(\Db_t)\exp(-\Cns N m(\Db_t)^2) \quad a.s. \\
        \label{eq:two-set:non-shrink-high-prob}
        \Probbig{m(\Db_{t}\backslash \Db_{t+1}) > 0 \givenbig \Fullstate_t} &\leq \KnsTwo \exp(-\Cns N m(\Db_t)^2) \quad a.s.,
    \end{align}
    where $\KnsOne, \KnsTwo, \Cns$ are positive constants independent of $N$.
\end{restatable}

\begin{restatable}[Sufficient coverage]{lemma}{sufficientcoverage}\label{lem:two-set:sufficient-coverage}
    Assume Assumptions~\ref{assump:aperiodic-unichain}, \ref{assump:non-degeneracy} and \ref{assump:local-stability} hold.
    Consider the two-set policy (\Cref{alg:two-set}) such that $(\eta, \errfe)$ is feasibility-ensuring, $\eta=\Theta(1)$, $\errfe=O(1/N)$, and $\errtol=O(1/N)$.
    For any $t\geq0$, we have
    \begin{align}
        \label{eq:two-set:sufficient-coverage:goal-1}
        1-m(\Db_t) &\leq  \KscOne \hw(X_t, \Da_t) + \frac{\KscTwo}{N} \quad a.s.\\
        \label{eq:two-set:sufficient-coverage:goal-2}
        m(\Db_t) &= 1 \quad \text{ if } \hu(X_t, [N]) \leq \eta - \errfe,
    \end{align}
    where $\KscOne, \KscTwo$ are positive constants independent of $N$.
\end{restatable}

\subsubsection{Proof of \hcref{lem:Vds-properties}}
With these three lemmas in hand, we now prove \Cref{lem:Vds-properties}.

\begin{proof}%
Recall that $\Vds(\fullstate)$ is defined as
\begin{equation}\label{eq:two-set:full-lyapunov}
    \Vds(\fullstate) = \hu(x, \Db) + \hw(x, \Da) + \Lds (1 - m(\Db)).
\end{equation}

\medskip
\textbf{Proof of \eqref{eq:Vds-dom}.}
By \Cref{lem:hw-hu-properties}, $\hu(x, D)$ is Lipschitz continuous in $D$, so $\hu(x, \Db)$ changes by at most $2\lamu^{1/2}(1-m(\Db))$ when we replace $\Db$ with $[N]$. Consequently, 
\[
    \Vds(\fullstate) \geq  \hu(x, \Db) + 2\lamu^{1/2} (1 - m(\Db)) \geq \hu(x, [N]) = \norm{x([N]) - \statdist}_\umat.
\]

\medskip
\textbf{An preliminary bound for proving \eqref{eq:Vds-mean} and \eqref{eq:Vds-tail}.}
Next, to prove \eqref{eq:Vds-mean} and \eqref{eq:Vds-tail}, we claim the following upper bound on $\Vds(\Fullstate_{t+1})-\Vds(\Fullstate_t)$: there exists a constant $\rhoFinal\in(0,1)$ such that
\begin{align}
    \nonumber
    & \Vds(\Fullstate_{t+1}) - \rhoFinal \Vds(\Fullstate_t) \\
    \label{eq:two-set:thm:v-diff-final-term-1}
    &\qquad \leq \big(\hu(X_{t+1}, \Db_t) - \rhou \hu(X_{t}, \Db_t)\big)+ \big(\hw(X_{t+1}, \Da_t) - \rhow \hw(X_{t}, \Da_t)\big)\\
    \label{eq:two-set:thm:v-diff-final-term-2}
    &\qquad \mspace{20mu} + 4\big(\lamu^{1/2} + \lamw^{1/2}\big) m(\Db_t \backslash \Db_{t+1}) + \frac{2\lamw^{1/2}}{N}. 
\end{align}
To show this, we start with the decomposition:
\begin{align}
    \Vds(\Fullstate_{t+1}) - \Vds(\Fullstate_t) 
    \label{eq:two-set:thm:state-change-diff-term}
    &= \big(\Vds(X_{t+1}, \Db_{t}, \Da_{t}) -  \Vds(X_{t}, \Db_{t}, \Da_{t})\big) \\
    \label{eq:two-set:thm:set-change-diff-term}
    &\mspace{20mu} + \big(\Vds(X_{t+1}, \Db_{t+1}, \Da_{t+1}) - \Vds(X_{t+1}, \Db_{t}, \Da_{t})).
\end{align}
We refer to the difference term in right-hand side of \eqref{eq:two-set:thm:set-change-diff-term} the \emph{set-update term}, and the difference term in \eqref{eq:two-set:thm:state-change-diff-term} the \emph{state-transition term}. 
We calculate these terms separately. 

For the state-transition term in \eqref{eq:two-set:thm:state-change-diff-term}, it follows directly from the definition of $\Vds$ that 
\begin{align}
    \nonumber
     & \Vds(X_{t+1}, \Db_{t}, \Da_{t}) -  \Vds(X_{t}, \Db_{t}, \Da_{t})\\
     \label{eq:two-set:thm:state-transition-term-1}
     &\qquad \leq \big(\hu(X_{t+1}, \Db_t) - \hu(X_{t}, \Db_t)\big)+ \big(\hw(X_{t+1}, \Da_t) - \hw(X_{t}, \Da_t)\big)
\end{align}

For the set-update term in \eqref{eq:two-set:thm:set-change-diff-term}, we can derive a bound fully in terms of the sizes of sets, using the Lipschitz continuity of $\hw(x, D)$ and $\hu(x, D)$ with respect to $D$ given in \eqref{eq:hw:lipschitz} of \Cref{lem:hw-hu-properties}:
\begin{align}
    \nonumber
    &\Vds(X_{t+1}, \Db_{t+1}, \Da_{t+1}) - \Vds(X_{t+1}, \Db_{t}, \Da_{t})  \\
    \nonumber
    &\qquad = \hu(X_{t+1}, \Db_{t+1}) - \hu(X_{t+1}, \Db_{t}) + \hw(X_{t+1}, \Da_{t+1}) - \hw(X_{t+1}, \Da_{t}) \\
    \nonumber
    &\qquad \mspace{20mu} - \Lds (m(\Db_{t+1}) - m(\Db_{t})) \\
    \label{eq:two-set:thm:set-update-term-1}
    &\qquad \leq 2\lamu^{1/2} \big(m(\Db_{t+1}\backslash \Db_t) + m(\Db_t \backslash \Db_{t+1})\big) + 2\lamw^{1/2} \big(m(\Da_{t+1}\backslash \Da_t) + m(\Da_t \backslash \Da_{t+1})\big)\\
    &\qquad \mspace{20mu} - \Lds (m(\Db_{t+1}\backslash \Db_t) - m(\Db_{t}\backslash \Db_{t+1})).
\end{align}
Next, we show that 
\begin{equation}
    \label{eq:two-set:thm:da-diff-bound}
    m(\Da_{t+1}\backslash \Da_t) + m(\Da_t \backslash \Da_{t+1}) \leq 2m(\Db_{t+1}\backslash\Db_t) - \beta \big(m(\Db_{t+1}) - m(\Db_t)\big) + \frac{1}{N}.
\end{equation}
We discuss based on whether $\Da_{t+1} \supseteq \Da_{t} \backslash \Db_{t+1}$ or $\Da_{t+1} \subseteq \Da_{t} \backslash \Db_{t+1}$. 
\begin{itemize}
    \item If $\Da_{t+1} \supseteq \Da_{t} \backslash \Db_{t+1}$, because $\Da_t$ is disjoint from $\Db_t$, it is not hard to see that
    \begin{equation}
        \label{eq:two-set:thm:set-update-term-2} 
        \Da_{t}\backslash \Da_{t+1} \subseteq \Db_{t+1}\backslash \Db_t,  \quad m(\Da_{t}\backslash \Da_{t+1}) \leq m(\Db_{t+1}\backslash \Db_t).
    \end{equation}
    Moreover, by the definition of $\Da_t$, $m(\Da_t) \geq \beta (1-m(\Db_{t})) - 1/N$ and  $m(\Da_{t+1}) \leq \beta (1-m(\Db_{t+1}))$, so 
    \begin{align}
        \nonumber
        m(\Da_{t+1}\backslash \Da_{t}) &= m(\Da_{t}\backslash \Da_{t+1}) + m(\Da_{t+1}) - m(\Da_{t}) \\
        \label{eq:two-set:thm:set-update-term-3} 
        &\leq m(\Db_{t+1}\backslash \Db_t) - \beta\big(m(\Db_{t+1}) - m(\Db_t)\big) + \frac{1}{N}. 
    \end{align}
    Combining \eqref{eq:two-set:thm:set-update-term-2} and \eqref{eq:two-set:thm:set-update-term-3}, we get \eqref{eq:two-set:thm:da-diff-bound}.
    \item If $\Da_{t+1} \subseteq \Da_{t} \backslash \Db_{t+1}$, we have $m(\Da_{t+1} \backslash \Db_{t}) = 0$ and 
    \begin{align}
        \nonumber
        m(\Da_{t} \backslash \Da_{t+1}) &=  m(\Da_{t}) - m(\Da_{t+1})  \\
        \nonumber
        &=  \lfloor\beta (1-m(\Db_{t}))\rfloor - \lfloor\beta (1-m(\Db_{t+1}))\rfloor \\
        \nonumber
        &\leq \beta \big(m(\Db_{t+1}) - m(\Db_t)\big) + \frac{1}{N},
    \end{align}
    which implies \eqref{eq:two-set:thm:da-diff-bound} because $m(\Db_{t+1}\backslash\Db_t) \geq m(\Db_{t+1}) - m(\Db_t)$ and $\beta < 1$. 
\end{itemize}
Plugging \eqref{eq:two-set:thm:da-diff-bound} into \eqref{eq:two-set:thm:set-update-term-1}
and rearranging the terms, we get an upper bound for the set-update term:
\begin{align}
     \Vds(X_{t+1}, \Db_{t+1}, \Da_{t+1}) - \Vds(X_{t+1}, \Db_{t}, \Da_{t}) 
     \label{eq:two-set:thm:set-update-term-4} 
    \leq 4\big(\lamu^{1/2} + \lamw^{1/2}\big) m(\Db_t \backslash \Db_{t+1}) + \frac{2\lamw^{1/2}}{N}, 
\end{align}
where the terms involving to $m(\Db_{t+1}\backslash \Db_t)$ have been canceled out. 

Substituting the above calculation into
\eqref{eq:two-set:thm:state-change-diff-term} and \eqref{eq:two-set:thm:set-change-diff-term}, we get
\begin{align}
    \nonumber
    & \Vds(\Fullstate_{t+1}) - \Vds(\Fullstate_t) \\
    \label{eq:two-set:thm:v-diff-term-1}
    &\qquad \leq \big(\hu(X_{t+1}, \Db_t) - \hu(X_{t}, \Db_t)\big)+ \big(\hw(X_{t+1}, \Da_t) - \hw(X_{t}, \Da_t)\big)\\
    \label{eq:two-set:thm:v-diff-term-2}
    &\qquad \mspace{20mu} + 4\big(\lamu^{1/2} + \lamw^{1/2}\big) m(\Db_t \backslash \Db_{t+1}) + \frac{2\lamw^{1/2}}{N}. 
\end{align}
By \Cref{lem:two-set:sufficient-coverage}, we have
\begin{align}
    \nonumber
    \Vds(\Fullstate_t) 
    &\leq  \hu(X_{t}, \Db_t) +  \hw(X_{t}, \Da_t) + \Lds\bigg(\KscOne \hw(X_t, \Da_t) + \frac{\KscTwo}{N}\bigg) \\
    \nonumber
    &\leq K_{Vh} \cdot \big((1-\rhou) \hu(X_{t}, \Db_t) + (1-\rhow) \hw(X_{t}, \Da_t)\big) + \frac{\Lds\KscTwo}{N},
\end{align}
for some constant $K_{Vh} > 0$. Taking $\rhoFinal = 1/K_{Vh}$, and multiplying both sides by $\rhoFinal$, we get
\begin{equation}
    \label{eq:v-h-bound}
    \rhoFinal \Vds(\Fullstate_t)  \leq (1-\rhou) \hu(X_{t}, \Db_t) + (1-\rhow) \hw(X_{t}, \Da_t) +  \frac{\rhoFinal\Lds\KscTwo}{N}.
\end{equation}
Adding \eqref{eq:v-h-bound} and the bound in \eqref{eq:two-set:thm:v-diff-term-1}--\eqref{eq:two-set:thm:v-diff-term-2}, we have proved the preliminary bound claimed in \eqref{eq:two-set:thm:v-diff-final-term-1}--\eqref{eq:two-set:thm:v-diff-final-term-2}.

\medskip
\textbf{Proof of \eqref{eq:Vds-mean}.}
Next, we prove \eqref{eq:Vds-mean}, restated here for the ease of reference: there exists a constant $\KVds$ independent of $N$, such that for any $\fullstate\in\fullssp$, we have
\begin{equation}
    \tag{\ref{eq:Vds-mean}}
    \E{\Bigl(\Vds(\Fullstate_{t+1}) - \rhoFinal\,\Vds(\Fullstate_t) \Bigr)^{\!+} \given \Fullstate_t = \fullstate} \leq \frac{\KVds}{\sqrt{N}}.
\end{equation}
To prove \eqref{eq:Vds-mean}, we apply \eqref{eq:two-set:thm:v-diff-final-term-1}--\eqref{eq:two-set:thm:v-diff-final-term-2} to get
\begin{align*}
    &\E{\Bigl(\Vds(\Fullstate_{t+1}) - \rhoFinal\,\Vds(\Fullstate_t) \Bigr)^{\!+} \given \Fullstate_t = \fullstate} \\
    &\qquad\leq     \big(\hu(X_{t+1}, \Db_t) - \rhou \hu(X_{t}, \Db_t)\big)^{+} + \big(\hw(X_{t+1}, \Da_t) - \rhow \hw(X_{t}, \Da_t)\big)^{+} \\
    &\qquad \mspace{20mu} + 4\big(\lamu^{1/2} + \lamw^{1/2}\big) m(\Db_t \backslash \Db_{t+1}) + \frac{2\lamw^{1/2}}{N} \\
    &\qquad\leq \frac{\KhuOne}{\sqrt{N}} + \frac{\KhwOne}{\sqrt{N}} + \frac{4\big(\lamu^{1/2} + \lamw^{1/2}\big) \KnsOne}{\sqrt{N}} +  \frac{2\lamw^{1/2}}{N},
\end{align*}
where the last step is due to \Cref{lem:hw-hu-properties} and \Cref{lem:two-set:non-shrink}. This proves \eqref{eq:Vds-mean}.

\medskip
\textbf{Proof of \eqref{eq:Vds-tail}.}
Finally, we prove \eqref{eq:Vds-tail}, restated below: there exists constants $\KVdst, \CVds>0$ independent of $N$, such that
\begin{equation}
    \tag{\ref{eq:Vds-tail}}
    \E{\Bigl(\Vds(\Fullstate_{t+1}) - \rhoFinal\,\Vds(\Fullstate_t) - \tfrac{(1-\rhoFinal)\threshbar}{2}\Bigr)^{\!+} \given \Fullstate_t} \leq \KVdst\exp(-\CVds N) \quad a.s.
\end{equation}
Let $r_0 \triangleq (1-\rhoFinal)\threshbar/2$ and let $M$ be a finite upper bound on $\Vds$ over $\fullssp$. Since $\bigl(\Vds(\Fullstate_{t+1}) - \rhoFinal\Vds(\Fullstate_t) - r_0\bigr)^{\!+} \leq M$ a.s., we have
\begin{equation}
    \label{eq:Vds-tail:pf-prob-to-exp}
    \E{\Bigl(\Vds(\Fullstate_{t+1}) - \rhoFinal\Vds(\Fullstate_t) - r_0\Bigr)^{\!+} \given \Fullstate_t}
    \leq M \cdot \ProbBig{\Vds(\Fullstate_{t+1}) > \rhoFinal \Vds(\Fullstate_t) + r_0 \givenBig \Fullstate_t} \quad a.s.,
\end{equation}
so it suffices to show the probability on the right decays as $\exp(-CN)$ for some $C > 0$ independent of $N$ and $\fullstate$.

By \eqref{eq:two-set:thm:v-diff-final-term-1}--\eqref{eq:two-set:thm:v-diff-final-term-2} and the union bound,
\begin{align}
    \nonumber
    &\ProbBig{\Vds(\Fullstate_{t+1}) > \rhoFinal \Vds(\Fullstate_t) + r_0\givenBig \Fullstate_t} \\
    \nonumber
    &\qquad \leq \Probbig{\hu(X_{t+1}, \Db_t) > \rhou \hu(X_{t}, \Db_t) + r_0/4 \givenbig \Fullstate_t} \\
    \nonumber
    &\qquad \mspace{20mu} + \Probbig{\hw(X_{t+1}, \Da_t) > \rhow \hw(X_{t}, \Da_t) + r_0/4 \givenbig \Fullstate_t} \\
    \label{eq:Vds-tail:pf-union-bound}
    &\qquad \mspace{20mu} + \Probbig{4\big(\lamu^{1/2} + \lamw^{1/2}\big) m(\Db_t \backslash \Db_{t+1})  > r_0/4 \givenbig \Fullstate_t} + \indibrac{2\lamw^{1/2}/N > r_0/4}.
\end{align}
We bound each term on the right-hand side.

By \eqref{eq:hu:high-prob} and \eqref{eq:hw:high-prob} of \Cref{lem:hw-hu-properties},
\begin{align}
    \label{eq:Vds-tail:pf-hu-bound}
    \Probbig{\hu(X_{t+1}, \Db_t) > \rhou \hu(X_{t}, \Db_t) + r_0/4 \givenbig \Fullstate_t} &\leq \KhuTwo \exp\!\Big(-\tfrac{\Chu}{16} N r_0^2\Big) \quad a.s., \\
    \label{eq:Vds-tail:pf-hw-bound}
    \Probbig{\hw(X_{t+1}, \Da_t) > \rhow \hw(X_{t}, \Da_t) + r_0/4 \givenbig \Fullstate_t} &\leq \KhwTwo \exp\!\Big(-\tfrac{\Chw}{16} N r_0^2\Big) \quad a.s.
\end{align}

For the third term in \eqref{eq:Vds-tail:pf-union-bound}, set $c \triangleq r_0/\big(16(\lamu^{1/2}+\lamw^{1/2})\big)$. If $m(\Db_t) \geq c$, \Cref{lem:two-set:non-shrink} gives
\begin{align}
    \nonumber
    \Probbig{4(\lamu^{1/2}+\lamw^{1/2}) m(\Db_t\backslash\Db_{t+1}) > r_0/4 \givenbig \Fullstate_t}
    &\leq \Probbig{m(\Db_t\backslash\Db_{t+1}) > 0 \givenbig \Fullstate_t} \\
    \nonumber
    &\leq \KnsTwo \exp(-\Cns N m(\Db_t)^2) \\
    \label{eq:Vds-tail:pf-ns-bound}
    &\leq \KnsTwo \exp(-\Cns N c^2).
\end{align}
If $m(\Db_t) < c$, then $4(\lamu^{1/2}+\lamw^{1/2}) m(\Db_t\backslash\Db_{t+1}) \leq 4(\lamu^{1/2}+\lamw^{1/2}) m(\Db_t) < r_0/4$, so the event is empty and the probability is zero.

Finally, the indicator $\indibrac{2\lamw^{1/2}/N > r_0/4}$ vanishes for $N > 8\lamw^{1/2}/r_0$, a fixed threshold independent of $\fullstate$.

Combining \eqref{eq:Vds-tail:pf-union-bound}, \eqref{eq:Vds-tail:pf-hu-bound}, \eqref{eq:Vds-tail:pf-hw-bound}, and \eqref{eq:Vds-tail:pf-ns-bound}, there exist constants $\tilde K, \tilde C > 0$ independent of $N$ and $\fullstate$ such that, for $N$ above a fixed threshold,
\begin{equation}
    \label{eq:Vds-tail:pf-prob-bound}
    \ProbBig{\Vds(\Fullstate_{t+1}) > \rhoFinal \Vds(\Fullstate_t) + r_0 \givenBig \Fullstate_t} \leq \tilde K \exp\!\big(-\tilde C N\big) \quad a.s.
\end{equation}
Substituting into \eqref{eq:Vds-tail:pf-prob-to-exp} and setting $\KVdst = M\tilde K$ and $\CVds = \tilde C$ establishes \eqref{eq:Vds-tail}.
\end{proof}

It remains to prove the three supporting lemmas stated at the beginning of this subsection.

\subsubsection{Proofs of supporting lemmas}

Next, we prove \Cref{lem:two-set:non-shrink}, the almost non-shrinking property of $\Db_t$.

\begin{proof}[Proof of \Cref{lem:two-set:non-shrink}]
    Without loss of generality, we assume that $\Db_t \neq \emptyset$, because otherwise the two inequalities stated in the lemma hold automatically.
    By the definition of $\Db_{t+1}$ in \Cref{alg:two-set}, $\Db_{t+1} \supsetneq \Db_t$ only when $\slk(X_{t+1}, \Db_t) < 0$; when this happens, $m(\Db_t \backslash \Db_{t+1}) \leq m(\Db_t)$. Therefore,
    \begin{align}
        \label{eq:non-shrink:mean-interm-bound}
        \Ebig{(m(\Db_t\backslash \Db_{t+1}))^+ \givenbig \Fullstate_t}
        &\leq m(\Db_t) \Probbig{\slkb(X_{t+1}, \Db_t) < 0 \givenbig \Fullstate_t} \\
        \label{eq:non-shrink:tail-interm-bound}
        \Probbig{m(\Db_{t}\backslash \Db_{t+1}) > 0 \givenbig \Fullstate_t}
        &= \Probbig{\slkb(X_{t+1}, \Db_t) < 0 \givenbig \Fullstate_t}.
    \end{align}
    Therefore, it remains to bound $\Prob{\slkb(X_{t+1}, \Db_t) < 0 \givenbig \Fullstate_t}$. Recall that $\slkb(x, D) = \eta m(D) -   \normplain{x(D)- m(D) \statdist}_\umat - \errfe = \eta m(D) - \hu(x, D) - \errfe$, so
    \begin{align}
        \nonumber
        \Prob{\slkb(X_{t+1}, \Db_t) < 0 \givenbig \Fullstate_t}
        &= \Prob{\hu(X_{t+1},\Db_t) > \eta m(\Db_t) -\errfe \givenbig \Fullstate_t},
    \end{align}
    When $\Db_t \neq \emptyset$, we have
    $\slkb(X_t, \Db_t) \geq 0$, implying $\hu(X_t, \Db_t) \leq \eta m(\Db_t) - \errfe$.
    Therefore,
    \begin{align}
        \nonumber
        & \Prob{\hu(X_{t+1}, \Db_t) \geq \eta m(\Db_t)} \\
        \nonumber
        &\qquad \leq \Prob{\hu(X_{t+1}, \Db_t) - \rhou \hu(X_t, \Db_t) \geq (1-\rhou) \eta m(\Db_t) - (1-\rhou) \errfe} \\
        \label{eq:pf-almost-nonshrink:interm-1}
        &\qquad \leq \KhuTwo \exp\!\big(-\Chu N (1-\rhou)^2 (\eta m(\Db_t)-\errfe)^2\big),
    \end{align}
    where the last inequality comes from \eqref{eq:hu:high-prob} of \Cref{lem:hw-hu-properties} with the constants $\KhuTwo, \Chu$ defined therein. To bound the exponent in \eqref{eq:pf-almost-nonshrink:interm-1}, observe that $1-\rhou=\Theta(1), \eta = \Theta(1)$ and $\errfe = O(1/N)$, so
    \begin{align*}
        \Chu N (1-\rhou)^2(\eta m(\Db_t)-\errfe)^2
        &\geq \Chu N (1-\rhou)^2 \Big(\eta^2 m(\Db_t)^2 - 2\eta m(\Db_t) \errfe \Big) \\
        &\geq \Chu N (1-\rhou)^2 \eta^2  m(\Db_t)^2 - 2 \Chu (1-\rhou)^2 \eta\, N\errfe.
    \end{align*}
    Since $N\errfe = O(1)$, the second term in the last expression is bounded by a constant $K_0$, so
    \[
        \Prob{\slkb(X_{t+1}, \Db_t) < 0 \givenbig \Fullstate_t} \leq \KhuTwo\, e^{K_0} \exp\!\big(-\Cns N m(\Db_t)^2\big),
    \]
    with $\Cns = \Chu (1-\rhou)^2 \eta^2$. Combining this bound with the earlier bounds \eqref{eq:non-shrink:mean-interm-bound} and \eqref{eq:non-shrink:tail-interm-bound} for $\E{m(\Db_t\backslash \Db_{t+1})}$ and $\Prob{m(\Db_{t}\backslash \Db_{t+1}) > 0}$ yields
    \begin{align*}
        \Ebig{m(\Db_t\backslash \Db_{t+1}) \givenbig \Fullstate_t} &\leq \KnsOne\, m(\Db_t) \exp(-\Cns N m(\Db_t)^2) \quad a.s. \\
        \Probbig{m(\Db_{t}\backslash \Db_{t+1}) > 0 \givenbig \Fullstate_t} &\leq \KnsTwo \exp(-\Cns N m(\Db_t)^2) \quad a.s.,
    \end{align*}
    with $\KnsOne = \KnsTwo = \KhuTwo\, e^{K_0}$.
\end{proof}

Finally, we prove \Cref{lem:two-set:sufficient-coverage}, the sufficient coverage property.

\begin{proof}[Proof of \Cref{lem:two-set:sufficient-coverage}]
    Observe that \eqref{eq:two-set:sufficient-coverage:goal-2} follows directly from the definition of the policy, so we only need to prove \eqref{eq:two-set:sufficient-coverage:goal-1}.
    We claim that either $m(\Da_t) \leq \errtol$, or
    \begin{equation}
        \label{eq:two-set:sufficient-coverage:interm-goal-1}
        \lamu^{1/2} \hw(X_t, \Da_t) >  \eta m(\Da_t) - \errtol - \errfe.
    \end{equation}
    We first show \eqref{eq:two-set:sufficient-coverage:goal-1} assuming this claim.
    Recall that $m(\Da_{t})= \big\lfloor \beta (N -|\Db_{t}|)\big\rfloor \big/ N$. If $m(\Da_t) \leq \errtol$,
    \[
        \beta(1-m(\Db_t)) - \frac{1}{N} \leq m(\Da_t) \leq \errtol,
    \]
    so $1-m(\Db_t) \leq 1/(\beta N) + \errtol/\beta$, which implies \eqref{eq:two-set:sufficient-coverage:goal-1} by the non-negativity of $\hw(X_t, \Da_t)$.
    If $m(\Da_t) > \errtol$, \eqref{eq:two-set:sufficient-coverage:interm-goal-1} holds. Because $m(\Da_t) = \floor{\beta (N - |\Db_t|)} /N \geq \big(\beta (N - |\Db_t|) - 1\big) / N$, we have
    \begin{align*}
        \lamu^{1/2} \hw(X_t, \Da_t) &>  \eta \big(\beta(N -|\Db_{t}|)-1\big)\frac{1}{N} -  \errtol - \errfe \\
        &=  \eta \beta(1-m(\Db_t)) -  \frac{\eta}{N} - \errtol - \errfe.
    \end{align*}
    Rearranging gives
    \[
        1 - m(\Db_t) < \frac{\lamu^{1/2}}{\eta\beta} \hw(X_t, \Da_t) + \frac{1}{\beta N} + \frac{\errtol+\errfe}{\eta\beta}.
    \]
    Combining the two cases, we have
    \[
        1 - m(\Db_t) \leq \frac{\lamu^{1/2}}{\eta\beta} \hw(X_t, \Da_t) + \frac{1}{\beta N} + \frac{\errtol+\errfe}{\min(\eta,1)\beta}.
    \]
    Since $\errtol, \errfe = O(1/N)$, the additive term is $O(1/N)$, giving the bound $\KscOne \hw(X_t, \Da_t) + \KscTwo/N$ for positive constants $\KscOne, \KscTwo$ independent of $N$, establishing \eqref{eq:two-set:sufficient-coverage:goal-1}.

    Now we prove the claim by contradiction. Suppose, at a certain time $t$, we have $m(\Da_t) > \errtol$ and $ \lamu^{1/2}\hw(X_t, \Da_t) \leq \eta m(\Da_t) - \errtol - \errfe$. Because $\norm{v}_\umat \leq \lamu^{1/2}\norm{v}_2 \leq \lamu^{1/2}\norm{v}_\wmat$ for any $v\in \R^{|\sspa|}$,
    \begin{align}
        \nonumber
        \normbig{X_t(\Da_t) - m(\Da_t) \statdist}_\umat
        &\leq \lamu^{1/2}\normbig{X_t(\Da_t) - m(\Da_t) \statdist}_\wmat
        \leq  \eta m(\Da_t) - \errtol - \errfe.
    \end{align}
    Combined with the triangular inequality and the definition of $\Db_t$, we have
    \begin{align*}
        \normbig{X_t(\Db_t\cup\Da_t) - m(\Db_t\cup\Da_t) \statdist}_\umat &\leq \normbig{X_t(\Db_t) - m(\Db_t) \statdist}_\umat  + \normbig{X_t(\Da_t) - m(\Da_t) \statdist}_\umat \\
        &\leq \eta m(\Db_t) + \eta m(\Da_t) - \errtol - \errfe.
    \end{align*}
    Consequently, $\Db_t\cup\Da_t$ is a superset of $\Db_t$ such that $\slkb(X_t, \Db_t\cup\Da_t) \geq \errtol$ and $m(\Db_t \cup \Da_t) = m(\Db_t) + m(\Da_t) > m(\Db_t) + \errtol$, contradicting the $\errtol$-maximality of $\Db_t$.
    We have thus proved the claim that implies \eqref{eq:two-set:sufficient-coverage:goal-1}.
\end{proof}

\subsection{Proof of \hcref{lem:sketch:inst-reward}}\label{sec:two-set:pf-inst-reward}

\begin{proof}
    By \eqref{eq:sufficient-cond-for-OL-assertion}, when $\norm{x([N]) - \statdist}_\umat \leq \threshbar \triangleq \eta - \errfe$, \lppriority is feasible for all arms. Since the two-set policy chooses $\Db$ to be a maximal set that can follow \lppriority, we have $\Db=[N]$ and $\Da=\emptyset$. 

    By definition, 
    \begin{equation}
        r^\pi(\fullstate) = r^\pi(x, [N], \emptyset) = \sumsa r(s,a) \Ebig{Y_t(s,a)  \givenbig X_t, \Db_t = [N], \Da_t = \emptyset},
    \end{equation}
    where $t$ is an arbitrary time step, and $Y_t(s,a)$ is the random variable representing the fraction of arms in state $s$ taking action $a$ at time $t$. 
    By the definition of \lppriority, 
    \begin{align}
        \Ebig{Y_t(s,a) \givenbig X_t, \Db_t = [N], \Da_t = \emptyset} &= \pibs(a|s) X_t([N], s) \quad s\neq \sneu \\ 
        \Ebig{Y_t(\sneu,1) \givenbig X_t, \Db_t = [N], \Da_t = \emptyset, } &= \alpha - \sum_{s\neq \sneu} \pibs(1|s) X_t([N], s)  \\
        \Ebig{Y_t(\sneu,0) \givenbig X_t, \Db_t = [N], \Da_t = \emptyset, } &= 1 - \alpha - \sum_{s\neq \sneu} \pibs(0|s) X_t([N], s). 
    \end{align} 
    Therefore,
    \begin{equation}\label{eq:lp-priority-inst-reward}
        \begin{aligned}
        r^\pi(X_t, [N], \emptyset) 
        &= \sum_{s\neq \sneu, a\in\aspa} \big(r(s,a) - r(\sneu,a)\big) \pibs(a|s) X_t([N],s) \\
        &\mspace{20mu} + \alpha r(\sneu,1) + (1-\alpha) r(\sneu,0),
        \end{aligned}
    \end{equation}
    For any $v\in\simplex(\sspa)$, we define $\rhat(v)$ as
    \[
        \rhat(v) \triangleq \sum_{s\neq \sneu, a\in\aspa} \big(r(s,a) - r(\sneu,a)\big) \pibs(a|s) v(s) + \alpha r(\sneu,1) + (1-\alpha) r(\sneu,0).
    \]
    Then $\rhat(\cdot)$ is affine $v$, and $r^\pi(\fullstate) = \rhat(X_t([N]))$. 
\end{proof}

\subsection{Proof of \hcref{lem:drift-Vds-truncate-term}}\label{sec:two-set:pf-drift-truncate}

\begin{proof}
We use the abbreviation that $a \triangleq \Vds(\fullstate)$, $b \triangleq \Vds(\fullstate')$, and $r_0 \triangleq (1-\rhoFinal)\threshbar/2$.
The key algebraic identity
\begin{equation}\label{eq:trunc-key-identity}
    b - \frac{\threshbar}{2} = \Bigl(b - \rhoFinal a - r_0\Bigr) + \rhoFinal\Bigl(a - \frac{\threshbar}{2}\Bigr),
\end{equation}
together with subadditivity $(x+y)^+ \leq x^+ + y^+$, yields
\begin{equation}\label{eq:trunc-unified}
    \Bigl(b-\frac{\threshbar}{2}\Bigr)^{\!+} - \Bigl(a-\frac{\threshbar}{2}\Bigr)^{\!+}
    \leq  \Bigl(b - \rhoFinal a - r_0\Bigr)^{\!+} - (1-\rhoFinal)\Bigl(a-\frac{\threshbar}{2}\Bigr)^{\!+}.
\end{equation}
We derive the claim from \eqref{eq:trunc-unified} by cases on $a$.

\medskip
\noindent\textbf{Case 1: $a \leq \threshbar$.}
Since $(1-\rhoFinal)(a-\threshbar/2)^+ \geq 0$, \eqref{eq:trunc-unified} gives
\[
    \Bigl(b-\frac{\threshbar}{2}\Bigr)^{\!+} - \Bigl(a-\frac{\threshbar}{2}\Bigr)^{\!+}
    \leq  \Bigl(b - \rhoFinal a - r_0\Bigr)^{\!+},
\]
which matches the RHS of \eqref{eq:drift-Vds-truncate} because $\indibrac{a>\threshbar}=0$.

\medskip
\noindent\textbf{Case 2: $a > \threshbar$.}
Then $(a-\threshbar/2)^+ = a - \threshbar/2 > \threshbar/2$, so $(1-\rhoFinal)(a-\threshbar/2)^+ > (1-\rhoFinal)\threshbar/2 = r_0$.
Hence \eqref{eq:trunc-unified} gives
\[
    \Bigl(b-\frac{\threshbar}{2}\Bigr)^{\!+} - \Bigl(a-\frac{\threshbar}{2}\Bigr)^{\!+}
    \leq  \Bigl(b - \rhoFinal a - r_0\Bigr)^{\!+} - r_0,
\]
which matches the RHS of \eqref{eq:drift-Vds-truncate} because $\indibrac{a>\threshbar}=1$.
\end{proof}

\subsection{Proof of \hcref{lem:drift-local-approx-term}}\label{sec:two-set:pf-drift-local}

\begin{proof}
    
    \begin{align}
        \Delta \locapprox(\fullstate) 
        \nonumber
        &= \Ebig{(\statdist - X_{t+1}([N]))\,Q\gvec \givenbig \Fullstate_t = \fullstate} - \big(\statdist- x([N])\big)\,Q\gvec  \\
        \nonumber
        &= \Ebig{(\statdist - X_{t+1}([N]))\,Q\gvec \givenbig \Fullstate_t = \fullstate} - \big(\statdist - x([N])\big)\,(\Phi - \Phi + I) Q\gvec \\
        \nonumber
        &= \Ebig{(\statdist - X_{t+1}([N]))\,Q\gvec \givenbig \Fullstate_t = \fullstate} - \big(\statdist - x([N])\big)\,\Phi Q\gvec  - \big(\statdist - x([N])\big) \gvec \\
        \label{eq:drift-locapprox:step-1}
        &= \left(\Ebig{(\statdist - X_{t+1}([N]))\,Q\gvec \givenbig \Fullstate_t = \fullstate} - \big(\statdist - x([N])\big)\,\Phi Q\gvec\right)  - \big(\rrel - \rhat(x([N]))\big).
    \end{align}

    We claim that the first difference term in \eqref{eq:drift-locapprox:step-1} is bounded as:
    \begin{equation}\label{eq:drift-locapprox-claim}
        \Ebig{(\statdist - X_{t+1}([N]))\,Q\gvec \givenbig \Fullstate_t=\fullstate} - (\statdist - x([N]))\,\Phi Q\gvec \leq 4\rmax \lamq \indibrac{\Db\neq [N]}.
    \end{equation}
    We verify \eqref{eq:drift-locapprox-claim} by cases on $\Db$.
    \begin{itemize}
        \item Case $\Db=[N]$.
        By \Cref{lem:lp-priority-and-phi}, $\E{X_{t+1}([N])-\statdist \given \Fullstate_t=\fullstate} = (x([N])-\statdist)\Phi$, so
        \[
            \Ebig{(\statdist-X_{t+1}([N]))Q\gvec \givenbig \Fullstate_t=\fullstate} = (\statdist-x([N]))\Phi Q\gvec,
        \]
        and \eqref{eq:drift-locapprox-claim} holds with both sides equal to zero.

        \item Case $\Db\neq[N]$.
        By Jensen's inequality and the triangle inequality,
        \begin{align*}
            &\Ebig{(\statdist-X_{t+1}([N]))Q\gvec \givenbig \Fullstate_t=\fullstate} - (\statdist-x([N]))\Phi Q\gvec \\
            &\qquad \leq \Ebig{\bigl|(X_{t+1}([N])-\statdist)Q\gvec\bigr| \givenbig \Fullstate_t}
                 + \bigl|(x([N])-\statdist)\Phi Q\gvec\bigr| \\
            &\qquad \leq \Ebig{\norm{(X_{t+1}([N])-\statdist)Q}_1 \givenbig \Fullstate_t}\norm{\gvec}_\infty
                 + \norm{(x([N])-\statdist)\Phi Q}_1\norm{\gvec}_\infty \\
            &\qquad \leq 2\lamq\norm{\gvec}_\infty \\
            &\qquad \leq 4\rmax\lamq,
        \end{align*}
        where the second last step uses $X_{t+1}([N]),x([N])\in\Delta(\sspa)$ and the definition of $\lamq$, and the last step uses $\norm{\gvec}_\infty \leq 2\rmax$.
    \end{itemize}

    Substituting \eqref{eq:drift-locapprox-claim} into \eqref{eq:drift-locapprox:step-1}, we get
    \begin{align*}
        \Delta \locapprox(\fullstate) 
        &\leq -(\rrel - \rhat(x([N]))) + 4\rmax\lamq\indibrac{\Db\neq [N]} \\
        &\leq  -(\rrel - \rhat(x([N]))) + 4\rmax\lamq\indibrac{\norm{x([N])-\statdist}_\umat > \threshbar}.
    \end{align*}
    where the last step is due to $\indibrac{\Db\neq [N]} \leq \indibrac{\norm{x([N])-\statdist}_\umat > \threshbar}$. 
    Because $\rhat(x([N])) = r^\pi(\fullstate)$ when $\norm{x([N])-\statdist}_\umat > \threshbar$ and $\rhat(x([N])) \leq r^\pi(\fullstate) + 2\rmax$ otherwise, 
    we get
    \[
        \Delta \locapprox(\fullstate) \leq  -(\rrel - r^\pi(\fullstate)) + (4\rmax\lamq+2\rmax)\indibrac{\norm{x([N])-\statdist}_\umat > \threshbar}. \qedhere
    \]
\end{proof}

\section{Proof of the lower bound (\hcref{thm:instability-lower-bound})}
\label{sec:lower-bound-proof}

We prove \Cref{thm:instability-lower-bound}, the lower bound on $\rrel - \rsysn$ for regular unstable RBs (\Cref{def:regular-unstable}). 

Here is a brief overview of the proof.
Given any policy $\pi$, we define a key quantity, $\disy(Y^\pi_t)$, that measures how much the state-action count vector deviates from that prescribed by \lppriority in non-neutral states.
Up to a constant factor, the long-run average expectation of $\disy(Y^\pi_t)$ lower bounds the gap to the LP bound, $\rrel-\rsysn$ (\Cref{lem:opt-gap-linear-comp}). A small value of $\disy(Y^\pi_t)$ also ensures that, on a large subset of arms, the expected state count vector obeys the affine transition identity associated with \lppriority, up to a rounding error (\Cref{lem:exist-subset-lp-priority}).
As \lppriority is locally unstable in regular unstable RBs, we use Lyapunov analysis to show that the long-run average expectation of $\disy(Y^\pi_t)$ is $\Omega(1/\sqrt{N})$, thereby obtaining the same order of lower bound on $\rrel-\rsysn$.
The Lyapunov analysis is developed through \Cref{lem:uus-ust-properties,lem:hus-drift,lem:hst-drift} and the proof of \Cref{thm:instability-lower-bound} at the end of this section.

We begin by relating $\rrel-\rsysn$ to the long-run average expectation of $\disy(Y^\pi_t)$.

\begin{restatable}{lemma}{optgaplinearcomp}\label{lem:opt-gap-linear-comp}
   Suppose that the LP relaxation \eqref{eq:lp-single} has a unique optimal solution $y^*$, and $y^*(s,1) + y^*(s,0) > 0,$ $\forall s\in\sspa$. Then
   \vspace{-8pt}
    \begin{equation}
        \rrel - \rsysn \geq \rewardgap \limsup_{T\to\infty} \frac{1}{T} \sum\nolimits_{t=0}^{T-1} \Ebig{\disy(Y_t^{\pi})},
    \end{equation}
    where $\rewardgap$ is a positive constant; $\disy$ is a function of state-action distribution given by
    \[
        \disy(y) = \sum\nolimits_{s\in S^+} y(s,0)+ \sum\nolimits_{s\in S^-}y(s,1),
    \]
    where $S^+ = \{s\in\sspa\colon y^*(s,1) > 0, y^*(s,0) = 0\}$, $S^- = \{s\in\sspa\colon y^*(s,1) = 0, y^*(s,0) > 0\}$.
\end{restatable}

\begin{proof}%
    We first investigate the dual of \eqref{eq:lp-single}, whose properties help us prove the lemma:
    \begin{align}
        \tag{Dual-LP}\label{eq:lp-dual} \underset{\lambda, \mu, \{f(s)\}_{s\in\sspa}}{\text{minimize}} &  \mu \\
        \text{subject to}\mspace{12mu}
        \label{eq:dual-lp-constraint}
        &  r(s,a) +  (\alpha - \indibrac{a=1}) \lambda - \mu  + \sum_{s'} P(s,a,s')f(s') \leq f(s)  \quad \forall s\in\sspa, a\in\aspa \\
        \nonumber
        & \lambda, \mu, f(s) \in \R \quad \forall s\in\sspa
    \end{align}
    Note that the primal problem \eqref{eq:lp-single} is always feasible and bounded, so \eqref{eq:lp-dual} is also feasible and bounded.
    By strong duality of linear programs, the optimal value of the dual problem is $\mu^* = \rrel$.
    Moreover, there always exists an optimal solution of the dual problem $(\lambda^*, \vlam, \mu^*)$ that satisfies strict complementarity with the unique primal optimal solution $y^*$ \citep{GolTuc_56_strict_cs}:
    \begin{equation}
        \label{eq:c-s}
        r(s,a) - \lambda^* \indibrac{a=1} + \lambda^* \alpha - \mu^* + \sum_{s'} P(s,a,s')\vlam(s') - \vlam(s) = 0 \quad \text{ iff } \quad y^*(s,a) \neq 0.
    \end{equation}
    Let $\rlam(s,a) = r(s,a) -  \lambda^* \indibrac{a=1} + \lambda^* \alpha$.
    By \eqref{eq:c-s} and the definition of $S^+$ and $S^-$,
    there exists a constant $\rewardgap > 0$ such that
    \begin{align}
        \label{eq:lb:scs-1}
         \rlam(s, 0) - \rrel + \sum_{s'} P(s,0,s')\vlam(s') &\leq \vlam(s) - \rewardgap   \quad \forall s\in S^+ \\
        \label{eq:lb:scs-2}
          \rlam(s, 1) - \rrel + \sum_{s'} P(s,1,s')\vlam(s') &\leq \vlam(s) - \rewardgap \quad \forall s\in S^-.
    \end{align}
    Define the function $\qlam\colon \sspa\times \aspa \to \R$ as
    \[
        \qlam(s,a) = \rlam(s,a) - \rrel + \sum_{s'} P(s,a,s')\vlam(s').
    \]
    Then \eqref{eq:dual-lp-constraint} implies that $\vlam(s) \geq \qlam(s,a)$ for any $s\in\sspa,a\in\aspa$;
    \eqref{eq:lb:scs-1} and \eqref{eq:lb:scs-2} imply that $\qlam(s,1) - \qlam(s,0)$ is at least $\rewardgap$ for $s\in S^+$ and at most $-\rewardgap$ for $s\in S^-$.

    We next compare $\rrel$ with the expected reward per arm at time $t$, retaining the changes in $\vlam(S_t(i))$ that will cancel when summed over time. For any policy $\pi$, initial state vector $\veS_0$, and time $t$, we have
    \begin{align}
        \nonumber
        &\rrel-\frac{1}{N}\sumN\E{r(S_t(i),A_t(i))}
        -\frac{1}{N}\sumN\E{\vlam(S_{t+1}(i))-\vlam(S_t(i))}\\
        \label{eq:lb:apply-rate-conservation}
        &\qquad = \frac{1}{N}\sumN\EBig{\rrel-\rlam(S_t(i),A_t(i))
        -\sum_{s'\in\sspa}P(S_t(i),A_t(i),s')\vlam(s')+\vlam(S_t(i))}\\
        \nonumber
        &\qquad = \frac{1}{N}\sumN\EBig{-\qlam(S_t(i),A_t(i))+\vlam(S_t(i))}\\
        \label{eq:lb:apply-action-gap}
        &\qquad \geq \frac{\rewardgap}{N}\sumN\EBig{\indibrac{S_t(i)\in S^+,A_t(i)=0}
        +\indibrac{S_t(i)\in S^-,A_t(i)=1}}\\
        &\qquad = \rewardgap\Ebig{\disy(Y_t)}.
    \end{align}
    In \eqref{eq:lb:apply-rate-conservation}, we use the identity $\sumN\rlam(S_t(i),A_t(i))=\allowbreak\sumN r(S_t(i),A_t(i))$ and the tower property for $\vlam(S_{t+1}(i))$. Both identities hold for any history-dependent policies.

    To obtain the long-run reward bound, we sum the inequality in the preceding display over $t=0,\ldots,T-1$ and divide by $T$ to get
    \begin{align*}
        &\rrel-\frac{1}{TN}\sum_{t=0}^{T-1}\sumN\E{r(S_t(i),A_t(i))} - \frac{1}{TN}\sumN\E{\vlam(S_T(i))-\vlam(S_0(i))} \\
        &\qquad\geq\frac{\rewardgap}{T}\sum_{t=0}^{T-1}\Ebig{\disy(Y_t)}.
    \end{align*}
    Since $\vlam$ is bounded on $\sspa$ and the long-run average reward exists by our standing assumption on policies, the left-hand side converges to $\rrel-\rsysn$ as $T\to\infty$. Taking limsup on both sides proves the lemma.
\end{proof}

Next, we show that there is a subset $D_t$ on which all non-neutral arms take the actions prescribed by \lppriority and exactly $\floor{\alpha|D_t|}$ arms are active. This subset is large when $\disy(Y_t)$ is small. We bound both the next-step state count vector's expectation and its deviation from that expectation.

\begin{lemma}\label{lem:exist-subset-lp-priority}
    Consider an instance of regular unstable RBs (\Cref{def:regular-unstable}). Under any policy $\pi$, at any time $t$, there exists a set $D_t\subseteq[N]$ depending on $X_t$ and $\veA_t$ with size $m(D_t) \geq 1 - 2\disy(Y_t)/\beta - 1/(\beta N)$, where $\beta = \min(\alpha,1-\alpha)$,
    such that for every $i\in D_t$, $A_t(i)=1$ if $S_t(i)\in S^+$ and $A_t(i)=0$ if $S_t(i)\in S^-$, and $\sum_{i\in D_t}A_t(i)=\floor{\alpha|D_t|}$.
    For such $D_t$, we have
    \begin{itemize}[leftmargin=1em]
        \item The expectation of $X_{t+1}(D_t)$ satisfies
        \begin{equation}
        \label{eq:subset-lp-priority}
            \normbig{\Ebig{X_{t+1}(D_t) - m(D_t) \statdist \givenbig X_t, \veA_t} - (X_t(D_t)- m(D_t)\statdist)\Phi}_1 \leq \frac{2}{N}.
        \end{equation}
        \item The expected distance of $X_{t+1}(D_t)$ from its conditional expectation satisfies
        \begin{equation}
            \label{eq:subset-transition-fluctuation}
            \Ebig{\normbig{X_{t+1}(D_t)-\Ebig{X_{t+1}(D_t)\givenbig X_t,\veA_t}}_1\givenbig X_t,\veA_t}
            \leq \frac{|\sspa|\sqrt{|D_t|}}{2N}\leq\frac{|\sspa|}{2\sqrt N}.
        \end{equation}
    \end{itemize}
\end{lemma}

\begin{proof}%
    We construct the subset $D_t$ by removing a suitable subset of arms from $[N]$.
    We first remove each arm $i$ with state $S_t(i)\in S^+$ and $A_t(i) = 0$ or $S_t(i) \in S^-$ and $A_t(i) = 1$. After this step, we get a subset $D'$ with $N(1-\disy(Y_t))$ arms. The budget usage of $D'$ equals $\alpha N - N \sum_{s\in S^-} Y_t(s,1)$.
    Then $D_t$ is obtained by removing arms that take action $1$ or action $0$ from $D'$ such that the budget usage of $D_t$ equals $\floor{\alpha |D_t|}$.
    If $\alpha N - N \sum_{s\in S^-} Y_t(s,1)  > \alpha |D'|$, it is not hard to show that the number of arms removed from $D'$ is bounded by
    \[
        |D'| - |D_t|  \leq \frac{1}{1-\alpha} \Big(\absbig{\alpha N - N \sum_{s\in S^-} Y_t(s,1) - \alpha |D'| }  + 1\Big) \leq  \frac{1}{1-\alpha} \big(N\disy(Y_t) + 1\big);
    \]
    If $\alpha N - N \sum_{s\in S^-} Y_t(s,1) \leq \alpha |D'|$, it is not hard to show that the number of arms removed from $D'$ is bounded by
    \[
         |D'| - |D_t| \leq \frac{1}{\alpha} \Big(\absbig{\alpha N - N \sum_{s\in S^-} Y_t(s,1) - \alpha |D'|}  + 1\Big) \leq  \frac{1}{\alpha} \big(N\disy(Y_t) + 1\big).
    \]
    In both cases, we have
    \begin{equation}
        |D'| - |D_t|  \leq \frac{1}{\beta}\big(N\disy(Y_t) + 1\big),
    \end{equation}
    where recall that $\beta = \min(\alpha, 1-\alpha)$.
    Now we get a subset $D_t$ such that $m(D_t) \geq 1 - \disy(Y_t)- (\disy(Y_t)+ 1/N) / \beta \geq 1-2\disy(Y_t)/\beta - 1/(\beta N)$. Moreover, for each $i\in D_t$, $A_t(i) = 1$ if $S_t(i) \in S^+$; $A_t(i) = 0$ if $S_t(i) \in S^-$; the total budget usage of $D_t$ equals $\floor{\alpha |D_t|}$.

    We now derive the subset analogue of the affine transition identity in \Cref{lem:lp-priority-and-phi} (see also \citealp{GasGauYan_23_whittles}). The number of active arms in the neutral state is $\floor{\alpha |D_t|} - N\sum_{s\neq\sneu} \pibs(1|s)X_t(D_t,s)$. Summing the expected transitions over all states gives
    \[
        \E{X_{t+1}(D_t) \givenplain X_t, \veA_t}
        = X_t(D_t) \Big(P_\pibs - \costvec^\top \big(P_1(\sneu) - P_0(\sneu)\big) \Big) + \frac{\floor{\alpha |D_t|}}{N} (P_1(\sneu) - P_0(\sneu)).
    \]
    whereas by the definition of $\Phi$,
    \[
        (X_t(D_t) - m(D_t)\statdist) \Phi + m(D_t) \statdist = X_t(D_t) \Big(P_\pibs - \costvec^\top \big(P_1(\sneu) - P_0(\sneu)\big) \Big) + m(D_t) \alpha (P_1(\sneu) - P_0(\sneu)).
    \]
    The two right-hand sides differ by $[\floor{\alpha |D_t|}/N - \alpha m(D_t)](P_1(\sneu) - P_0(\sneu))$. Since $\abs{\floor{\alpha |D_t|} - \alpha |D_t|}<1$ and $\norm{P_1(\sneu) - P_0(\sneu)}_1\leq 2$, this proves \eqref{eq:subset-lp-priority}.

    To prove \eqref{eq:subset-transition-fluctuation}, condition on $X_t,\veA_t$, which determines $D_t$. For each state $s$, $X_{t+1}(D_t,s)$ is $1/N$ times a sum of $|D_t|$ independent Bernoulli variables, one for each arm's transition into $s$. Its conditional variance is therefore at most $|D_t|/(4N^2)$. By Cauchy--Schwarz,
    \begin{align*}
        &\Ebig{\absbig{X_{t+1}(D_t,s)-\Ebig{X_{t+1}(D_t,s)\givenbig X_t,\veA_t}}\givenbig X_t,\veA_t}\\
        &\qquad\leq\sqrt{\operatorname{Var}\big(X_{t+1}(D_t,s)\givenbig X_t,\veA_t\big)}
        \leq\frac{\sqrt{|D_t|}}{2N}.
    \end{align*}
    Summing over $s\in\sspa$ proves the bound, including when $D_t$ is empty.
\end{proof}

Next, we conduct Lyapunov analysis with the help of \Cref{lem:exist-subset-lp-priority}.  To construct Lyapunov functions, we first define two seminorms.
Consider a real Jordan decomposition of $\Phi$, $\Phi = V J V^{-1}$, using real blocks for conjugate pairs of nonreal eigenvalues. We split $J$ into $J = \Jus + \Jst$, where $\Jus$ retains the blocks of $J$ with eigenvalues' moduli greater than $1$ and is zero on the other blocks, while $\Jst$ retains the blocks with eigenvalues' moduli smaller than $1$ and is zero on the other blocks.
Let $\Ius$ be the diagonal matrix equal to the identity on the coordinates of the unstable blocks and zero elsewhere; let $\Ist$ be the identity matrix subtracting $\Ius$.
We define the matrices $\Uus$ and $\Ust$ via the infinite series:
\begin{align}
    \label{eq:Uus-def}
    \Uus &= \sum_{k=1}^\infty \big(V \Jus^{\dagger} V^{-1}\big)^k \big(V \Jus^{\dagger} V^{-1}\big)^{\top k}  \\
    \label{eq:Ust-def}
    \Ust &= V \Ist V^{-1}  \big(V \Ist V^{-1}\big)^{\top} + \sum_{k=1}^\infty \big(V \Jst V^{-1} \big)^k \big(V \Jst V^{-1}\big)^{\top k}
\end{align}
where the superscript $\dagger$ denotes the Moore--Penrose pseudo-inverse. Thus $\Jus^\dagger$ inverts each unstable block and is zero on the stable coordinates.

The next lemma shows the well-definedness of $\Uus$ and $\Ust$ and the properties of their seminorms. Its proof is given in \Cref{app:proof-uus-ust}.

\begin{restatable}{lemma}{uusustproperties}\label{lem:uus-ust-properties}
    For any regular unstable RBs (\Cref{def:regular-unstable}),
    the matrices $\Uus$ and $\Ust$ given in \eqref{eq:Uus-def} and \eqref{eq:Ust-def} are well-defined and positive semidefinite.
    Further, if we let $\norm{v}_{\Uus} = \sqrt{v \Uus v^\top}$ and $\norm{v}_{\Ust} = \sqrt{v \Ust v^\top}$ for any $v\in \R^{|\sspa|}$, then
    \begin{align}
        \label{eq:Uus-expansion}
        \norm{v \Phi}_{\Uus} &\geq (1 + \gapus) \norm{v}_{\Uus} \\
        \label{eq:Uus-expansion-max}
        \norm{v \Phi}_{\Uus} &\leq (1 + \gapusmax) \norm{v }_{\Uus} \\
        \label{eq:Ust-contraction}
        \norm{v \Phi}_{\Ust} &\leq (1-\gapst) \norm{v }_{\Ust},
    \end{align}
    where $\gapus,\gapusmax>0$ and $0<\gapst\leq1$ are constants.
    Moreover, there exists a positive constant $\Knorm$ such that for any $v\in \R^{|\sspa|}$,
    \begin{equation}
        \label{eq:Uus-plus-Ust-strength}
         \norm{v}_{\Uus} +  \norm{v}_{\Ust}  \geq \Knorm \norm{v}_1.
    \end{equation}
\end{restatable}

Now we are ready to conduct the Lyapunov analysis.
We define the Lyapunov functions $\hst$ and $\hus$: for any system state $x$ and $D\subseteq[N]$, let
\begin{align}
    \hst(x, D) &=  \norm{x(D) - m(D)\statdist}_{\Ust}  \\
     \hus(x, D) &=  \norm{x(D) - m(D)\statdist}_{\Uus}.
\end{align}
Our next lemma gives a lower bound for the drift of $\hus(x, D)$ under \emph{any policy}.

\begin{lemma}
    \label{lem:hus-drift}
    Consider an instance of regular unstable RBs (\Cref{def:regular-unstable}). Under any policy $\pi$, for all sufficiently large $N$ and any time step $t$, we have
    \begin{align}
        \nonumber
          & \Ebig{\hus(X_{t+1}, [N]) \givenbig X_t, \veA_t} - \hus(X_t, [N])\\
          &\qquad \geq \frac{\gapus K_1- (2+\gapus + \gapusmax) K_1 \indibrac{\goodevent_t^c}}{\sqrt{N}}-\frac{\lamus^{1/2}}{\beta}\left((16+2\gapusmax)\disy(Y_t)+\frac{6+\gapusmax}{N}\right) \quad a.s.,
    \end{align}
    where $\lamus$ is the largest eigenvalue of $\Uus$ and $\beta = \min(\alpha, 1-\alpha)$; the event $\goodevent_t$ is given by
    \[
        \goodevent_t = \Big\{\hus(X_t, [N]) > \frac{K_1}{\sqrt{N}} + \frac{2\lamus^{1/2} \disy(Y_t)}{\beta} + \frac{\lamus^{1/2}}{\beta N} \;\text{or}\; \norm{X_t([N]) - \statdist}_1 \leq \radiusus -  \frac{6\disy(Y_t)}{\beta} - \frac{3}{\beta N}\Big\},
    \]
    and $K_1 = \lamus^{1/2} \noiseortho / (2+\gapus+\gapusmax)$. Here we use a radius $\radiusus\leq \frac12\min\{y^*(\sneu,0),y^*(\sneu,1)\}$ in Property~\ref{assump:orthogonal-noise}, decreasing it if necessary.
\end{lemma}

\begin{proof}%
    Let $D_t$ be the set given in \Cref{lem:exist-subset-lp-priority}. By the definition of $\hus$ and the size bound $m(D_t) \geq 1-2\disy(Y_t)/\beta - 1/(\beta N)$, we have
    \begin{align}
        \nonumber
        & \Ebig{\hus(X_{t+1}, [N]) \givenbig X_t, \veA_t} - \hus(X_t, [N])  \\
        \label{eq:pf-hus-drift:interm-0}
        &\qquad \geq  \Ebig{\hus(X_{t+1}, D_t)  \givenbig X_t, \veA_t} - \hus(X_t, D_t) - 4\lamus^{1/2} (1-m(D_t))  \\
        \nonumber
        &\qquad \geq  \Ebig{\hus(X_{t+1}, D_t)  \givenbig X_t, \veA_t} - \hus(X_t, D_t) - 8 \lamus^{1/2} \disy(Y_t) / \beta - 4\lamus^{1/2} /(\beta N).
    \end{align}
    Note that \eqref{eq:pf-hus-drift:interm-0} uses the definition of $\hus$ and the fact that $\norm{v}_{\Uus} \leq \lamus^{1/2} \norm{v}_{1}$.

    First suppose that $\hus(X_t,D_t)>K_1/\sqrt N$. Then
    \begin{align*}
         \Ebig{\hus(X_{t+1}, D_t) \givenbig X_t, \veA_t}
         &\geq \normbig{(X_t(D_t)-m(D_t)\statdist)\Phi}_{\Uus}
             -\frac{2\lamus^{1/2}}{N}\\
         &\geq (1+\gapus)\hus(X_t,D_t)-\frac{2\lamus^{1/2}}{N}\\
         &\geq \hus(X_t, D_t) + \frac{\gapus K_1}{\sqrt{N}} - \frac{2\lamus^{1/2}}{N},
    \end{align*}
    The first inequality uses Jensen's inequality and \eqref{eq:subset-lp-priority}, with $\norm{v}_{\Uus}\leq\lamus^{1/2}\norm{v}_1$ giving the term $2\lamus^{1/2}/N$. The second uses the norm lower bound \eqref{eq:Uus-expansion}; the third uses $\hus(X_t,D_t)>K_1/\sqrt N$. Substituting the last line into \eqref{eq:pf-hus-drift:interm-0} and using $\beta\leq1$ yields
    \begin{equation}
        \label{eq:pf-hus-drift:goal-1}
        \Ebig{\hus(X_{t+1}, [N]) \givenbig X_t, \veA_t} - \hus(X_t, [N]) \geq \frac{\gapus K_1}{\sqrt{N}} - \frac{6\lamus^{1/2}}{\beta N} -  \frac{8 \lamus^{1/2} \disy(Y_t)}{\beta}.
    \end{equation}

    Next, suppose that $\hus(X_t,D_t)\leq K_1/\sqrt N$, and write
    \[
        Z_D=X_{t+1}(D_t)-m(D_t)\statdist-(X_t(D_t)-m(D_t)\statdist)\Phi.
    \]
    We express the drift on $D_t$ in terms of $\norm{Z_D}_{\Uus}$:
    \begin{align}
        \nonumber
        & \Ebig{\hus(X_{t+1},D_t)\givenbig X_t,\veA_t}-\hus(X_t,D_t)\\
        \nonumber
        &\qquad \geq \Ebig{\norm{Z_D}_{\Uus}\givenbig X_t,\veA_t}
        -\normbig{(X_t(D_t)-m(D_t)\statdist)\Phi}_{\Uus}-\hus(X_t,D_t)\\
        \label{eq:pf-hus-drift:interm-3}
        &\qquad \geq \Ebig{\norm{Z_D}_{\Uus}\givenbig X_t,\veA_t}-(2+\gapusmax)\hus(X_t,D_t).
    \end{align}
    The first inequality is the triangle inequality; the second uses the norm upper bound \eqref{eq:Uus-expansion-max} in \Cref{lem:uus-ust-properties}.
    To lower bound $\Ebig{\norm{Z_D}_{\Uus}\givenbig X_t,\veA_t}$, we construct a reference transition under \lppriority on $[N]$, on the event
    \[
        B_t=\{\norm{X_t([N])-\statdist}_1\leq\radiusus\}.
    \]
    For sufficiently large $N$, our choice of $\radiusus$ ensures that \lppriority is feasible on $B_t$. We set $\widetilde A_t(i)=1$ for $S_t(i)\in S^+$ and $\widetilde A_t(i)=0$ for $S_t(i)\in S^-$; activate exactly $\alpha N-\sum_{i\in[N]}\indibrac{S_t(i)\in S^+}$ neutral arms, retaining as many actual actions as possible and breaking ties by arms' IDs; set all other neutral actions to zero.

    These reference actions differ from the actual actions on exactly $N\disy(Y_t)$ non-neutral arms. Since both action vectors use budget $\alpha N$, the number of neutral arms whose actions differ is
    \[
        \sum_{i:S_t(i)=\sneu}\abs{\widetilde A_t(i)-A_t(i)}
        =N\absBig{\sum_{s\in S^-}Y_t(s,1)-\sum_{s\in S^+}Y_t(s,0)}
        \leq N\disy(Y_t).
    \]
    Thus at most $2N\disy(Y_t)$ arms have different actions in the two vectors.

    Let $\widetilde S_{t+1}(i)$ have transition probabilities $P(S_t(i),\widetilde A_t(i),\cdot)$, coupled with $S_{t+1}(i)$ to agree whenever the actions agree. Take transition pairs independently across arms conditional on $X_t,\veA_t$, and let $\widetilde X_{t+1}$ be the resulting reference state-count vectors. Each arm with different actions contributes at most $2/N$ to the $\ell_1$ distance, so
    \begin{equation}
        \label{eq:pf-hus-drift:interm-4}
        \norm{\widetilde X_{t+1}([N])-X_{t+1}([N])}_1\leq4\disy(Y_t)
        \quad\text{almost surely on }B_t.
    \end{equation}
    Outside $B_t$, set $\widetilde X_{t+1}=X_{t+1}$.

    We now apply Property~\ref{assump:orthogonal-noise} of \Cref{def:regular-unstable} to the reference transition. Define
    \[
        \widetilde Z_N=\widetilde X_{t+1}([N])-\statdist-(X_t([N])-\statdist)\Phi.
    \]
    The reference transition satisfies
    \begin{equation}
        \label{eq:pf-hus-drift:interm-5}
        \Ebig{\normbig{\widetilde Z_N}_{\Uus}\givenbig X_t,\veA_t}
        \geq\frac{\lamus^{1/2}\noiseortho}{\sqrt N}
        \quad\text{on }B_t.
    \end{equation}
    To see this, let $\xi^\top$ be a unit eigenvector of $\Uus$ corresponding to $\lamus$, so $\norm{v}_{\Uus}\geq\lamus^{1/2}\abs{v\xi^\top}$. The column space of $\Uus$ is contained in that of $\Phi$; since $\statdist\Phi=0$, we have $\statdist\Uus=0$ and $\statdist\xi^\top=0$. Conditional on $X_t,\veA_t$, the reference actions are determined and follow \lppriority on $B_t$, so Property~\ref{assump:orthogonal-noise} gives \eqref{eq:pf-hus-drift:interm-5}.

    To compare $\norm{Z_D}_{\Uus}$ with $\normbig{\widetilde Z_N}_{\Uus}$, subtract the definitions of $Z_D$ and $\widetilde Z_N$:
    \begin{align*}
        Z_D-\widetilde Z_N
        ={}&X_{t+1}([N])-\widetilde X_{t+1}([N])\\
        &-\big(X_{t+1}(D_t^c)-m(D_t^c)\statdist\big)\\
        &+\big(X_t(D_t^c)-m(D_t^c)\statdist\big)\Phi.
    \end{align*}
    For the fixed complement $D_t^c=[N]\setminus D_t$, nonnegativity of the count vectors and $\statdist\Uus=0$ imply, for $u=t,t+1$,
    \[
        \norm{X_u(D_t^c)-m(D_t^c)\statdist}_{\Uus}
        =\norm{X_u(D_t^c)}_{\Uus}\leq\lamus^{1/2}m(D_t^c).
    \]
    It follows that
    \[
        \norm{Z_D-\widetilde Z_N}_{\Uus}\leq4\lamus^{1/2}\disy(Y_t)
        +(2+\gapusmax)\lamus^{1/2}m(D_t^c)
        \quad\text{on }B_t.
    \]
    This inequality bounds the three terms in $Z_D-\widetilde Z_N$ using \eqref{eq:pf-hus-drift:interm-4}, the bound on $\hus(X_u,D_t^c)$, and \eqref{eq:Uus-expansion-max}. The triangle inequality and \eqref{eq:pf-hus-drift:interm-5} then yield
    \begin{align}
        \nonumber
        \Ebig{\norm{Z_D}_{\Uus}\givenbig X_t,\veA_t}
        &\geq\frac{\lamus^{1/2}\noiseortho}{\sqrt N}\indibrac{B_t}\\
        \label{eq:pf-hus-drift:subset-noise}
        &\quad-4\lamus^{1/2}\disy(Y_t)
        -(2+\gapusmax)\lamus^{1/2}m(D_t^c).
    \end{align}
    Inequality \eqref{eq:pf-hus-drift:subset-noise} also holds outside $B_t$, where its right-hand side is nonpositive.

    Substituting \eqref{eq:pf-hus-drift:subset-noise} into \eqref{eq:pf-hus-drift:interm-3}, using $\hus(X_t,D_t)\leq K_1/\sqrt N$, and then applying \eqref{eq:pf-hus-drift:interm-0} gives
    \begin{align}
        \nonumber
        & \Ebig{\hus(X_{t+1},[N])\givenbig X_t,\veA_t}-\hus(X_t,[N])\\*
        \label{eq:pf-hus-drift:goal-2}
        &\qquad \geq\frac{\lamus^{1/2}\noiseortho\indibrac{B_t}-(2+\gapusmax)K_1}{\sqrt N}
           -4\lamus^{1/2}\disy(Y_t)
        -(6+\gapusmax)\lamus^{1/2}m(D_t^c).
    \end{align}
    We now derive the lemma's inequality on $\goodevent_t$ and $\goodevent_t^c$ from \eqref{eq:pf-hus-drift:goal-1} and \eqref{eq:pf-hus-drift:goal-2}. First, we bound the sums subtracted after their $1/\sqrt N$ terms by the expression subtracted in the lemma. For \eqref{eq:pf-hus-drift:goal-2}, \Cref{lem:exist-subset-lp-priority} and $\beta\leq1$ give
    \begin{align}
        \nonumber
        &4\lamus^{1/2}\disy(Y_t)+(6+\gapusmax)\lamus^{1/2}m(D_t^c)\\*
        \label{eq:pf-hus-drift:subtracted-sum}
        &\qquad\leq\frac{\lamus^{1/2}}{\beta}
        \left((16+2\gapusmax)\disy(Y_t)+\frac{6+\gapusmax}{N}\right).
    \end{align}
    The sum $6\lamus^{1/2}/(\beta N)+8\lamus^{1/2}\disy(Y_t)/\beta$ in \eqref{eq:pf-hus-drift:goal-1} is also at most the right-hand side of \eqref{eq:pf-hus-drift:subtracted-sum}.

    Next, we determine the coefficient of $1/\sqrt N$ on each event. The triangle inequality and \Cref{lem:exist-subset-lp-priority} give
    \[
        \hus(X_t,[N])-\hus(X_t,D_t)
        \leq\hus(X_t,D_t^c)\leq\lamus^{1/2}m(D_t^c)
        \leq\frac{2\lamus^{1/2}\disy(Y_t)}{\beta}+\frac{\lamus^{1/2}}{\beta N}.
    \]
    Thus the first alternative in $\goodevent_t$ implies $\hus(X_t,D_t)>K_1/\sqrt N$, while the second implies $B_t$. Hence
    \[
        \goodevent_t\subseteq\Big\{\hus(X_t,D_t)>K_1/\sqrt N\Big\}\cup B_t.
    \]
    If $\hus(X_t,D_t)>K_1/\sqrt N$, we use \eqref{eq:pf-hus-drift:goal-1}, whose $1/\sqrt N$ coefficient is $\gapus K_1$. Otherwise, we use \eqref{eq:pf-hus-drift:goal-2}, whose coefficient is $\gapus K_1$ on $B_t$ by the choice of $K_1$, and $-(2+\gapusmax)K_1$ off $B_t$. The inclusion therefore gives coefficient $\gapus K_1$ on $\goodevent_t$, and at least $-(2+\gapusmax)K_1$ on $\goodevent_t^c$. Subtracting the right-hand side of \eqref{eq:pf-hus-drift:subtracted-sum} proves the lemma.
\end{proof}

In the next lemma, we upper bound the drift of $\hst(x, D)$ under \emph{any policy}.
\begin{lemma}
    \label{lem:hst-drift}
     Consider an instance of regular unstable RBs (\Cref{def:regular-unstable}). Under any policy $\pi$, for any time step $t$, we have
    \begin{equation}
        \Ebig{\hst(X_{t+1}, [N]) \givenbig X_t, \veA_t} \leq (1-\gapst) \hst(X_t, [N]) + \frac{\lamst^{1/2}|\sspa|}{\sqrt{N}} + \frac{10\lamst^{1/2}|\sspa|}{\beta N} +  \frac{8\lamst^{1/2} \disy(Y_t)}{\beta} \quad a.s.,
    \end{equation}
    where $\lamst$ is the largest eigenvalue of $\Ust$; $\beta = \min(\alpha, 1-\alpha)$.
\end{lemma}

\begin{proof}%
    Let $D_t$ be given by \Cref{lem:exist-subset-lp-priority}, and write $\overline X_{t+1}(D_t)=\Ebig{X_{t+1}(D_t)\givenbig X_t,\veA_t}$. Then
    \begin{align}
        \nonumber
        \Ebig{\hst(X_{t+1},D_t)\givenbig X_t,\veA_t}
        &\leq \normbig{\overline X_{t+1}(D_t)-m(D_t)\statdist}_{\Ust}
            +\frac{\lamst^{1/2}|\sspa|}{2\sqrt N}\\
        \nonumber
        &\leq \normbig{(X_t(D_t)-m(D_t)\statdist)\Phi}_{\Ust}
            +\lamst^{1/2}\left(\frac{|\sspa|}{2\sqrt N}+\frac{2}{N}\right)\\
        \label{eq:pf-hst-drift:interm-2}
        &\leq(1-\gapst)\hst(X_t,D_t)
            +\lamst^{1/2}\left(\frac{|\sspa|}{2\sqrt N}+\frac{2}{N}\right).
    \end{align}
    The first two inequalities use \eqref{eq:subset-transition-fluctuation} and \eqref{eq:subset-lp-priority}, respectively, together with the triangle inequality and $\norm{v}_{\Ust}\leq\lamst^{1/2}\norm{v}_1$. The third uses \eqref{eq:Ust-contraction} in \Cref{lem:uus-ust-properties}.

    To bound the drift on $[N]$, use $\hst(X_u,D_t^c)\leq2\lamst^{1/2}(1-m(D_t))$ for $u=t,t+1$. Since $0<\gapst\leq1$, the triangle inequality gives
    \begin{align}
        \nonumber
        & \Ebig{\hst(X_{t+1},[N])\givenbig X_t,\veA_t}-(1-\gapst)\hst(X_t,[N]) \\*
        \nonumber
        &\qquad \leq \Ebig{\hst(X_{t+1},D_t)\givenbig X_t,\veA_t}-(1-\gapst)\hst(X_t,D_t)
            +4\lamst^{1/2}(1-m(D_t)) \\*
        \label{eq:pf-hst-drift:interm-1}
        &\qquad \leq \frac{\lamst^{1/2}|\sspa|}{2\sqrt N}+\frac{6\lamst^{1/2}}{\beta N}
            +\frac{8\lamst^{1/2}\disy(Y_t)}{\beta},
    \end{align}
    using \eqref{eq:pf-hst-drift:interm-2}, the size bound in \Cref{lem:exist-subset-lp-priority}, and $\beta\leq1$. Since $|\sspa|\geq1$, this implies the stated bound.
\end{proof}

Now we are ready to prove \Cref{thm:instability-lower-bound}.

\begin{proof}[Proof of \Cref{thm:instability-lower-bound}]
    For any system state $x$ and subset $D\subseteq [N]$,
    define
    \[
        \hlb(x, D) = \ratiolb \hus(x, D) - \hst(x, D),
    \]
    where $\ratiolb = 2\lamst^{1/2}|\sspa| / (\gapus K_1)$.
    We will show that with this choice of $\ratiolb$, there are constants $a_{\mathrm{lb}},b_{\mathrm{lb}}>0$, independent of $N$ and $t$, such that for all sufficiently large $N$,
    \begin{align}
        \label{eq:pf-thm3:goal}
        \E{\hlb(X_{t+1}, [N])} - \E{\hlb(X_{t}, [N])}
        \geq \frac{a_{\mathrm{lb}}}{\sqrt N} - b_{\mathrm{lb}}\E{\disy(Y_t)} \quad \forall t\geq 0.
    \end{align}
    Suppose \eqref{eq:pf-thm3:goal} holds. Summing it over $t=0,\ldots,T-1$ and dividing by $T$ gives
    \[
        \frac{\E{\hlb(X_T,[N])}-\E{\hlb(X_0,[N])}}{T}
        \geq\frac{a_{\mathrm{lb}}}{\sqrt N}-\frac{b_{\mathrm{lb}}}{T}\sum_{t=0}^{T-1}\Ebig{\disy(Y_t)}.
    \]
    For each fixed $N$, the function $\hlb(\cdot,[N])$ is bounded, so the left-hand side tends to zero as $T\to\infty$. Rearranging and taking the liminf yields
    \[
        \liminf_{T\to\infty}\frac{1}{T}\sum_{t=0}^{T-1}\Ebig{\disy(Y_t)}
        \geq\frac{a_{\mathrm{lb}}}{b_{\mathrm{lb}}\sqrt N}.
    \]
    The corresponding limsup is at least this liminf. Thus \Cref{lem:opt-gap-linear-comp} gives $\rrel-\rsysn=\Omega(1/\sqrt N)$, proving \Cref{thm:instability-lower-bound}.

    To prove \eqref{eq:pf-thm3:goal}, combine \Cref{lem:hus-drift,lem:hst-drift} and collect terms:
    \begin{align}
        \nonumber
        & \E{\hlb(X_{t+1}, [N]) \givenbig X_t, \veA_t} - \hlb(X_{t}, [N]) \\
        \label{eq:pf-thm3:interm-2}
        &\qquad \geq \frac{\lamst^{1/2}|\sspa|}{\sqrt N} + \Big(\gapst \hst(X_t, [N]) - \frac{\ratiolb(2+\gapus+\gapusmax)K_1}{\sqrt{N}} \indibrac{\goodevent_t^c} \Big) \\
        \label{eq:pf-thm3:interm-3}
        &\qquad \mspace{20mu} - \big((16+2\gapusmax)\ratiolb \lamus^{1/2} + 8\lamst^{1/2}\big)\beta^{-1} \disy(Y_t)  - O\Big(\frac{1}{N}\Big).
    \end{align}
    Here $\ratiolb\gapus K_1-\lamst^{1/2}|\sspa|=\lamst^{1/2}|\sspa|>0$ by our choice of $\ratiolb$. All $O(1/N)$ terms below, including this one, denote nonnegative bounds with constants independent of $X_t,\veA_t$ and $t$.

    To bound the parentheses in \eqref{eq:pf-thm3:interm-2}, on $\goodevent_t^c$ we use
    \begin{align*}
        \hst(X_t, [N])
        &\geq \Knorm \norm{X_t([N]) - \statdist}_1 - \hus(X_t, [N]) \\
        &\geq \Knorm \radiusus - \frac{K_1}{\sqrt{N}} - \big(6\Knorm + 2\lamus^{1/2}\big)\beta^{-1} \disy(Y_t) - O\Big(\frac{1}{N}\Big).
    \end{align*}
    The two inequalities use \eqref{eq:Uus-plus-Ust-strength} and the definition of $\goodevent_t$, respectively. Consequently, on $\goodevent_t^c$ and for sufficiently large $N$,
    \begin{align}
        \nonumber
        & \gapst \hst(X_t, [N])  - \frac{\ratiolb(2+\gapus+\gapusmax)K_1}{\sqrt{N}} \indibrac{\goodevent_t^c} \\
        \label{eq:pf-thm3:interm-4}
        &\qquad \geq  - \gapst \big( 6\Knorm + 2\lamus^{1/2} \big) \beta^{-1} \disy(Y_t) - O\Big(\frac{1}{N}\Big).
    \end{align}
    Inequality \eqref{eq:pf-thm3:interm-4} follows from the lower bound for $\hst(X_t,[N])$ and $K_1(\gapst+\ratiolb(2+\gapus+\gapusmax))/\sqrt N\leq\gapst\Knorm\radiusus$ for sufficiently large $N$.

    On $\goodevent_t$, the left-hand side of \eqref{eq:pf-thm3:interm-4} is $\gapst\hst(X_t,[N])\geq0$, while its right-hand side is nonpositive. Thus \eqref{eq:pf-thm3:interm-4} holds on both events. Substituting \eqref{eq:pf-thm3:interm-4} into \eqref{eq:pf-thm3:interm-2}--\eqref{eq:pf-thm3:interm-3} gives
    \begin{align}
        \nonumber
        & \E{\hlb(X_{t+1}, [N]) \givenbig X_t, \veA_t} - \hlb(X_{t}, [N]) \\
        \nonumber
        &\qquad \geq \frac{\lamst^{1/2}|\sspa|}{\sqrt{N}}-b_{\mathrm{lb}}\disy(Y_t)-O\Big(\frac{1}{N}\Big),
    \end{align}
    where
    \[
        b_{\mathrm{lb}}=\big(6\gapst\Knorm+2\gapst\lamus^{1/2}
        +(16+2\gapusmax)\ratiolb\lamus^{1/2}+8\lamst^{1/2}\big)\beta^{-1}.
    \]
    For sufficiently large $N$, the $O(1/N)$ term is at most $\lamst^{1/2}|\sspa|/(2\sqrt N)$. Taking $a_{\mathrm{lb}}=\lamst^{1/2}|\sspa|/2$ and expectations over $X_t,\veA_t$ proves \eqref{eq:pf-thm3:goal} and finishes the proof.
\end{proof}

\suppressfloats[t]
\section{Experiments}
\label{sec:experiments}

\begin{figure}[t]
     \subcaptionbox{First non-UGAP instance. \label{fig:new2-eight-state-045}}{\includegraphics[height=5.5cm]{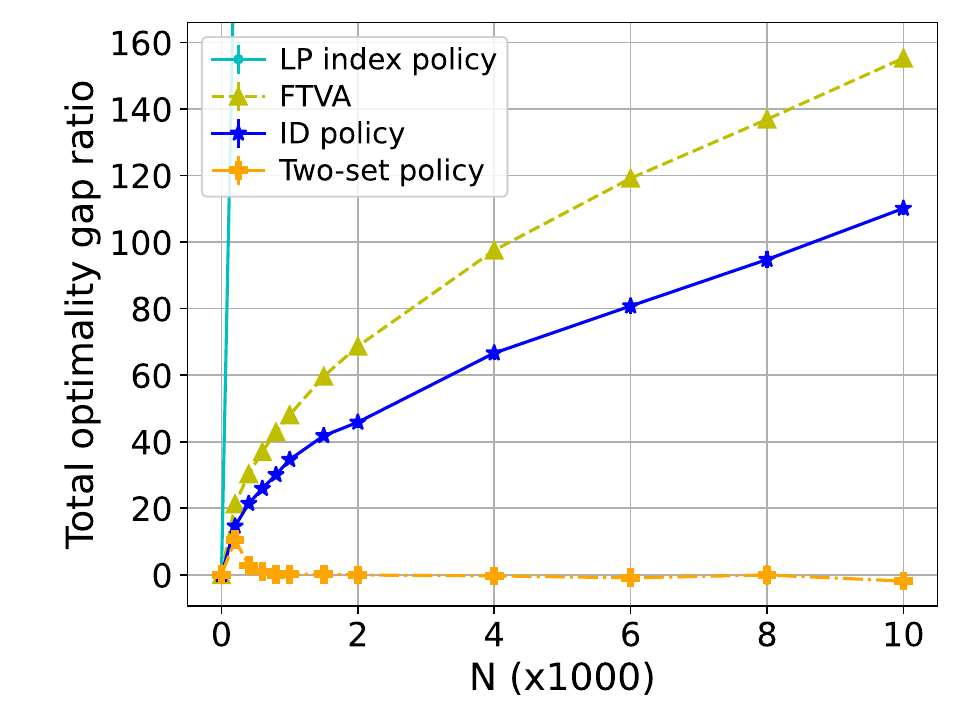}}
    \hfill
     \subcaptionbox{Second non-UGAP instance. \label{fig:conveyor-nd}}{\includegraphics[height=5.5cm]{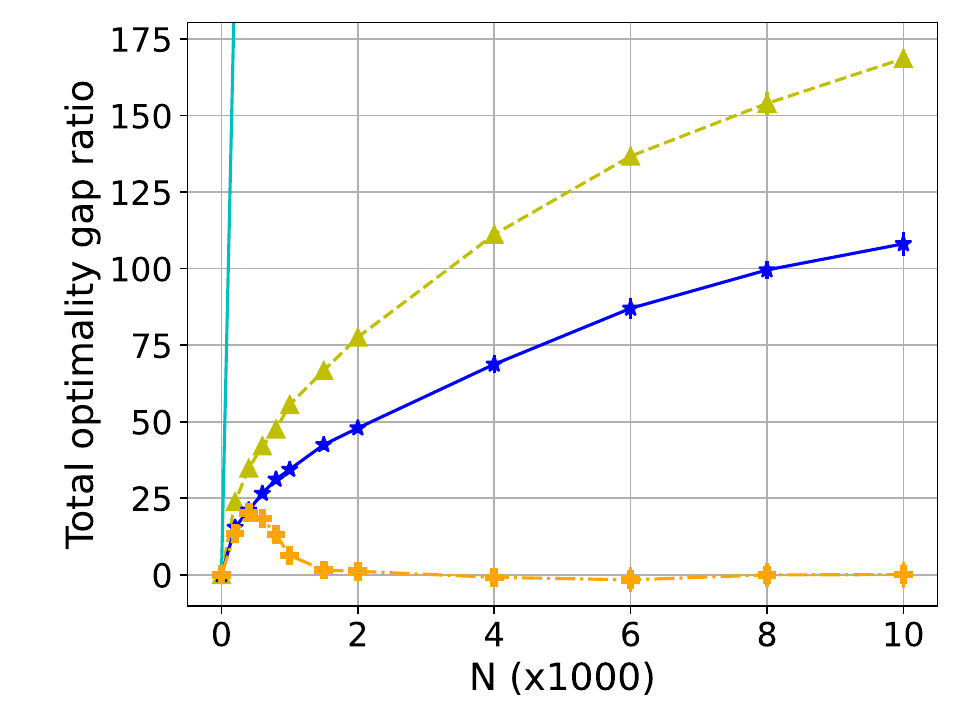}}
     \caption{RB instances satisfying Assumptions~\ref{assump:aperiodic-unichain}, \ref{assump:non-degeneracy}, and \ref{assump:local-stability}, but violating UGAP.}
     \label{fig:constructed-examples}
\end{figure}

In this section, we compare the performance of the two-set policy for RBs with those of the policies from prior work; these policies include two popular versions of LP-Priority policies, the Whittle index policy \citep{Whi_88_rb} and the LP index policy \citep{GasGauYan_23_exponential}, as well as two more recent policies, Follow-the-Virtual-Advice \citep{HonXieCheWan_23}, and the ID policy \citep{HonXieCheWan_24}.

We make the comparisons on two types of instances of the RB problem. The first type of RB instances satisfy all of our assumptions, namely, Assumptions~\ref{assump:aperiodic-unichain}, \ref{assump:non-degeneracy}, and \ref{assump:local-stability}, but violate UGAP. On this type of instances, no previous policy has been shown to achieve an exponential optimality gap.
The second type of RB instances are random instances generated from the uniform distribution (see \Cref{app:exp-details:examples}\ for details on how they are generated). On this type of instances, simulations in prior work have shown that the Whittle index and LP index policies often perform quite well \citep[see, e.g.][]{GasGauYan_23_exponential}; here we want to see if the two-set policy also performs reasonably well.
In addition to the two types of instances mentioned above, there also exist RB instances where one of our three assumptions fails.
In particular, when Assumptions~\ref{assump:non-degeneracy} or \ref{assump:local-stability} fails, the two-set policy could preserve an $O(1/\sqrt{N})$ optimality gap if the actions of the arms in $(\Db_t\cup\Da_t)^c$ are chosen appropriately, as argued in \Cref{remark:two-set-one-set} (See \Cref{app:exp-details:two-set} for implementation details). Moreover, in this case, the empirical performances of the two-set policy and the ID policy turn out to be indistinguishable, which we omit displaying here.

For the first type of instances, we consider two examples, which are adapted from \citep{HonXieCheWan_23}. The definitions of these two examples are given in \Cref{app:exp-details:examples}, where we also discuss potential methods to construct more instances of this type.
In Figure~\ref{fig:constructed-examples}, we consider the \emph{total optimality gap ratio} as the performance metric, which is defined as $N (1 - \rsysn/\rrel)$.
Under this metric, a horizontal line indicates that $1-\rsysn/\rrel=O(1/N)$, which in turn implies an $O(1/N)$ optimality gap.
In both examples, we find the Whittle index policy to be non-indexable.
The curves for the LP index policy grow linearly with a sharp slope, indicating its asymptotic suboptimality.
The curves for FTVA and the ID policy grow with $N$, consistent with their $O(1/\sqrt{N})$ theoretical bounds.
The curves for the two-set policy quickly converge to zero, demonstrating better-than-$O(1/N)$ optimality gaps.
On both examples, the two-set policy performs significantly better than the other policies simulated here.

\begin{figure}[t]
     \subcaptionbox{Total optimality gap ratios.\label{fig:uniform-1-total}}{\includegraphics[height=5.5cm]{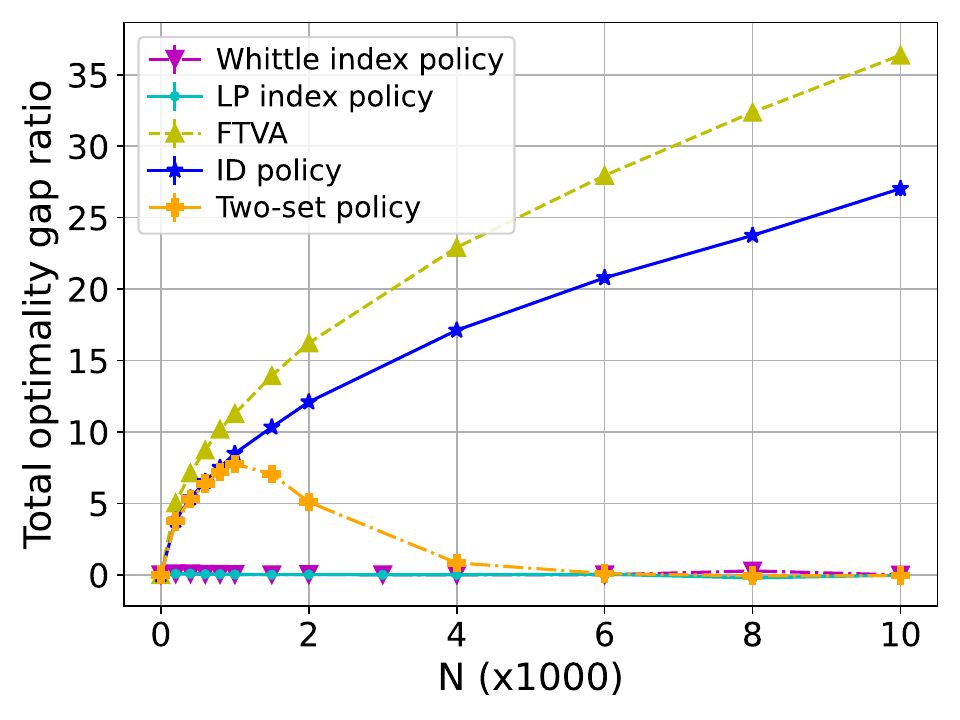}}
    \hfill
     \subcaptionbox{Log optimality gap ratios. \label{fig:uniform-1-log}}{\includegraphics[height=5.5cm]{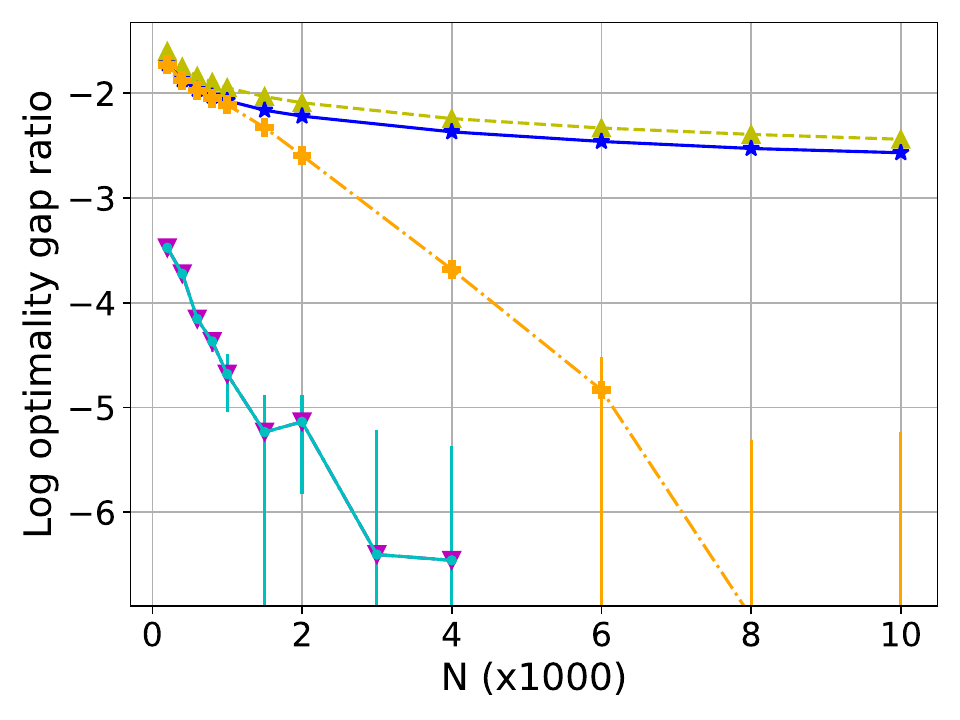}}
     \caption{First random instance sampled from uniform distribution. In \Cref{fig:uniform-1-log}, we do not include data for the Whittle index and LP index policies when $N > 4000$ due to the difficulty of high-accuracy simulations.}
     \label{fig:uniform-1}
\end{figure}

\begin{figure}[t]
     \subcaptionbox{Second random instance.\label{fig:uniform-6-total}}{\includegraphics[height=5.5cm]{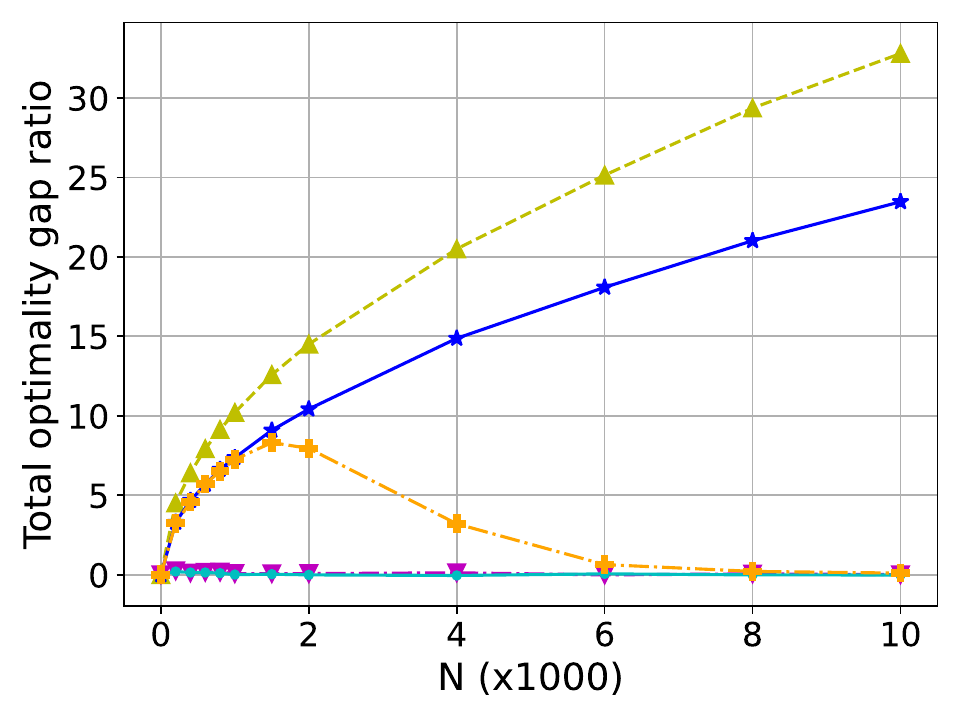}}
    \hfill
     \subcaptionbox{Third random instance.\label{fig:uniform-0-total}}{\includegraphics[height=5.5cm]{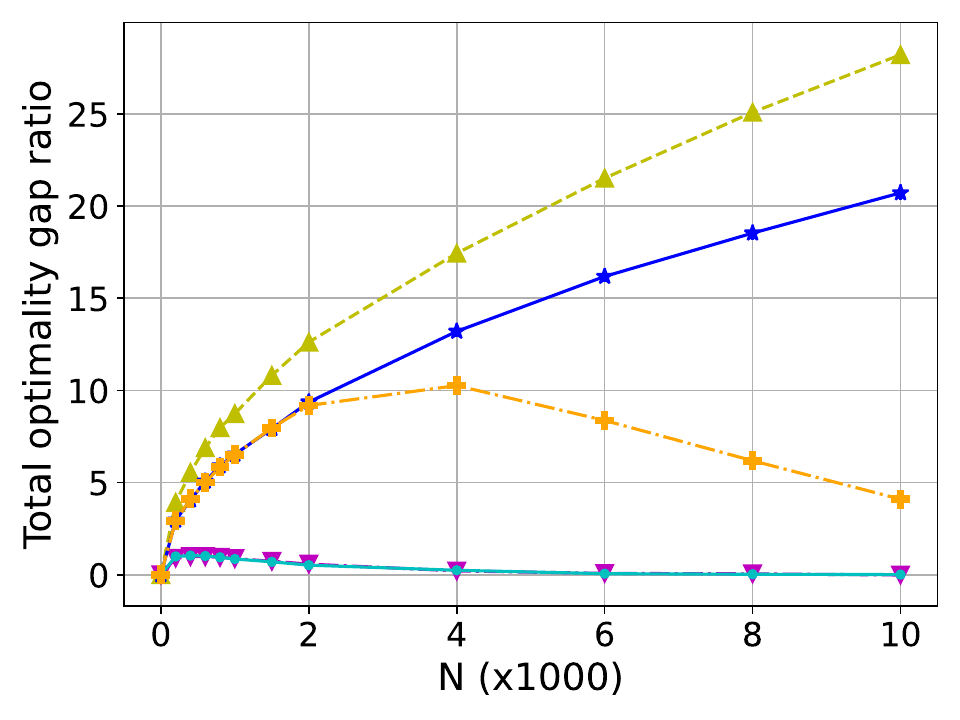}}
     \caption{Total optimality gap ratios for two additional uniformly generated random instances.}
     \label{fig:uniform-6-and-0}
\end{figure}

For the second type of instances, we consider a few uniformly randomly generated instances, each with an eight-state single-armed MDP.
The simulation result of one such instance is displayed in \Cref{fig:uniform-1}. In \Cref{fig:uniform-1-total}, we compare the total optimality gap ratios of different policies. As shown in the figure, the Whittle index policy and the LP index policy perform the best, whereas FTVA and the ID policy perform visibly worse.
The total optimality gap ratios of the two-set policy match those of the ID policy for small $N$, but quickly decrease after $N>2000$ and converge to zero.
To better distinguish their decay rates, in \Cref{fig:uniform-1-log}, we also plot the \emph{log optimality gap ratios}, defined as $\log_{10}(1-\rsysn/\rrel)$, for the same RB instance. The rapid decrease in the curves corresponding to the Whittle index policy, the LP index policy, and the two-set policy is consistent with exponential decay, as predicted by the theory.

We find that the trends in \Cref{fig:uniform-1} are representative and can be observed in many other random instances. For example, \Cref{fig:uniform-6-total} shows the total optimality gap ratios in another random instance. Although the Whittle index and LP index policies still perform better than the two-set policy, the two-set policy demonstrates better-than-$O(1/N)$ optimality gaps and consistently performs better than those policies with $O(1/\sqrt{N})$ optimality gaps.
Note that there is also a smaller fraction of random instances that shows different trends, where the two-set policy starts displaying better-than-$O(1/N)$ optimality gaps only for large $N$; the simulation result on one such instance is shown in \Cref{fig:uniform-0-total}.

\section{Conclusion}
\label{sec:conclusion}
In this paper, we have studied the conditions for achieving exponential asymptotic optimality in discrete-time RBs with the infinite-horizon average-reward criterion. We propose a novel two-set policy that achieves exponential asymptotic optimality under three assumptions: aperiodic unichain (\Cref{assump:aperiodic-unichain}), non-degeneracy (\Cref{assump:non-degeneracy}), and local stability (\Cref{assump:local-stability}). Our assumptions are strictly weaker than those in the prior work, where a global attractor assumption is also required for their policies to achieve exponential asymptotic optimality.

To complement our positive result, we also discuss the necessities of our three assumptions. In particular, we prove a lower bound on $\rrel - \ropt$ for a large class of RB instances that do not satisfy local stability, illustrating the fundamental role of local stability for exponential asymptotic optimality.

Finally, we numerically evaluate the performance of the two-set policy. Our numerical results show that the two-set policy achieves the best performance among the policies simulated here when our three assumptions are satisfied but UGAP is violated; the two-set policy also performs reasonably well on RB instances generated uniformly at random.

For future work, an interesting direction is to extend our results to RBs with \emph{heterogeneous} arms, where the arms could have distinct transition dynamics and reward functions. 
In this case, the best known optimality gap is $O(1/\sqrt{N})$ \citep{ZhaHonWan_25_het} for the long-run average-reward setting. In the finite-horizon setting, \citet{Zha_24_finite_het_nondegeneracy} has achieved better optimality gaps than $O(1/\sqrt{N})$, and it remains open if their machinery can be adapted to our setting.

\section*{Acknowledgments}
This work was conducted while Yige Hong was a PhD student in the Computer Science Department at Carnegie Mellon University.
The work of Yige Hong and Weina Wang was supported in part by National Science Foundation (NSF) grants ECCS-2145713, CCF-2403194, CCF-2428569, and ECCS-2432545.
The work of Yudong Chen was supported in part by NSF grant CCF-2233152 and a Vilas Associates Award.
The work of Qiaomin Xie was supported in part by NSF grants ECCS-2339794 and ECCS-2432546.

\bibliographystyle{plainnat-no-identifiers}
\bibliography{bib/refs}

\appendix

\crefalias{section}{appendix}
\crefalias{subsection}{subappendix}
\crefalias{subsubsection}{subsubappendix}
\crefname{appendix}{Appendix}{Appendices}
\Crefname{appendix}{Appendix}{Appendices}
\crefname{subappendix}{Appendix}{Appendices}
\Crefname{subappendix}{Appendix}{Appendices}
\crefname{subsubappendix}{Appendix}{Appendices}
\Crefname{subsubappendix}{Appendix}{Appendices}

\section{Local stability and the uniform global attractor property}
\label{app:ugap}

This appendix proves the implication from the Uniform Global Attractor Property (UGAP) to local stability discussed in \Cref{remark:UGAP}. This auxiliary result is not used in the proofs of our main results; it clarifies their relationship to prior work. We first recall the UGAP assumption from \citep{GasGauYan_23_exponential,GasGauYan_23_whittles}.

\begin{assumption}[Uniform Global Attractor Property]
    \label{assump:ugap}
    Define the mean transition map
    \[
        \phi(x) \triangleq  \E{X_{1}([N]) \givenplain X_0([N]) = x},
    \]
    where the actions are determined by the \lppriority policy.
    Let $x_{t+1} = \phi(x_t)$ for $t = 0, 1, \dots$ starting from any $x_0\in\simplex(\sspa)$.
    For each $\epsilon > 0$, there exists $T < \infty$ independent of both $x_0$ and $N$ such that $\norm{x_t - \statdist} \leq \epsilon$ for all $t \geq T$ and all $x_0\in\simplex(\sspa)$.
\end{assumption}

\begin{proposition}[UGAP implies local stability]\label{prop:ugap-implies-local-stability}
Assuming non-degeneracy (\Cref{assump:non-degeneracy}) and UGAP (\Cref{assump:ugap}), we have $\rho(\Phi) < 1$, so local stability (\Cref{assump:local-stability}) holds.
\end{proposition}

\begin{proof}
By non-degeneracy and \Cref{lem:lp-priority-and-phi}, there exists a radius $\eta_0 > 0$ such that whenever $\norm{x - \statdist}_1 \leq \eta_0$, the condition \eqref{eq:lp-priority-condition-1} holds and the mean transition map is linear, given by: $\phi(x) - \statdist = (x - \statdist)\,\Phi$.

Suppose for contradiction that $\rho(\Phi) \geq 1$, and let $\lambda\in\mathbb{C}$ be an eigenvalue of $\Phi$ with $|\lambda| \geq 1$. We start by showing the following claim: There exist a right eigenvector $y$ of $\Phi$ at $\lambda$
and constants $b_{\min}, b_{\max} > 0$ such that for all $t\geq 0$,
\begin{equation}\label{eq:ugap:two-sided-bound}
    b_{\min}\,|\lambda|^t \norm{\Rey}_1 \leq  \norm{(\Rey)\,\Phi^t}_1 \leq  b_{\max}\,|\lambda|^t \norm{\Rey}_1,
\end{equation}
where $\Rey$ denotes the entrywise real part of the vector $y$. We let $\Imy$ denote the entrywise imaginary part of $y$, and let $\bar{y}$ be the entrywise complex conjugate of $y$.
We discuss two cases based on whether $\Phi$ has a real eigenvector at $\lambda$:
\begin{itemize}
    \item \emph{Case A: There exists a real right eigenvector $y$ of $\Phi$ at $\lambda$.} Then $\Imy = 0$, $\Rey = y$, and $y\,\Phi^t = \lambda^t y$, so $\norm{(\Rey)\,\Phi^t}_1 = |\lambda|^t \norm{\Rey}_1$ and \eqref{eq:ugap:two-sided-bound} holds with $b_{\min} = b_{\max} = 1$.

    \item \emph{Case B: All right eigenvectors of $\Phi$ at $\lambda$ are complex.} Let $y$ be any right eigenvector of $\Phi$ at $\lambda$. Because $y = \Rey + i \; \Imy$,
    the two real vectors $\Rey$ and $\Imy$ must be linearly independent; otherwise, either $\Rey$ or $\Imy$ would be an eigenvector of $\Phi$ at $\lambda$, contradicting the hypothesis of Case B.
    Denoting $\theta_0$ as the argument of the complex scalar $\lambda$, we have
    \[
        (\Rey) \, \Phi^t = \frac{y + \bar{y}}{2} \Phi^t = \frac{1}{2}\left( \lambda^t y + \bar{\lambda}^t \bar{y} \right) = \mathrm{Re}(\lambda^t y).
    \]
    Therefore, the $L_1$ norm of $(\Rey) \, \Phi^t$ can be written as
    \begin{equation}
        \label{eq:rey-phit-l1}
        \norm{(\Rey) \, \Phi^t}_1 = \norm{\mathrm{Re}(\lambda^t y)}_1 = \abs{\lambda}^t \norm{\mathrm{Re}(e^{i\theta_0t}y)}_1. %
    \end{equation}
    To prove \eqref{eq:ugap:two-sided-bound}, it remains to show that the function $g(\theta) \triangleq \norm{\mathrm{Re}(e^{i\theta}y)}_1 / \norm{\Rey}_1$ is upper and lower bounded by some constants $b_{\max}, b_{\min}>0$, uniformly for all $\theta \in [0,2\pi]$. This follows from the following three facts:
    \begin{enumerate}
        \item Because $\Rey$ and $\Imy$ are linearly independent, the denominator $\norm{\Rey}_1$ is non-zero, so $g(\theta)$ is finite for all $\theta$.
        \item Because $\Rey$ and $\Imy$ are linearly independent, the numerator $\norm{\mathrm{Re}(e^{i\theta}y)}_1 > 0$ for all $\theta$, so $g(\theta) > 0$ for all $\theta$.
        \item The function $g(\theta)$ is continuous over the compact set $[0, 2\pi]$, so it can attain its maximum and minimum.
    \end{enumerate}
    These three facts imply the existence of $b_{\max}, b_{\min}>0$ such that for all $\theta\in[0,2\pi]$,
    \[
        b_{\min}\leq \norm{\mathrm{Re}(e^{i\theta}y)}_1 / \norm{\Rey}_1 \leq b_{\max},
    \]
    which, when combined with \eqref{eq:rey-phit-l1}, implies \eqref{eq:ugap:two-sided-bound}.
\end{itemize}

Now we conclude the proof of the proposition using \eqref{eq:ugap:two-sided-bound}. We normalize $y$ so that $\norm{\Rey}_1 = 1$. By UGAP with $\epsilon \triangleq \eta_0\,b_{\min}/(2\,b_{\max}+2) < \eta_0$, there exists a constant $T < \infty$ such that $\norm{x_t - \statdist}_1 \leq \epsilon$ for all $t \geq T$ and all $x_0\in\simplex(\sspa)$.
We fix $x_0 = \statdist + \delta\,\Rey$, where $\delta \triangleq \eta_0 / \big((b_{\max}+1)\,|\lambda|^{T+1}\big)$.

We now establish by induction that for all $t=0,1,\dots, T+1$,
\begin{align}
    \label{eq:ugap:induction-claim-recursion}
    x_t - \statdist &= \delta\,(\Rey)\,\Phi^t \\
    \label{eq:ugap:induction-claim-range}
    \norm{x_t - \statdist}_1 &\leq  \eta_0.
\end{align}

\begin{itemize}
\item \emph{Base case ($t=0$).} By construction, $x_0 - \statdist = \delta\,\Rey = \delta\,(\Rey)\,\Phi^0$, and $\norm{x_0 - \statdist}_1 = \delta \norm{\Rey}_1 = \delta < \eta_0$.
\item \emph{Induction step.} Assume that we have proved the induction hypothesis for some $t\leq T$. We first prove \eqref{eq:ugap:induction-claim-recursion} for $t+1$: because \eqref{eq:ugap:induction-claim-range} holds for $t$, by the definition of $x_{t+1}$, we have
\[
    x_{t+1} = \statdist + (x_t - \statdist)\Phi.
\]
Substituting the induction hypothesis that $x_{t} - \statdist = \delta (\Rey)\Phi^t$, we get  $x_{t+1} - \statdist = \delta (\Rey)\Phi^{t+1}$, which proves \eqref{eq:ugap:induction-claim-recursion} for $t+1$. Then, by the upper bound in \eqref{eq:ugap:two-sided-bound} and the definition of $\delta$, we have%
\[
    \norm{x_{t+1} - \statdist}_1 = \delta\,\norm{(\Rey)\,\Phi^{t+1}}_1 \leq  \delta\,b_{\max}\,|\lambda|^{t+1} \leq  \delta\,b_{\max}\,|\lambda|^{T+1} < \eta_0,
\]
which proves \eqref{eq:ugap:induction-claim-range} for $t+1$.
\end{itemize}

By induction, \eqref{eq:ugap:induction-claim-recursion} and \eqref{eq:ugap:induction-claim-range} hold for all $t=0,1,\dots, T+1$.
Applying \eqref{eq:ugap:induction-claim-recursion} at $T+1$, we get
\[
    \norm{x_{T+1} - \statdist}_1 = \delta\,\norm{(\Rey)\,\Phi^{T+1}}_1 \geq  \delta \, b_{\min}\,|\lambda|^{T+1} = 2\epsilon > \epsilon.
\]
This contradicts UGAP, which requires $\norm{x_{T+1} - \statdist}_1 \leq \epsilon$. Therefore, we must have $\rho(\Phi) < 1$.
\end{proof}

\label{app:ugap:lastpage}

\clearpage
\setcounter{page}{1}
\renewcommand{\thepage}{ec\arabic{page}}
\section*{Electronic Companion}
\addcontentsline{toc}{section}{Electronic Companion}

\setcounter{section}{0}
\renewcommand{\thesection}{EC.\arabic{section}}
\renewcommand{\thesubsection}{\thesection.\arabic{subsection}}
\renewcommand{\thesubsubsection}{\thesubsection.\arabic{subsubsection}}
\crefalias{section}{section}
\crefalias{subsection}{subsection}
\crefalias{subsubsection}{subsubsection}

\setcounter{equation}{0}
\renewcommand{\theequation}{EC.\arabic{equation}}
\setcounter{figure}{0}
\renewcommand{\thefigure}{EC.\arabic{figure}}
\setcounter{table}{0}
\renewcommand{\thetable}{EC.\arabic{table}}

\setcounter{proposition}{0}
\renewcommand{\theproposition}{EC.\arabic{proposition}}
\setcounter{corollary}{0}
\renewcommand{\thecorollary}{EC.\arabic{corollary}}
\setcounter{claim}{0}
\renewcommand{\theclaim}{EC.\arabic{claim}}
\setcounter{conjecture}{0}
\renewcommand{\theconjecture}{EC.\arabic{conjecture}}
\setcounter{assumption}{0}
\renewcommand{\theassumption}{EC.\arabic{assumption}}
\setcounter{remark}{0}
\renewcommand{\theremark}{EC.\arabic{remark}}
\setcounter{example}{0}
\renewcommand{\theexample}{EC.\arabic{example}}
\setcounter{definition}{0}
\renewcommand{\thedefinition}{EC.\arabic{definition}}

\section{Auxiliary facts about the assumptions for restless bandits}
\label{app:rb:assumptions}

This section provides context for \Cref{assump:non-degeneracy,assump:local-stability}. \Cref{subsec:rb:non-degeneracy-validity} compares the unique-neutral-state assumption with LP non-degeneracy, while \Cref{subsec:rb:local-stability-validity} gives an equivalent characterization of local stability in terms of $\Phi_0$. Neither result is used in the proofs of our main results. 

Below, $\rho(M)$ denotes the spectral radius of a matrix $M$, the maximum modulus of its eigenvalues.

\subsection{Comparison of \Cref{assump:non-degeneracy} with LP non-degeneracy}\label{subsec:rb:non-degeneracy-validity}

We note that the non-degeneracy in \Cref{assump:non-degeneracy} (having a unique neutral state at $y^*$) is distinct from $y^*$ being a non-degenerate basic feasible solution (BFS) of the LP relaxation \eqref{eq:lp-single}. To compare the two, we assume that $y^*$ is a BFS. We partition $\sspa = S^0 \cup S^+ \cup S^- \cup \sempty$, where $S^+ = \{s\colon y^*(s,1)>0,\,y^*(s,0)=0\}$, $S^- = \{s\colon y^*(s,1)=0,\,y^*(s,0)>0\}$, and $S^0 = \{s\colon y^*(s,1)>0,\,y^*(s,0)>0\}$ is the set of neutral states. After dropping one redundant flow equation from \eqref{eq:flow-balance}, the LP relaxation has $|\sspa|+1$ effective equality constraints. The relationship between the two notions of non-degeneracy splits into two cases:
\begin{itemize}
    \item When $\absplain{\sempty}=0$, the two notions coincide: $y^*$ satisfies \Cref{assump:non-degeneracy} if and only if it is a non-degenerate BFS of the LP relaxation.
    \item When $\absplain{\sempty}>0$, the two notions are mutually exclusive.
\end{itemize}
These claims follow from a counting argument:
When $y^*$ is a non-degenerate BFS of the LP relaxation, it has exactly $\absplain{\sspa}+1$ non-zero entries, which implies that $2\absplain{S^0}+\absplain{S^+}+\absplain{S^-}=\absplain{\sspa}+1$. Because $\absplain{S^0}+\absplain{S^+}+\absplain{S^-}+\absplain{\sempty}=\absplain{\sspa}$, we get $\absplain{S^0}=1+\absplain{\sempty}$. Therefore, when $\absplain{\sempty}=0$, we have $\absplain{S^0}=1$, so \Cref{assump:non-degeneracy} holds. When $\absplain{\sempty}>0$, we have $\absplain{S^0}\geq 2$, contradicting \Cref{assump:non-degeneracy}.

It remains to verify the converse direction of the first bullet, i.e., that when $\sempty=\emptyset$, $\absplain{S^0}=1$ forces $y^*$ to be a non-degenerate BFS. Suppose for contradiction $y^*$ is a degenerate BFS; then it has strictly less than $\absplain{\sspa}+1$ non-zero entries, giving $\absplain{S^0}<1+\absplain{\sempty}=1$. Hence $\absplain{S^0}=0$, contradicting \Cref{assump:non-degeneracy}.

\subsection{Validity of \Cref{assump:local-stability}}\label{subsec:rb:local-stability-validity}

This subsection establishes an equivalent characterization of \Cref{assump:local-stability} in terms of $\Phi_0$. Recall that the matrices $\Phi$ and $\Phi_0$ are defined as
\begin{align}
    \label{eq:Phi-def-restate}
    \Phi &= P_{\pibs} - (\costvec-\alpha\vone)^\top(P_1(\sneu) - P_0(\sneu)) - \vone^\top \statdist \\
    \label{eq:Phi-0-def-restate}
    \Phi_0 &= P_{\pibs} - (\costvec-\alpha\vone)^\top(P_1(\sneu) - P_0(\sneu)).
\end{align}
\Cref{prop:local-stability-equivalence} below implies that $\rho(\Phi) < 1$ (\Cref{assump:local-stability}) is equivalent to the assumption that $\Phi_0$ is contractive when right-multiplying the vectors in $\simplex(\sspa)$.

\begin{proposition}[Equivalent characterization of local stability]\label{prop:local-stability-equivalence}
Let $\Phi$ and $\Phi_0$ be as in \eqref{eq:Phi-def-restate} and \eqref{eq:Phi-0-def-restate}. We have the following:
\begin{enumerate}[leftmargin=2em, label=(\roman*)]
    \item $\Phi$ and $\Phi_0$ coincide on the hyperplane $H \triangleq \{v\in\R^{|\sspa|} : v\vone^\top = 0\}$, i.e., $v\Phi = v\Phi_0$ for all $v\in H$.
    \item The eigenvalues of $\Phi$ are exactly the eigenvalues of $\Phi_0$ with one occurrence of the eigenvalue $1$ (associated with the right eigenvector $\vone^\top$) replaced by $0$, leaving all other eigenvalues unchanged.
\end{enumerate}
\end{proposition}

\begin{proof}
For (i): for any $v\in H$, $v(\Phi - \Phi_0) = -v\,\vone^\top\statdist = -(v\vone^\top)\statdist = 0$, so $v\Phi = v\Phi_0$.

For (ii): Brauer's theorem on rank-one spectral perturbations \citep{Bra_52} states that if $Ax = \lambda x$ for some right eigenvector $x$ of a square matrix $A$ at an eigenvalue $\lambda$, then for any column vector $y$, the perturbed matrix $A + xy^\top$ has the same eigenvalues as $A$ except that one occurrence of $\lambda$ is replaced by $\lambda + y^\top x$ (see \citealp[Example~1.2.8]{HorJoh_13} for a textbook treatment, and \citealp{McDPap_21} for a short and elementary proof).

We will apply Brauer's theorem with $A = \Phi_0$, $\lambda = 1$, $x = \vone^\top$, and $y^\top = -\statdist$.
To this end, we first notice that $\vone^\top$ is a right eigenvector of $\Phi_0$ at eigenvalue $1$, which follows from the facts that $P_\pibs\vone^\top = \vone^\top$, and $(P_1(\sneu) - P_0(\sneu))\vone^\top = 0$.
Brauer's theorem then implies that $\Phi = \Phi_0 - \vone^\top\statdist = \Phi_0 + xy^\top$ has the same spectrum as $\Phi_0$ except that one occurrence of $1$ is replaced by $1 + y^\top x = 1 - \statdist\vone^\top = 0$.
\end{proof}

\section{General lemmas on subroutine dynamics}\label{app:preliminary-proofs}
This section collects general-purpose facts about the two subroutines. \Cref{app:subroutine-transition-lemmas} proves the transition identities for \frcontrol and \lppriority; \Cref{app:proof-of-exist-fe} proves the feasibility-ensuring property used in the policy construction. The weighted-norm definitions and contraction properties used here are established in \Cref{sec:weighted-l2-norms}.

\subsection{Transition dynamics under subroutines}\label{app:subroutine-transition-lemmas}

The lemma below characterizes the dynamics of a $N$-armed RB system if all arms follow \frcontrol (\Cref{alg:follow-pibs}). 

\begin{restatable}{lemma}{followpibs}\label{lem:pibs-transition}
    For any time step $t$, 
    \begin{equation}
        \E{X_{t+1}([N]) \givenplain X_t, \textup{all arms follow \Cref{alg:follow-pibs}}} = X_t([N]) P_{\pibs} \quad a.s., 
    \end{equation}
    where expectation is taken entry-wise for the random vector $X_{t+1}([N])$; $P_{\pibs}$ is the transition matrix induced by $\pibs$ on the single-armed MDP, i.e., $P_{\pibs}(s,s') = \sum_{a\in\aspa} \pibs(a|s) P(s,a,s')$ for $s,s'\in\sspa$.  
\end{restatable}

\begin{proof}%
    By definition, if all $N$ arms follow \frcontrol at time $t$, the expected number of arms in state $s$ taking action $a$ is equal to $X_t([N], s)\pibs(a|s)$. Among these arms, the fraction of arms transitioning to state $s'$ in the next time step equals $P(s,a,s')$. Since the state transition of each arm is Markovian given the state and action, the expected number of arms going from state $s$ to state $s'$ and taking action $a$ equals $X_t([N], s)\pibs(a|s)P(s,a,s')$. Summing over $s$ and $a$, we get 
    \begin{align*}
        & \E{X_{t+1}([N], s') \givenplain X_t, \textup{all arms follow \Cref{alg:follow-pibs}}} \\
        &\qquad = \sumsa X_t([N], s)\pibs(a|s)P(s,a,s') \quad \forall s'\in\sspa,
    \end{align*}
    which completes the proof of the lemma.  
\end{proof}

If all arms in the RB system follow the \lppriority (\Cref{alg:follow-lp-priority}), the expectation of the scaled state-count vector $X_t([N])$ has a linear dynamics, as stated in the lemma below. The lemma is basically a restatement of Lemma 2 in \citep{GasGauYan_23_whittles}, although in \citep{GasGauYan_23_whittles}, the LP-Priority policy was considered, and it was assumed that $y^*(s,1) + y^*(s,0) > 0$ for all $s\in\sspa$.

\followlppriority*

\begin{proof}%
Throughout this proof, we use the shorthand $x$ for $x([N])$, and omit ``$\veA_t\sim \lppriority$'' in the conditioning. 

We fixed an arbitrary $x$ and let $X_t = x$.  
For any $s\neq \sneu$ and $a\in\aspa$, the expected number of arms taking action $a$ equals $x(s) \pibs(a|s)$; for $s=\sneu$, the expected number of arms taking action $1$ equals $\alpha -\sum_{s\neq\sneu }x(s)\pibs(1|s)$. and the expected number of arms taking action $0$ equals $1-\alpha- \sum_{s\neq\sneu }x(s)\pibs(0|s)$. 
Therefore, 
for any $s'\in\sspa$ such that $s'\neq \sneu$, 
\begin{align}
    &\E{X_{t+1}(s') \givenplain X_t=x} \nonumber \\
    \nonumber
    &= \sum_{s\neq \sneu, a\in\aspa} x(s) \pibs(a|s) P(s,a,s') \\ 
    \nonumber
    &\mspace{20mu} + \Big(\alpha - \sum_{s\neq \sneu} x(s)\pibs(1|s)\Big) P(\sneu, 1, s') 
    \nonumber
    + \Big(1 - \alpha - \sum_{s\neq \sneu} x(s) \pibs(0|s) \Big) P(\sneu, 0, s')  \\
    \nonumber
    &= \sum_{s\neq \sneu} x(s)P_\pibs(s, s')
    + \Big(\alpha - \sum_{s\neq \sneu} x(s) \pibs(1|s)\Big)P_1(\sneu, s') 
    + \Big(1 - \alpha - \sum_{s\neq \sneu} x(s)\pibs(0|s) \Big)P_0(\sneu, s') \\
    \label{eq:derive-phi-1}
    &= \sums x(s)P_\pibs(s,s')  - x(\sneu)P_\pibs(\sneu, s')  \\
    \nonumber
    &\mspace{20mu} + \Big(\alpha - \sum_{s\neq \sneu} x(s) \pibs(1|s)\Big)P_1(\sneu, s') 
    + \Big(1 - \alpha - \sum_{s\neq \sneu} x(s)\pibs(0|s) \Big)P_0(\sneu, s'). 
\end{align}
Because $P_\pibs(\sneu,s') =  \pibs(1|\sneu) P_1(\sneu,s')+ \pibs(0|\sneu)P_0(\sneu, s')$, we expand the second term in \eqref{eq:derive-phi-1} and rearrange the terms to get
\begin{align}
    & \E{X_{t+1}(s') \givenplain X_t=x} \nonumber\\
    &=  \sums x(s)P_\pibs(s,s')
    +  \Big(\alpha - \sums x(s) \pibs(1|s) \Big) P_1(\sneu, s')
    + \Big(1 - \alpha - \sums x(s) \pibs(0|s) \Big) P_0(\sneu, s')\nonumber \\
    &= \sums x(s) P_\pibs(s,s') 
    + \Big(\alpha - \sums x(s)\pibs(1|s) \Big) P_1(\sneu, s')
    -\Big(  \alpha - \sums x(s)\pibs(1|s) \Big) P_0(\sneu, s') \label{eq:derive-phi-2}\\
    &= \sums x(s)\Big(P_\pibs(s,s') -  \pibs(1|s)\big(P_1(\sneu, s') - P_0(\sneu, s')\big) \Big) + \alpha (P_1(\sneu, s') - P_0(\sneu, s')),  \label{eq:derive-phi-3}
\end{align}
where \eqref{eq:derive-phi-2} uses the facts that $\pibs(0|s) = 1-\pibs(1|s)$ and $\sums x(s) = 1$. 

Next, we calculate $\E{X_{t+1}(\sneu) \givenplain X_t=x}$. 
Observe that $X_{t+1}(\sneu) = 1 - \sum_{s'\neq \sneu} X_{t+1}(s')$, so
\begin{align}
    \nonumber
    & \E{X_{t+1}(\sneu) \givenplain X_t=x}\\
    \nonumber
    &= 1 - \sum_{s'\neq \sneu} \E{X_{t+1}(s') \givenplain X_t=x} \nonumber \\
    \nonumber
    &= \sums  x(s)\Big(1 - \sum_{s'\neq \sneu} P_\pibs(s,s') + \sum_{s'\neq\sneu}  \pibs(1|s) \big(P_1(\sneu, s') - P_0(\sneu, s')\big) \Big) - \alpha \sum_{s'\neq\sneu}(P_1(\sneu, s') - P_0(\sneu, s')) \\
    \label{eq:derive-phi-3-sneu}
    &= \sums x(s) \Big(P_\pibs(s,\sneu) -  \pibs(1|s) \big(P_1(\sneu, \sneu) - P_0(\sneu, \sneu)\big)\Big)  + \alpha (P_1(\sneu, \sneu) - P_0(\sneu, \sneu)).
\end{align}
Note that \eqref{eq:derive-phi-3-sneu} has the same form as \eqref{eq:derive-phi-3}, with $s'$ replaced by $\sneu$. 

We rewrite \eqref{eq:derive-phi-3} and \eqref{eq:derive-phi-3-sneu} in matrix form to get 
\begin{align}\label{eq:derive-phi-4}
    \E{X_{t+1} \givenplain X_t=x} 
    &= x\Big(P_\pibs - \costvec^\top \big(P_1(\sneu) - P_0(\sneu)\big) \Big) + \alpha (P_1(\sneu) - P_0(\sneu)).
\end{align}
Recall that we have let $\Phi = P_\pibs - \vone^\top \statdist - (\costvec - \alpha \vone)^\top \big(P_1(\sneu) - P_0(\sneu)\big)$, we can verify that the RHS of \eqref{eq:derive-phi-4} is the same as $(x-\statdist)\Phi + \statdist$. Specifically, 
\begin{align}
    (x-\statdist)\Phi + \statdist 
    \nonumber
    &= (x-\statdist)\Big(P_\pibs - \vone^\top \statdist - (\costvec - \alpha \vone)^\top \big(P_1(\sneu) - P_0(\sneu)\big)\Big) + \statdist \\
    \nonumber
    &= x \Big(P_\pibs - \costvec^\top \big(P_1(\sneu) - P_0(\sneu)\big) \Big) + x\Big(- \vone^\top \statdist -  \alpha \vone^\top \big(P_1(\sneu) - P_0(\sneu)\big)\Big) \\
    \nonumber
    &\mspace{20mu} - \statdist \Big(P_\pibs - \costvec^\top \big(P_1(\sneu) - P_0(\sneu)\big) \Big) - \statdist \Big(- \vone^\top \statdist -  \alpha \vone^\top \big(P_1(\sneu) - P_0(\sneu)\big)\Big) + \statdist\\
    \nonumber
    &= x \Big(P_\pibs - \costvec^\top \big(P_1(\sneu) - P_0(\sneu)\big) \Big) - \statdist - \alpha \big(P_1(\sneu) - P_0(\sneu)\big) \\
    \label{eq:derive-phi-5}
    &\mspace{20mu} - \statdist  + \alpha \big(P_1(\sneu) - P_0(\sneu)\big) +  \statdist +  \alpha \big(P_1(\sneu) - P_0(\sneu)\big) + \statdist \\
    \nonumber
    &= x\Big(P_\pibs - \costvec^\top \big(P_1(\sneu) - P_0(\sneu)\big) \Big) + \alpha (P_1(\sneu) - P_0(\sneu)). 
\end{align}
where to get \eqref{eq:derive-phi-5} we use the facts that $x\vone^\top = \statdist \vone^\top = 1$ and $\statdist \costvec^\top = \alpha$. 
Substituting the RHS of \eqref{eq:derive-phi-4}  with $(x-\statdist)\Phi + \statdist$, we have proved that under the \lppriority, $\E{X_{t+1} \givenplain X_t=x} = (x-\statdist)\Phi + \statdist$. 
This completes the proof of \Cref{lem:lp-priority-and-phi}. 
\end{proof}

\subsection{Proof of \hcref{lem:feasibility-ensuring}}
\label{app:proof-of-exist-fe}

In this subsection, we prove \Cref{lem:feasibility-ensuring}, which shows that there exists $\eta = \Theta(1)$ and $\errfe = O(1/N)$ such that $(\eta, \errfe)$ is feasibility ensuring, i.e., for any system state $x$ and $D\subseteq[N]$ such that
\begin{equation}\tag{\ref{eq:sufficient-cond-for-OL-assertion}}
  \norm{x(D) - m(D) \statdist}_{\umat} \leq \eta m(D) - \errfe,
\end{equation}
we have
\begin{align}
    \label{eq:lp-priority-condition-XD-restate-1}
    \sum\nolimits_{s\neq \sneu} \pibs(1|s) x(D, s) &\leq  \alpha m(D) - \frac{|\sempty|+1}{N} \\ 
    \label{eq:lp-priority-condition-XD-restate-2}
     \sum\nolimits_{s\neq \sneu} \pibs(0|s) x(D, s) &\leq (1-\alpha)m(D) - \frac{|\sempty|+1}{N}. 
\end{align}

\feasibilityensuring*

\begin{proof}    
    We first show \eqref{eq:lp-priority-condition-XD-restate-1}. Observe that $\sums \pibs(1|s) \statdist(s) = \sumsa y^*(s,a) = \alpha$, so 
    \begin{align}
        \nonumber
         \alpha m(D) - \sum_{s\neq \sneu} \pibs(1|s) x(D, s) 
        &= y^*(\sneu, 1)m(D) + \sum_{s\neq\sneu} \pibs(1|s) \big(m(D)\statdist(s) - x(D, s)\big) \\
        \nonumber
        &\geq y^*(\sneu, 1)m(D) - \sum_{s\neq\sneu} \abs{m(D)\statdist(s) - x(D, s)} \\
        \nonumber
        &\geq y^*(\sneu, 1)m(D) - \norm{m(D)\statdist - x(D)}_1 \\
        \nonumber
        &\geq y^*(\sneu, 1)m(D) - |\sspa|^{1/2}\norm{m(D)\statdist - x(D)}_\umat,
    \end{align}
    where to get the last expression, we use the argument that $U$'s eigenvalues are at least $1$, so $\norm{v}_1\leq |\sspa|^{1/2} \norm{v}_2 \leq  |\sspa|^{1/2} \norm{v}_U$ for any $|\sspa|$-dim vector $v$. 
    To bound $\norm{m(D)\statdist - x(D, s)}_\umat$, recall that $D$ is chosen such that $\slkb(x, D) \geq 0$, so by the definition of $\slk(x, D)$, we have
    \[
        \norm{x(D) - m(D) \statdist}_U \leq \eta m(D) - \errfe \leq |\sspa|^{-1/2} y^*(\sneu, 1) m(D) - \frac{|\sspa|^{-1/2}(|\sempty|+1)}{N}.  
    \]
    Therefore, combining the above calculations, we have
    \[
         \alpha m(D) - \sum_{s\neq \sneu} \pibs(1|s) x(D, s) \geq \frac{|\sempty|+1}{N}, 
    \]
    which implies \eqref{eq:lp-priority-condition-XD-restate-1}. We can prove \eqref{eq:lp-priority-condition-XD-restate-2} similarly.
\end{proof}

\section{Proof of \hcref{lem:hw-hu-properties}}
\label{app:proof-hw-hu}

\hwhuproperties*

\begin{proof}[Proof of \Cref{lem:hw-hu-properties}]
    \textbf{Proving \eqref{eq:hw:drift} and \eqref{eq:hw:high-prob}}
    By the pseudo-contraction properties of $\wmat$-weighted norms given in \Cref{lem:one-step-contraction-W-U},
    \begin{align*}
        \rhow \hw(X_t, D)
        &=  \rhow \norm{X_{t}(D) - m(D)\statdist}_\wmat \\
        &\geq \norm{(X_{t}(D) - m(D)\statdist)P_\pibs }_\wmat \\
        & = \norm{X_{t}(D)P_\pibs - m(D)\statdist}_\wmat.
    \end{align*}
    Consequently,
    \begin{align*}
        \hw(X_{t+1}', D) - \rhow \hw(X_t, D)
        &\leq  \norm{X_{t+1}'(D)- m(D)\statdist}_\wmat -  \norm{X_{t}(D)P_\pibs - m(D)\statdist}_\wmat \\
        &\leq  \norm{X_{t+1}'(D)- X_t(D) P_{\pibs}}_\wmat \\
        &\leq \lamw^{1/2} \norm{X_{t+1}'(D)- X_t(D) P_{\pibs}}_2 \\
        &\leq \lamw^{1/2} \norm{X_{t+1}'(D)- X_t(D) P_{\pibs}}_1.
    \end{align*}
    Therefore, it suffices to show that
    \begin{align}
        \label{eq:hw:interm-goal-1}
        \Ebig{ \norm{X_{t+1}'(D)- X_t(D) P_{\pibs}}_1 \givenbig X_t}
        &= O(1/\sqrt{N}) \quad a.s.  \\
        \label{eq:hw:interm-goal-2}
        \Prob{\norm{X_{t+1}'(D)- X_t(D) P_{\pibs}}_1 > r \givenbig X_t}
        &= O(\exp(-C\lamw N r^2))   \quad a.s. \quad \forall r\geq 0.
    \end{align}

    To show \eqref{eq:hw:interm-goal-1} and \eqref{eq:hw:interm-goal-2}, we examine the distribution of $\normbig{X_{t+1}'(D)- X_t(D) P_{\pibs}}_1$ given $X_t$.
    Note that by \Cref{lem:pibs-transition}, $\Eplain{X_{t+1}'(D) \givenplain X_t} = X_t(D) P_{\pibs}$, so
    \begin{align}
        \nonumber
         \normbig{X_{t+1}'(D)- X_t(D) P_{\pibs}}_1
         &= \normbig{X_{t+1}'(D) - \E{X_{t+1}'(D) \givenplain X_t} }_1   \\
         \label{eq:hw:interm-3}
         &\leq \normbig{X_{t+1}'(D) - \E{X_{t+1}'(D) \givenplain X_t, (A_t(i))_{i\in D}} }_1 \\
         \label{eq:hw:interm-4}
         &\mspace{20mu}  +  \normbig{ \E{X_{t+1}'(D) \givenplain X_t, (A_t(i))_{i\in D}} - \E{X_{t+1}'(D) \givenplain X_t} }_1.
    \end{align}

    We first derive an expectation bound and a high-probability bound for \eqref{eq:hw:interm-3}.
    Note that for each $s\in\sspa$,
    \begin{align*}
        X_{t+1}'(D, s) = \frac{1}{N} \sum_{i\in D} \indibrac{S_{t+1}'(i) = s},
    \end{align*}
    where $S_{t+1}'(i)$ denote the state of arm $i$ at time $t+1$ if the arms in $D$ follow the \frcontrol at time $t$.
    Given $X_t$ and $(A_t(i))_{i\in D}$, $\indibrac{S_{t+1}'(i) = s}$ for $i\in D$ are independent Bernoulli random variables, each with mean $P(S_t(i), A_t(i), s)$.
    By Cauchy-Schwarz,
    \begin{align}
        \nonumber
        & \EBig{\absBig{ X_{t+1}'(D, s) - \Eplain{X_{t+1}'(D, s) \givenplain X_t, (A_t(i))_{i\in D}}} \givenBig X_t, (A_t(i))_{i\in D}}  \\
        \nonumber
        &\qquad \leq  \Var{X_{t+1}'(D, s) \givenBig X_t, (A_t(i))_{i\in D}}^{1/2} \\
        \nonumber
        &\qquad = \frac{1}{N} \left(\sum_{i\in D} \Var{\indibrac{S_{t+1}'(i) = s}, \givenBig X_t, (A_t(i))_{i\in D}} \right)^2 \\
        \nonumber
        &\qquad \leq \frac{1}{\sqrt{N}}.
    \end{align}
    By Hoeffding inequality, for all $r\geq 0$,
    \begin{align}
        \nonumber
        \ProbBig{\absBig{ X_{t+1}'(D, s) - \Eplain{X_{t+1}'(D, s) \givenplain X_t, (A_t(i))_{i\in D}}} > r \givenBig X_t, (A_t(i))_{i\in D}} \leq 2\exp\big(-2N r^2\big).
    \end{align}
    Therefore, summing over all $s\in\sspa$ and taking expectations over $(A_t(i))_{i\in D}$, we get
    \begin{align}
        \label{eq:hw:interm-5}
         \EBig{\normBig{ X_{t+1}'(D, s) - \Eplain{X_{t+1}'(D, s) \givenplain X_t, (A_t(i))_{i\in D}}}_1 \givenBig X_t} &\leq \frac{|\sspa|}{\sqrt{N}} \\
         \label{eq:hw:interm-6}
         \ProbBig{\normBig{ X_{t+1}'(D, s) - \Eplain{X_{t+1}'(D, s) \givenplain X_t, (A_t(i))_{i\in D}}}_1 > r \givenBig X_t} &\leq 2|\sspa|\exp\Big(-\frac{2N r^2}{|\sspa|^2}\Big),
    \end{align}
    where \eqref{eq:hw:interm-6} comes from applying a union bound.

    Next, we bound the term in \eqref{eq:hw:interm-4}.
    For any $s\in\sspa$, we have
    \begin{align*}
        \E{X_{t+1}'(D, s) \givenplain X_t, (A_t(i))_{i\in D}} &= \sum_{s'\in\sspa, a\in\aspa} \Gamma_t(s',a) P(s', a, s)  \\
        \E{X_{t+1}'(D, s) \givenplain X_t}
        &= \sum_{s'\in\sspa, a\in\aspa} \pibs(a|s) X_t(D, s') P(s', a, s),
    \end{align*}
    where $N\cdot\Gamma_t(s,a)$ is the number of arms in $D$ with state $s$ taking action $a$ at time $t$. By the definition of the \frcontrol (\Cref{alg:follow-pibs}),
    \[
      \Gamma_t(s',1) \in \{\ceil{N\pibs(1|s') X_t(D, s')}/N, \floor{N\pibs(1|s') X_t(D, s')}/N \},
    \]
    and $\Gamma_t(s,0) = X_t(D, s) - \Gamma_t(s,1)$.
    Consequently,
    \begin{align}
        \nonumber
        & \absbig{\E{X_{t+1}'(D, s) \givenplain X_t, (A_t(i))_{i\in D}} - \E{X_{t+1}'(D, s) \givenplain X_t} }  \\
        \nonumber
        &\qquad \leq \sum_{s'\in\sspa, a\in\aspa} \abs{\pibs(a|s') X_t(D, s')  -  \Gamma_t(s',a)} P(s', a, s) \\
        \label{eq:hw:interm-7}
        &\qquad \leq \frac{1}{N} \sum_{s'\in\sspa, a\in\aspa} P(s', a, s) \quad a.s.
    \end{align}
    Summing over all $s\in\sspa$ in \eqref{eq:hw:interm-7}, we get
    \begin{align}
        \label{eq:hw:interm-8}
        \normbig{\E{X_{t+1}'(D, s) \givenplain X_t, (A_t(i))_{i\in D}} - \E{X_{t+1}'(D, s) \givenplain X_t}}_1
        &\leq \frac{2|\sspa|}{N} \quad a.s.
    \end{align}

    Combining the above calculations, we get
    \begin{align}
        \nonumber
        & \Ebig{ \norm{X_{t+1}'(D)- X_t(D) P_{\pibs}}_1 \givenbig X_t} \\
        \nonumber
        &\qquad \leq   \Ebig{ \norm{X_{t+1}'(D) - \E{X_{t+1}'(D, s) \givenplain X_t, (A_t(i))_{i\in D}}}_1 \givenbig X_t}  \\
        \nonumber
        &\qquad \mspace{20mu} +  \Ebig{\normbig{\E{X_{t+1}'(D, s) \givenplain X_t, (A_t(i))_{i\in D}} - \E{X_{t+1}'(D, s) \givenplain X_t}}_1 \givenbig X_t}  \\
        \nonumber
        &\qquad \leq \frac{1}{\sqrt{N}} + \frac{2|\sspa|}{N} \\
        &\qquad \leq \frac{3|\sspa|}{\sqrt{N}},
    \end{align}
    which implies \eqref{eq:hw:interm-goal-1}.
    We also have that for any $r\geq 2|\sspa| / N$,
    \begin{align}
        \nonumber
        & \Prob{\norm{X_{t+1}'(D)- X_t(D) P_{\pibs}}_1 > r \givenbig X_t} \\
        \nonumber
        &\qquad \leq  \Probbig{ \norm{X_{t+1}'(D) - \E{X_{t+1}'(D, s) \givenplain X_t, (A_t(i))_{i\in D}}}_1 > r - \frac{2|\sspa|}{N} \givenbig X_t} \\
        \nonumber
        &\qquad \leq 2|\sspa| \exp\Big(-\frac{2N (r-2|\sspa|/N)^2}{|\sspa|^2}\Big) \\
        \label{eq:hw:interm-9}
        &\qquad \leq 2|\sspa| \exp\Big(-\frac{N r^2}{|\sspa|^2} + \frac{8}{N}\Big)
    \end{align}
    where we use the fact that $(r - 2|\sspa|/N)^2 \geq r^2 / 2 - (2|\sspa|/N)^2$ in the \eqref{eq:hw:interm-9}.
    Note that the above upper bound also holds for $r< 2|\sspa| / N$ because the term in \eqref{eq:hw:interm-9} is greater than $1$ for $r < 2|\sspa| / N$. Because $\exp(8/N) \leq e^{8}$, we have proved \eqref{eq:hw:interm-goal-2}:
    \[
         \Prob{\norm{X_{t+1}'(D)- X_t(D) P_{\pibs}}_1 > r \givenbig X_t} \leq 2|\sspa| e^{8} \exp\!\Big(-\frac{N r^2}{|\sspa|^2}\Big).
    \]
    Combining with the bound $\hw(X_{t+1}', D) - \rhow \hw(X_t, D) \leq \lamw^{1/2}\norm{X_{t+1}'(D)- X_t(D) P_{\pibs}}_1$ established earlier and the expectation bound \eqref{eq:hw:interm-goal-1}, we obtain
    \begin{align}
        \tag{\ref{eq:hw:drift}}
        \Ebig{(\hw(X_{t+1}', D) - \rhow \hw(X_t, D))^+ \givenbig X_t} &\leq \frac{\KhwOne}{\sqrt{N}} \quad a.s. \\
        \tag{\ref{eq:hw:high-prob}}
        \Probbig{\hw(X_{t+1}', D) > \rhow \hw(X_t, D) + r   \givenbig X_t} &\leq \KhwTwo \exp(-\Chw N r^2) \quad a.s. \quad \forall r \geq 0,
    \end{align}
    with the explicit choices $\KhwOne = 3\lamw^{1/2}|\sspa|$, $\KhwTwo = 2|\sspa|e^{8}$, and $\Chw = 1/(\lamw|\sspa|^2)$.

    \textbf{Proving the analogs of \eqref{eq:hw:drift} and \eqref{eq:hw:high-prob} for $\hu$.}
    Now let $X_{t+1}'$ denote the system at time $t+1$ given that the arms in $D$ follow the $\lppriority$ at time $t$.
    The proof is almost verbatim: We first show
    \begin{align}
        \label{eq:hu:interm-1}
         \EBig{\normBig{ X_{t+1}'(D, s) - \Eplain{X_{t+1}'(D, s) \givenplain X_t, (A_t(i))_{i\in D}}}_1 \givenBig X_t} &\leq \frac{|\sspa|}{\sqrt{N}} \\
         \label{eq:hu:interm-2}
         \ProbBig{\normBig{ X_{t+1}'(D, s) - \Eplain{X_{t+1}'(D, s) \givenplain X_t, (A_t(i))_{i\in D}}}_1 > r \givenBig X_t} &\leq 2|\sspa|\exp\Big(-\frac{2N r^2}{|\sspa|^2}\Big),
    \end{align}
    using identical arguments. Then we can show that
    \begin{align}
        \label{eq:hu:interm-3}
        \normbig{\E{X_{t+1}'(D, s) \givenplain X_t, (A_t(i))_{i\in D}} - \E{X_{t+1}'(D, s) \givenplain X_t}}_1
        &\leq \frac{4|\sspa|}{N} \quad a.s.,
    \end{align}
    using the fact that under \lppriority (\Cref{alg:follow-lp-priority}), for all $s'\in\sspa$ and $a\in\aspa$,
    \begin{align*}
        \abs{\pibs(a|s') X_t(D, s')  -  \Gamma_t(s',a)} &\leq \frac{1}{N} \quad a.s. \quad \forall s'\neq \sneu \\
        \abs{\pibs(a|\sneu) X_t(D, \sneu)  -  \Gamma_t(\sneu,a)} &\leq \frac{|\sspa|}{N} \quad a.s.
    \end{align*}
    Finally, we combine \eqref{eq:hu:interm-1},  \eqref{eq:hu:interm-2} and \eqref{eq:hu:interm-3} to get that under $\hu$,
    \begin{align}
        \label{eq:hu:drift}
        \Ebig{(\hu(X_{t+1}', D) - \rhou \hu(X_t, D))^+ \givenbig X_t} &\leq \frac{\KhuOne}{\sqrt{N}} \quad a.s. \\
        \label{eq:hu:high-prob}
        \Probbig{\hu(X_{t+1}', D) > \rhou \hu(X_t, D) + r   \givenbig X_t} &\leq \KhuTwo \exp(-\Chu N r^2) \quad a.s. \quad \forall r \geq 0,
    \end{align}
    for explicit positive constants $\KhuOne, \KhuTwo, \Chu$ derived analogously to $\KhwOne, \KhwTwo, \Chw$.

    \textbf{The rest of the inequalities in this lemma.}
    The other two inequalities, \eqref{eq:hw:strength} and \eqref{eq:hw:lipschitz}, and their analogs for $\hu$ are relatively straightforward from the definitions of the weighted $L_2$ norms, so we skip them here. Similar calculations have also been done in \citep{HonXieCheWan_24} for $\hw$, so we refer the readers to \citep{HonXieCheWan_24} for the detailed calculations.
\end{proof}

\section{Proof of \hcref{lem:uus-ust-properties}}
\label{app:proof-uus-ust}

\uusustproperties*

\begin{proof} %
    To show that $\Uus$ and $\Ust$ are well-defined, note that the moduli of all eigenvalues of $V \Jus^{\dagger} V^{-1}$ and $V \Jst V^{-1}$ are strictly less than $1$, so their $k$-th powers converge exponentially to zero as $k\to\infty$, implying the convergence of the series in \eqref{eq:Uus-def} and \eqref{eq:Ust-def}. Our real Jordan decomposition makes these matrices real. Thus each term in the two series is positive semidefinite, and both series converge to positive semidefinite matrices.

    To prove the norm lower bound \eqref{eq:Uus-expansion}, we first compute
    \begin{align}
        \nonumber
        \Phi \Uus \Phi^\top
        &= \sum_{k=1}^\infty \big(V J \Jus^{\dagger k} V^{-1}\big)\big(V J \Jus^{\dagger k} V^{-1}\big)^\top \\
        \nonumber
        &= \big(V J \Jus^{\dagger} V^{-1}\big)\big(V J \Jus^{\dagger} V^{-1}\big)^\top +
        \sum_{k=2}^\infty \big(V J \Jus^{\dagger k} V^{-1}\big)\big(V J \Jus^{\dagger k} V^{-1}\big)^\top \\
        \nonumber
        &= \big(V \Ius V^{-1}\big)\big(V \Ius V^{-1}\big)^\top +
        \sum_{k=2}^\infty \big(V \Jus^{\dagger (k-1)} V^{-1}\big)\big(V \Jus^{\dagger (k-1)} V^{-1}\big)^\top  \\
        \label{eq:uus-matrix-identity}
        &= \big(V \Ius V^{-1}\big)\big(V \Ius V^{-1}\big)^\top + \Uus.
    \end{align}
    We claim that there exists a positive constant $\gapus'$ such that for any $v \in \R^{|\sspa|}$,
    \begin{equation}
        \label{eq:ius-bound-uus}
        v\big(V \Ius V^{-1}\big)\big(V \Ius V^{-1}\big)^\top v^\top \geq \gapus' v \Uus v^\top,
    \end{equation}
    To prove \eqref{eq:ius-bound-uus}, let $\spus$ be the linear subspace spanned by rows of $V^{-1}$ corresponding to the non-zero entries on $\Ius$'s diagonal, and let $\spst$ be the linear subspace spanned by the remaining rows. Then $\R^{|\sspa|}=\spus\oplus\spst$, and we can write
    \[
        v=v^\parallel+v^\perp,\qquad
        v^\parallel=vV\Ius V^{-1}\in\spus,\qquad
        v^\perp=vV\Ist V^{-1}\in\spst.
    \]
    These components need not be orthogonal. The matrix $V\Ius V^{-1}$ acts as the identity on $\spus$ and annihilates $\spst$. Also, $\spst$ is the left kernel of $\Uus$: every vector in $\spst$ annihilates every factor $(V\Jus^\dagger V^{-1})^k$ in $\Uus$'s defining series, and the $k=1$ term has no other vectors in its left kernel. Consequently,
    \[
        v \Uus v^\top =  (v^\parallel + v^\perp) \Uus (v^\parallel+v^\perp)^\top = v^\parallel \Uus (v^\parallel)^\top.
    \]
    If $v^\parallel=0$, both sides of \eqref{eq:ius-bound-uus} are zero. For $v^\parallel\neq0$, there is a finite positive constant $M$, independent of $v$, such that
    \[
        v\Uus v^\top
        \leq M\normbig{v^\parallel}_2^2
        =M\norm{vV\Ius V^{-1}}_2^2.
    \]
    Indeed, the continuous function $w\mapsto w\Uus w^\top$ is bounded on the compact set $\{w\in\spus:\norm{w}_2=1\}$; taking $w=v^\parallel/\norm{v^\parallel}_2$ yields the previous displayed inequality. Taking $\gapus'=1/M$ proves \eqref{eq:ius-bound-uus}. We can now bound $\norm{v\Phi}_{\Uus}$:
    \begin{align*}
        \norm{v \Phi}_{\Uus}
        &= \sqrt{v \Phi \Uus \Phi^\top v^\top} \\
        &\geq \sqrt{(1 + \gapus') v \Uus v^\top} \\
        &= \sqrt{1+\gapus'} \norm{v}_{\Uus}.
    \end{align*}
    The inequality uses \eqref{eq:uus-matrix-identity} and \eqref{eq:ius-bound-uus}. Letting $\gapus = \sqrt{1+\gapus'} - 1$, we obtain \eqref{eq:Uus-expansion}.

    Next, we prove \eqref{eq:Uus-expansion-max}.
    For any $v\in \R^{|\sspa|}$, we have $v \Uus v^\top \geq \normbig{v V \Jus^\dagger V^{-1}}_2^2$.
    We claim that there exists a positive constant $\gapusmax'$ such that
    \begin{equation}
         \label{eq:ius-bound-uus-reverse}
        \norm{v V \Ius V^{-1}}_2 \leq \gapusmax' \normbig{v V \Jus^\dagger V^{-1}}_2.
    \end{equation}
    To prove \eqref{eq:ius-bound-uus-reverse}, write $v=v^\parallel+v^\perp$. Both sides are unchanged when $v$ is replaced by $v^\parallel$, and both vanish if $v^\parallel=0$. The map $w\mapsto wV\Jus^\dagger V^{-1}$ is invertible on $\spus$. Thus the set $\{w\in\spus:\normbig{wV\Jus^\dagger V^{-1}}_2=1\}$ is compact, and $\norm{w}_2=\norm{wV\Ius V^{-1}}_2$ is bounded there. Apply this bound to $w=v^\parallel/\normbig{v^\parallel V\Jus^\dagger V^{-1}}_2$ to obtain \eqref{eq:ius-bound-uus-reverse}. It follows that
    \begin{align*}
        \norm{v \Phi}_{\Uus}
        &= \sqrt{v\Phi \Uus \Phi^\top v^\top} \\
        &= \sqrt{v\big(V \Ius V^{-1}\big)\big(V \Ius V^{-1}\big)^\top v^\top + v\Uus v^\top} \\
        &\leq \sqrt{(\gapusmax')^2 +1 } \norm{v}_{\Uus}.
    \end{align*}
    The second equality uses \eqref{eq:uus-matrix-identity}. The inequality uses \eqref{eq:ius-bound-uus-reverse} and $v\Uus v^\top\geq\normbig{vV\Jus^\dagger V^{-1}}_2^2$, which follows from the first term of the series defining $\Uus$. Taking $\gapusmax = \sqrt{(\gapusmax')^2 +1 } - 1$ establishes \eqref{eq:Uus-expansion-max}.

    To prove \eqref{eq:Ust-contraction}, we compute
    \begin{align}
        \nonumber
        \Phi \Ust \Phi^\top
        &= V J \Ist V^{-1}  \big(V J \Ist V^{-1}\big)^{\top} + \sum_{k=1}^\infty \big(V J \Jst^k V^{-1} \big) \big(V J \Jst^k V^{-1}\big)^{\top} \\
        \nonumber
        &= \sum_{k=1}^\infty \big(V \Jst^k V^{-1} \big) \big(V \Jst^k V^{-1}\big)^{\top} \\
        \label{eq:ust-matrix-identity}
        &= \Ust -  V \Ist V^{-1}  \big(V \Ist V^{-1}\big)^{\top}.
    \end{align}
    We claim that there exists a positive constant $\gapst'$ such that for any $v\in \R^{|\sspa|}$,
    \begin{equation}
        \label{eq:ist-bound-ust}
        v\big(V \Ist V^{-1}\big)\big(V \Ist V^{-1}\big)^\top v^\top \geq \gapst' v \Ust v^\top.
    \end{equation}
    To obtain \eqref{eq:ist-bound-ust}, apply the compactness argument for \eqref{eq:ius-bound-uus} on $\spst$ instead of $\spus$. Here $\spus$ is the left kernel of $\Ust$, as follows from $\Ust$'s defining series and its first term $V\Ist V^{-1}(V\Ist V^{-1})^\top$. Thus $v\Ust v^\top=v^\perp\Ust(v^\perp)^\top$ and $\norm{vV\Ist V^{-1}}_2=\norm{v^\perp}_2$; the case $v^\perp=0$ again gives equality on both sides.
    Since $\Phi\Ust\Phi^\top$ is positive semidefinite, \eqref{eq:ust-matrix-identity} gives $V\Ist V^{-1}(V\Ist V^{-1})^\top\preceq\Ust$. Also, $\Ust\neq0$ because $\Phi$ has the eigenvalue $0$ in its stable part. Thus \eqref{eq:ist-bound-ust} implies $0<\gapst'\leq1$. We obtain
    \begin{align*}
        \norm{v\Phi}_{\Ust} &= \sqrt{v\Phi \Ust \Phi^\top v^\top} \\
        &\leq \sqrt{(1-\gapst')v\Ust v^\top} \\
        &= \sqrt{1-\gapst'} \norm{v}_{\Ust}.
    \end{align*}
    The inequality uses \eqref{eq:ust-matrix-identity} and \eqref{eq:ist-bound-ust}. Taking $\gapst = 1 - \sqrt{1-\gapst'}\in(0,1]$ gives \eqref{eq:Ust-contraction}.

    To show \eqref{eq:Uus-plus-Ust-strength}, note that
    \begin{align*}
        \norm{v}_{\Uus} +  \norm{v}_{\Ust}
        &= \sqrt{v \Uus v^\top} + \sqrt{v \Ust v^\top}  \\*
        &\geq \sqrt{v (\Uus+\Ust) v^\top} \\*
        &= \norm{v}_{\Uus+\Ust}.
    \end{align*}
    The inequality is the Euclidean triangle inequality applied to $(\sqrt{v\Uus v^\top},0)$ and $(0,\sqrt{v\Ust v^\top})$. Since all norms in a finite-dimensional Euclidean space are equivalent, it suffices to show $\norm{\cdot}_{\Uus+\Ust}$ is a norm, or equivalently, $\Uus+\Ust$ is positive definite. For any $v\in \R^{|\sspa|}$, we have
    \begin{align*}
        v (\Uus+\Ust) v^\top
        &\geq v \Big(\big(V \Jus^{\dagger} V^{-1}\big) \big(V \Jus^{\dagger} V^{-1}\big)^{\top} + V \Ist V^{-1}  \big(V \Ist V^{-1}\big)^{\top} \Big) v^\top \\*
        &= vV \big(\Jus^\dagger V^{-1} V^{-\top}(\Jus^\dagger)^\top+\Ist V^{-1} V^{-\top} \Ist  \big)   V^\top v^\top.
    \end{align*}
    The last expression equals $\normbig{vV\Jus^\dagger V^{-1}}_2^2+\norm{vV\Ist V^{-1}}_2^2$. If $v(\Uus+\Ust)v^\top=0$, both squared norms vanish. The first implies $v\in\spst$, and the second implies $v\in\spus$. Since $\spus\cap\spst=\{0\}$, we obtain $v=0$. Therefore, $\Uus+\Ust$ is positive definite, which implies \eqref{eq:Uus-plus-Ust-strength}.
\end{proof}

\section{Examples and remarks on the assumptions}
\subsection{Examples of violating \hcref{assump:aperiodic-unichain}--\hcref{assump:non-degeneracy}} \label{app:examples_vio_assump}

Below, we give some examples that violate \Cref{assump:aperiodic-unichain} or \Cref{assump:non-degeneracy}, where the gap to the LP upper bound cannot be $O(\exp(-CN))$.

\subsubsection*{An aperiodic but non-unichain example.}
We first give an example of RBs where $P_\pibs$ is aperiodic but not unichain.
Consider the RB problem whose single-armed MDP has two states, $A$ and $B$, with no transitions between these two states.
The reward function is given by $r(A, 1) = r(B, 0) = 1$ and $r(A, 0) = r(B, 1) = 0$.
Let $\alpha = 1/2$ in the budget constraint.
Since there is no transition between these two states, the Markov chain induced by the optimal single-armed policy, $\pibs$, $P_\pibs$, is not unichain. However, each of the recurrent classes $\{A\}$ and $\{B\}$ is aperiodic in this Markov chain.

It is easy to verify that the optimal solution to the LP relaxation \eqref{eq:lp-single} is  $y^*(A,1) = y^*(B, 0) = 1/2$ and $y^*(s,a) = 0$ for the rest of $s\in\sspa, a\in\aspa$; the optimal value of  \eqref{eq:lp-single} is $\rrel = 1$.
Now consider the $N$-armed problem where all arms are initialized in state $A$. In this case, the long-run average reward of any policy is $1/2$. Therefore, $\rrel - \rsysn \geq 1/2 = \Omega(1/\sqrt{N})$ for any $\pi$.

\subsubsection*{A unichain but periodic example.}
Next, we give an example of RBs where  $P_\pibs$ is unichain but periodic, and $\rrel - \rsysn = \Theta(1)$. This example was included in \citep{HonXieCheWan_24}, and we paraphrase it here for completeness.

Consider the RB problem whose single-armed MDP has two states, $A$ and $B$.
The state of the single-armed MDP deterministically transitions to the other state in each time step, regardless of the action applied.
Let the reward function be $r(A,0)=r(B,1)=1$ and $r(A,1)=r(B,0)=0$; let $\alpha = 1/2$ in the budget constraint.
Obviously, the Markov chain induced by any single-armed policy is unichain and periodic with a period $2$.

One can easily see the optimal solution to the LP relaxation \eqref{eq:lp-single} is by $y^*(A, 0) = y^*(B, 1) = 1/2$ and $y^*(A, 1) = y^*(B, 0) = 0$; the optimal value of \eqref{eq:lp-single} is $\rrel = 1$.
Now consider the $N$-armed problem where all arms are initialized in state $A$. Then at any time $t$, either all arms are in state $A$ or all arms are in state $B$. In this case, under any policy $\pi$, the long-run average $\rsysn$ is $1/2$, so $\rrel - \rsysn= 1/2 = \Omega(1/\sqrt{N})$.

\subsubsection*{An example satisfying \hcref{assump:aperiodic-unichain} but not \hcref{assump:non-degeneracy}.}

Next, we give an example of RB satisfying \Cref{assump:aperiodic-unichain} but violating \Cref{assump:non-degeneracy}.
Suppose each arm in the $N$-armed problem has state $S_t(i) = \text{Ber}(1/2)$ i.i.d.\ across $i$ and $t$. Let $\alpha = 1/2$ in the budget constraint. The reward function $r(s,1) = s$ and $r(s,0) = 0$ for $s = 0, 1$.
Intuitively, the decision maker observes the state of each arm, and pulls $N/2$ of them. Each arm pulled by the decision maker generates an amount of reward equal to its state.
The single-armed MDP is clearly aperiodic unichain under any policy because its state is i.i.d.

The optimal solution of the LP relaxation \eqref{eq:lp-single} is $y^*(1,1) = y^*(0,0) = 1/2$, and $y^*(s,a) = 0$ for other $s\in\sspa, a\in\aspa$; the optimal value $\rrel = 1/2$. Intuitively, the optimal single-armed policy pulls the arm if and only if it is in state $1$.

In the $N$-armed problem, the optimal policy is obviously pulling $N/2$ arms with the largest states. Then the optimal reward $\ropt$ is given by
\[
    \ropt = \frac{1}{N}\EBig{\min\Big(\sumN S_i(t), \frac{N}{2}\Big)},
\]
for an arbitrary $t$.
Since $S_t(i) = \text{Ber}(1/2)$ and are i.i.d.\ across $i$, it can be shown that $\ropt = 1/2 - \Theta(1/\sqrt{N})$. Therefore, $\rrel - \ropt = \Theta(1/\sqrt{N})$.

\subsection{Remarks on \hcref{assump:orthogonal-noise} of regular unstable RBs (\hcref{def:regular-unstable})}
\label{app:remark-noise-assumption}

In this subsection, we prove an intuitive sufficient condition for Property \ref{assump:orthogonal-noise} of regular unstable RBs and also provide a weaker version of this property.

We first state and prove the intuitive sufficient condition for Property \ref{assump:orthogonal-noise}:

\begin{lemma}\label{lem:assump-6-helper}
    Consider an RB instance with the single-armed MDP $(\sspa, \aspa, P, r)$, where $P(s,a,s') > 0$ for all $s\in\sspa, a\in\aspa, s'\in\sspa$.
    Then this RB instance satisfies Property \ref{assump:orthogonal-noise} in \Cref{def:regular-unstable}.
\end{lemma}

\begin{proof}
    Throughout this proof, we fix $X_t$ and $\veA_t$ and omit them in the conditioning when writing expectations or probabilities.

    Recalling Property \ref{assump:orthogonal-noise}, we will find a constant $\noiseortho>0$, independent of the unit vector $\xi$, $N$, and $t$, such that for every $\xi\in\R^{|\sspa|}$ with $\normplain{\xi}_2=1$ and $\statdist\xi^\top=0$, whenever \lppriority is applicable, we have
    \begin{equation}
        \label{eq:pf-assumption-6-helper-lemma:goal}
        \Ebig{\absbig{\big(X_{t+1}([N]) - \statdist - (X_t([N]) - \statdist)\Phi \big)  \xi^\top } \givenbig X_t, \veA_t} \geq \frac{\noiseortho}{\sqrt{N}},
    \end{equation}
    where $\veA_t$ follows the \lppriority.
    Note that in this case, the parameter $\radiusus$ in Property \ref{assump:orthogonal-noise} of \Cref{def:regular-unstable} could be any positive number such that $\norm{X_t([N]) - \statdist}_1 \leq \radiusus$ guarantees the condition \eqref{eq:lp-priority-condition-1} for applying the \lppriority.

    We first show that for any $\xi \in \R^{|\sspa|}$ such that $\norm{\xi}_2 = 1$ and $\statdist \xi^\top = 0$,
    \begin{equation}
        \label{eq:xi-max-minus-min}
        \ximax - \ximin \geq 1/\sqrt{|\sspa|}.
    \end{equation}
    where $\ximax = \max_{s\in\sspa} \xi(s)$ and $\ximin = \min_{s\in\sspa} \xi(s)$.
    Observe that we must have either $\ximax \geq 1/\sqrt{|\sspa|}$ or $\ximin \leq - 1/\sqrt{|\sspa|}$, because otherwise $\norm{\xi}_2 < 1$.
    We first show \eqref{eq:xi-max-minus-min} for the case with $\ximax \geq 1/\sqrt{|\sspa|}$. In this case, we must have $\ximin \leq 1/\sqrt{|\sspa|}$, because otherwise $\norm{\xi}_2 > 1$.
    Let $\vonescaled \in \R^{|\sspa|}$ be the vector whose entries are all $1/\sqrt{|\sspa|}$. Then
    \[
        \norm{\xi -\vonescaled}_\infty \leq \max\big(\ximax - 1/\sqrt{|\sspa|}, 1/\sqrt{|\sspa|} - \ximin\big) \leq \ximax - \ximin.
    \]
    On the other hand, because $\norm{\statdist}_1 = 1$ and $\statdist \vonescaled^\top = 1/\sqrt{|\sspa|}$,
    \[
         \norm{\xi - \vonescaled}_\infty \geq \statdist \vonescaled^\top - \statdist \xi^\top = 1/\sqrt{|\sspa|}.
    \]
    Combining the last two inequalities yields \eqref{eq:xi-max-minus-min}. For the case with $\ximin \leq -1/\sqrt{|\sspa|}$, \eqref{eq:xi-max-minus-min} can be proved in an analogous way.

    To show \eqref{eq:pf-assumption-6-helper-lemma:goal}, we invoke the Paley-Zygmund inequality to get
    \begin{align}
        \nonumber
        & \Prob{\big(X_{t+1}([N]) \xi^\top -\Eplain{X_{t+1}([N]) \xi^\top}\big)^2  > \frac{1}{2}\Var{X_{t+1}([N]) \xi^\top}} \\
        \label{eq:pf-assumption-6-helper-lemma:interm-0}
        &\qquad \geq \frac{1}{4} \frac{\Var{X_{t+1}([N]) \xi^\top}^2}{\E{\big(X_{t+1}([N]) \xi^\top -\Eplain{X_{t+1}([N]) \xi^\top}\big)^4}}.
    \end{align}

    Next, we lower bound the second moment of $X_{t+1}([N]) \xi^\top -\Eplain{X_{t+1}([N]) \xi^\top}$ in \eqref{eq:pf-assumption-6-helper-lemma:interm-0}.
    For any timestep $t$, given $X_t, \veA_t$, we have
    \begin{align}
        X_{t+1}([N]) \xi^\top
        &= \sum_{s\in\sspa} X_{t+1}([N], s) \xi(s) \\
        &= \frac{1}{N} \sum_{s\in\sspa}  \sumN \indibrac{\veS_{t+1}(i) = s} \xi(s) \\
        &= \frac{1}{N} \sumN \xi(\veS_{t+1}(i)).
    \end{align}
    Note that conditioned on $X_t, \veA_t$, $\xi(\veS_{t+1}(i))$ are independent across $i$. Therefore, the variance of $X_{t+1}([N]) \xi^\top$ is given by
    \[
        \Var{X_{t+1}([N]) \xi^\top} = \frac{1}{N^2} \sumN \Var{\xi(\veS_{t+1}(i))}.
    \]
    Each arm's contribution to this variance satisfies
    \begin{align*}
        \Var{\xi(\veS_{t+1}(i))}
        &\geq \min_{s,a,s'} P(s,a,s') \Big( (\ximax - \E{\xi(\veS_{t+1}(i))})^2 + (\ximin - \E{\xi(\veS_{t+1}(i))})^2 \Big) \\
        &\geq \frac{1}{2} \min_{s,a,s'} P(s,a,s')  (\ximax - \ximin)^2 \\
        &\geq \frac{\min_{s,a,s'} P(s,a,s')}{2|\sspa|},
    \end{align*}
    Here $\min_{s,a,s'}$ is a shorthand for $\min_{s,s'\in\sspa, a\in\aspa}$. The first inequality retains the states attaining $\ximax$ and $\ximin$, each of which has probability at least $\min_{s,a,s'}P(s,a,s')$. The second uses $(x-c)^2+(y-c)^2\geq(x-y)^2/2$, and the third uses \eqref{eq:xi-max-minus-min}.
    Therefore,
    \begin{equation}
        \label{eq:pf-assumption-6-helper-lemma:interm-1}
        \Var{X_{t+1}([N]) \xi^\top} \geq \frac{1}{N^2} \sumN \Var{\xi(\veS_{t+1}(i))} \geq \frac{\min_{s,a,s'} P(s,a,s')}{2|\sspa|N}.
    \end{equation}

    The corresponding fourth moment satisfies
    \begin{align}
        \nonumber
        & \E{\big(X_{t+1}([N]) \xi^\top -\Eplain{X_{t+1}([N]) \xi^\top}\big)^4} \\
        \nonumber
        &\qquad = \frac{1}{N^4} \Big( \sumN \E{\big(\xi(\veS_{t+1}(i)) - \E{\xi(\veS_{t+1}(i))}\big)^4}
        + 6 \sum_{1\leq i < j \leq N} \Var{\xi(\veS_{t+1}(i))} \Var{\xi(\veS_{t+1}(j))}
        \Big) \\
        \nonumber
        &\qquad \leq \frac{1}{N^4}\Big(N (\ximax - \ximin)^4 + 6 \frac{N(N-1)}{2} (\ximax - \ximin)^4 \Big) \\
        \label{eq:pf-assumption-6-helper-lemma:interm-2}
        &\qquad \leq \frac{3(\ximax-\ximin)^4}{N^2}.
    \end{align}
    The equality uses $X_{t+1}([N])\xi^\top=N^{-1}\sumN\xi(\veS_{t+1}(i))$ and conditional independence across arms. The first inequality bounds each centered summand in absolute value by $\ximax-\ximin$; the second uses $N+3N(N-1)\leq3N^2$.

    Substituting the calculations in \eqref{eq:pf-assumption-6-helper-lemma:interm-1} and \eqref{eq:pf-assumption-6-helper-lemma:interm-2} back to \eqref{eq:pf-assumption-6-helper-lemma:interm-0}, we get
    \begin{align}
        \nonumber
        & \Prob{\big(X_{t+1}([N]) \xi^\top -\Eplain{X_{t+1}([N]) \xi^\top}\big)^2  >  \frac{\min_{s,a,s'} P(s,a,s')}{4|\sspa|N}} \\
        \nonumber
        &\qquad \geq \Big(\frac{\min_{s,a,s'} P(s,a,s')}{2|\sspa|N}\Big)^2  \frac{N^2}{12(\ximax-\ximin)^4} \\
         \label{eq:pf-assumption-6-helper-lemma:interm-3}
        &\qquad = \frac{\min_{s,a,s'} P(s,a,s')^2}{48|\sspa|^2 (\ximax-\ximin)^4}.
    \end{align}
    We next lower bound the expected absolute deviation from the conditional mean. Let $p_{\min}=\min_{s,a,s'}P(s,a,s')>0$. Then
    \begin{align*}
        \E{\absbig{X_{t+1}([N])\xi^\top-\E{X_{t+1}([N])\xi^\top}}}
        &\geq\frac{\sqrt{p_{\min}}}{2\sqrt{|\sspa|N}}\,
            \frac{p_{\min}^2}{48|\sspa|^2(\ximax-\ximin)^4}\\*
        &\geq\frac{p_{\min}^{5/2}}{1536|\sspa|^{5/2}\sqrt N}.
    \end{align*}
    The first inequality multiplies the threshold $\sqrt{p_{\min}}/(2\sqrt{|\sspa|N})$ by the probability bound in \eqref{eq:pf-assumption-6-helper-lemma:interm-3}. The second uses $\ximax-\ximin\leq2$, which follows from $\norm{\xi}_2=1$.

    To obtain \eqref{eq:pf-assumption-6-helper-lemma:goal}, we use
    \begin{align*}
        &\E{\absbig{X_{t+1}([N])\xi^\top-\E{X_{t+1}([N])\xi^\top}}}\\*
        &\qquad\leq 2\E{\absbig{\big(X_{t+1}([N])-\statdist-(X_t([N])-\statdist)\Phi\big)\xi^\top}}.
    \end{align*}
    Indeed, the triangle inequality and Jensen's inequality give, for any integrable scalar random variable $\zeta$ and constant $a$,
    $\E{|\zeta-\E{\zeta}|}\leq\E{|\zeta-a|}+|\E{\zeta}-a|\leq2\E{|\zeta-a|}$.
    Apply this with $\zeta=X_{t+1}([N])\xi^\top$ and $a=\big(\statdist+(X_t([N])-\statdist)\Phi\big)\xi^\top$, which is fixed conditional on $X_t,\veA_t$.
    Consequently, \eqref{eq:pf-assumption-6-helper-lemma:goal} holds with
    \[
        \noiseortho=\frac{p_{\min}^{5/2}}{3072|\sspa|^{5/2}}>0.
    \]
    This constant is uniform over $\xi$, $X_t$, $\veA_t$, and $N$.
\end{proof}

Next, we state the weaker version of Property \ref{assump:orthogonal-noise} of regular unstable RBs.

\begin{property}[Weaker version of Property  \ref{assump:orthogonal-noise} in \Cref{def:regular-unstable}]\label{assump:weak-orthogonal-noise}
    There exist positive constants $\radiusus$ and $\noiseortho$ with $\radiusus \leq \min\{y^*(\sneu, 0), y^*(\sneu, 1)\}-1/N$, such that for any large-enough $N$ and any time step $t$, we have
    \begin{equation*}
        \Ebig{\normbig{X_{t+1}([N]) - \statdist - (X_t([N])-\statdist)\Phi  }_{\Uus} \givenbig X_t, \veA_t} \geq \frac{\lamus^{1/2} \noiseortho}{\sqrt{N}} \quad \text{ if } \norm{X_t([N]) - \statdist}_1 \leq \radiusus,
    \end{equation*}
    where $\veA_t$ follow the \lppriority (\Cref{alg:follow-lp-priority}), which is always feasible if $\norm{X_t([N]) - \statdist}_1 \leq \min\{y^*(\sneu, 0), y^*(\sneu, 1)\}-1/N$.
\end{property}

Intuitively, \Cref{assump:weak-orthogonal-noise} assumes that when $X_t([N])$ is close to $\statdist$, the next-step state-count vector $X_{t+1}([N])$ has enough randomness, as measured by the conditional expected $\Uus$-seminorm of its deviation from the affine prediction $\statdist+(X_t([N])-\statdist)\Phi$.

The proof of \Cref{thm:instability-lower-bound} also holds with Property \ref{assump:orthogonal-noise} replaced by \Cref{assump:weak-orthogonal-noise}. In the proof of \Cref{lem:hus-drift}, use \Cref{assump:weak-orthogonal-noise} directly to obtain \eqref{eq:pf-hus-drift:interm-5}; the rest of the proof is unchanged.

\subsection{Proof of \hcref{cor:exist-regular-unstable}}\label{app:locally-unstable-example}

\existregularunstable*

\begin{proof}
    We will verify that the example in Figure~1 of \citep{HonXieCheWan_23} is regular unstable.

    In this example, the single-armed MDP $(\sspa, \aspa, P, r)$ is defined as below: $\sspa = \{1,2,3\}$; $\aspa=\{0,1\}$; the transition kernel $P$ is given by
    \begin{align*}
    P( \cdot ,0,\cdot ) \ &=\ \begin{bmatrix}
    0.02232142 & 0.10229283 & 0.87538575\\
    0.03426605 & 0.17175704 & 0.79397691\\
    0.52324756 & 0.45523298 & 0.02151947
    \end{bmatrix}, \\
    P( \cdot ,1,\cdot ) \ &=\ \begin{bmatrix}
    0.14874601 & 0.30435809 & 0.54689589\\
    0.56845754 & 0.41117331 & 0.02036915\\
    0.25265570 & 0.27310439 & 0.4742399
    \end{bmatrix};
    \end{align*}
    The reward function $r$ is defined as
    \begin{align*}
        r( \cdot , 0) \ &=
        \ \begin{bmatrix}
            0 & 0 & 0
            \end{bmatrix}, \\
        r( \cdot , 1) \ &=
        \ \begin{bmatrix}
            0.37401552 & 0.11740814 & 0.07866135
            \end{bmatrix}.
    \end{align*}
    Let $\alpha = 0.4$ in the budget constraint.
    The optimal solution of the LP relaxation \eqref{eq:lp-single} is:
    \[
        y^*( \cdot ,\cdot ) \ =\ \begin{bmatrix}
            0 & 0.29943\\
            0.23768 & 0.10057\\
            0.36232 & 0
            \end{bmatrix},
    \]
    where the first column corresponds to $y^*(s,0)$ and the second column corresponds to $y^*(s, 1)$.

    Since $P(s,a,s')> 0$ for all $s,a,s'$, $P_\pibs$ is an aperiodic unichain, so \Cref{assump:aperiodic-unichain} holds. Also, from $y^*$, we can see that the state $2$ is the unique neutral state, implying \Cref{assump:non-degeneracy}.
    Moreover, one can numerically verify that $y^*$ is the unique solution to \eqref{eq:lp-single}, so Property \ref{assump:unique-and-communicating} of \Cref{def:regular-unstable} is satisfied;
    the moduli of the eigenvalues of $\Phi$ are $\{1.133, 0, 0.059\}$, so Property \ref{assump:instability} of \Cref{def:regular-unstable} is satisfied.

    Finally, observe that in this example, $P(s,a,s') > 0$ for any $s,s'\in\sspa$ and $a\in\aspa$. Therefore, we invoke the sufficient condition \Cref{lem:assump-6-helper} proved in \Cref{app:remark-noise-assumption} to conclude that this example satisfies Property \ref{assump:orthogonal-noise} of \Cref{def:regular-unstable}.
\end{proof}

\section{Experiment details}
\label{app:exp-details}

\subsection{Details of Two-Set Policy}
\label{app:exp-details:two-set}

\subsubsection{Determine $\eta$ and $\errfe$}
In the analysis, $\eta = |\sspa|^{-1/2} \min\{y^*(\sneu,1), y^*(\sneu,0)\}$, $\errfe = (|\sempty|+1)/N$.
In the simulations, we will try to choose a pair of feasibility-ensuring $(\eta, \errfe)$ with $\eta$ being as large as possible. 
Recall that $(\eta, \errfe)$ is feasibility-ensuring if for any system state $x$ and $D\subseteq[N]$ satisfying
\begin{equation}
    \label{eq:feasible-ensuring-exp-restate}
    \norm{x(D) - m(D)\statdist}_U \leq \eta m(D) - \errfe,
\end{equation}
we have
\begin{align}
    \label{eq:lp-priority-condition-1-restate}
    \sum_{s\neq \sneu} \pibs(1|s) x(D,s) &\leq \alpha m(D) - \frac{|\sempty|+1}{N}  \quad \text{ and },\\
    \label{eq:lp-priority-condition-2-restate}
    \sum_{s\neq \sneu} \pibs(0|s) x(D, s) &\leq (1-\alpha) m(D) - \frac{|\sempty|+1}{N}.
\end{align}
Equivalently, $(\eta, \errfe)$ is feasibility-ensuring if whenever 
\begin{align}
    \label{eq:lp-priority-condition-1-converse}
    \sum_{s\neq \sneu} \pibs(1|s) x(D, s) &\geq  \alpha m(D) - \frac{|\sempty|+1}{N} \quad \text{ or },\\
    \label{eq:lp-priority-condition-2-converse}
    \sum_{s\neq \sneu} \pibs(0|s) x(D, s) &\geq (1-\alpha)m(D) - \frac{|\sempty|+1}{N},
\end{align}
we have $\norm{x(D) - m(D) \statdist}_\umat \geq \eta m(D) - \errfe$. 
Thus, ignoring the $O(1/N)$ term, we consider solving the following two Second-Order Conic Programs (SOCPs), where the decision variable $\ve{z}$ represent $X_t(D) / m(D)$ for some non-empty $D$:
\begin{subequations}
\begin{align}
    \label{eq:solve-eta-1-begins}
    \underset{\ve{z}\in \R^{|\sspa|}}{\text{minimize}} \quad &   \norm{\ve{z} - \statdist}_{\umat}  \\
    \text{subject to} \quad & \sum_{s\neq \sneu} \pibs(1|s) \ve{z}(s) \geq \alpha \\
    & \ve{z} \vone^\top = 1\\
    \label{eq:solve-eta-1-ends}
    & \ve{z} \geq 0 
\end{align}
\end{subequations}
\begin{subequations}
\begin{align}
    \label{eq:solve-eta-2-begins}
    \underset{\ve{z}\in \R^{|\sspa|}}{\text{minimize}} \quad &   \norm{\ve{z} - \statdist}_{U}  \\
    \text{subject to} \quad & \sum_{s\neq \sneu} \pibs(0|s) \ve{z}(s) \geq 1-\alpha \\
    & \ve{z} \vone^\top = 1\\
    \label{eq:solve-eta-2-ends}
    & \ve{z} \geq 0.
\end{align}
\end{subequations}

The following proposition shows that one can obtain a pair of feasibility-ensuring ($\eta$, $\errfe$) from \eqref{eq:solve-eta-1-begins} and \eqref{eq:solve-eta-2-begins} if the two SOCPs are both feasible. 

\begin{proposition}\label{prop:socp-feasible-ensuring}
    Suppose the SOCPs in \eqref{eq:solve-eta-1-begins}--\eqref{eq:solve-eta-1-ends} and \eqref{eq:solve-eta-2-begins}--\eqref{eq:solve-eta-2-ends} are feasible; let their optimal values be $\eta_1$ and $\eta_2$, respectively. Let $\eta = \min(\eta_1, \eta_2)$ and let $\errfe = 2\sqrt{2} \lamu^{1/2} (|\sempty|+1) / N$. Then for any system state $x$ and $D\in [N]$ satisfying
    \eqref{eq:feasible-ensuring-exp-restate}, we have \eqref{eq:lp-priority-condition-1-restate} and \eqref{eq:lp-priority-condition-2-restate}. 
\end{proposition}

\begin{proof}
    In this proof, we write $\epsilon_0 = (|\sempty|+1) / N$ for simplicity. To get a contradiction, suppose \eqref{eq:lp-priority-condition-1-restate} does not hold, i.e.,
    \[
        \sum_{s\neq \sneu} \pibs(1|s) x(D, s) > \alpha m(D) - \epsilon_0.
    \]
    Let $p_{\max} = \max_{s\neq \sneu} \pibs(1|s)$. 
    Because \eqref{eq:solve-eta-1-begins}--\eqref{eq:solve-eta-1-ends} is feasible, we must have $p_{\max} \geq \alpha$. 
    We define the sets $S^a$ and $S^b$ to be
    \begin{align*}
        S^a &= \{s\in\sspa\colon s\neq \sneu, \pibs(1|s) = p_{\max}  \}\\
        S^b &=  \{s\in\sspa \colon s\neq \sneu, \pibs(1|s) < p_{\max}  \}.
    \end{align*}
    Now we construct a feasible $\ve{z}$ to the SOCP in \eqref{eq:solve-eta-1-begins}--\eqref{eq:solve-eta-1-ends} as
    \begin{equation}
        z(s) = 
        \begin{cases}
            (x(D, s) + c) / m(D)  \quad & \text{if } s\in S^a\\
            (x(D, s) - d)^+ / m(D) \quad & \text{if } s\in S^b\cup{\sneu},
        \end{cases}
    \end{equation}
    for some $d > 0$ to be determined later, and $c$ is such that $\sum_{s\in\sspa} z(s) = 1$. Note that such $c$ always exists for any $d$. 
    Observe that when $d$ is sufficiently large, $\sum_{s\neq \sneu} \pibs(1|s) \ve{z}(s) = p_{\max} \geq \alpha$. By the intermediate value theorem, we can let $d$ be such that $\sum_{s\neq \sneu} \pibs(1|s) \ve{z}(s) = \alpha$. Then $\ve{z}$ is feasible to the SOCP in \eqref{eq:solve-eta-1-begins}--\eqref{eq:solve-eta-1-ends}. 
    Therefore, by the definition of $\eta$, we must have 
    \begin{equation}
        \norm{\ve{z} - \statdist}_\umat \geq \eta. 
    \end{equation}

    To get a contradiction, we only need to show that 
    \begin{equation}
        \label{eq:proposition-socp-fe:intermediate-goal}
        \norm{x(D) - m(D)\ve{z}}_\umat \leq \errfe = 2\sqrt{2}\lamu^{1/2} \epsilon_0,
    \end{equation}
    which would imply that 
    \begin{align*}
        \norm{x(D) - m(D)\statdist}_\umat 
        &\geq \norm{m(D)\ve{z} - m(D)\statdist}_\umat - \norm{x(D) - m(D)\ve{z}}_\umat \\
        &\geq \eta m(D) - \errfe.
    \end{align*}

    To show \eqref{eq:proposition-socp-fe:intermediate-goal}, observe that 
    \begin{align*}
        \norm{x(D) - m(D)\ve{z}}_2 
        &= \sqrt{\sum_{s\in S^a} \big(x(D, s) - m(D) z(s)\big)^2 + \sum_{s\in S^b\cup \{\sneu\}} \big(x(D, s) - m(D) z(s)\big)^2} \\
        &= \sqrt{ \big(\sum_{s\in S^a} (x(D, s) - m(D) z(s))\big)^2 +  \big(\sum_{s\in S^b\cup \{\sneu\}}x(D, s) - m(D) z(s)\big)^2} \\
        &= \sqrt{2} \sum_{s\in S^b \cup \{\sneu\}} \big(x(D, s) - m(D) z(s)\big). 
    \end{align*}
    To bound $\sum_{s\in S^a} \big(x(D, s) - m(D) z(s)\big)$, observe that 
    \begin{align*}
        \epsilon_0 &\geq \alpha m(D) - \sum_{s\neq \sneu} \pibs(1|s) x(D, s) \\
        &= \sum_{s\neq \sneu} \pibs(1|s) m(D) z(s) - \sum_{s\neq \sneu} \pibs(1|s) x(D, s) \\
        &= \sum_{s\in S^a} \pibs(1|s) (m(D) z(s) - x(D, s)) + \sum_{s\in S^b} \pibs(1|s) (m(D) z(s) - x(D, s)) \\
        &= - p_{\max} \sum_{s\in S^b\cup \{\sneu\}}(m(D) z(s) - x(D, s)) + \sum_{s\in S^b} \pibs(1|s) (m(D) z(s) - x(D, s)) \\
        &= \sum_{s\in S^b} (p_{\max} - \pibs(1|s)) (x(D, s) - m(D) z(s)) + p_{\max} (x(D, \sneu) - m(D) z(\sneu) )\\
        &\geq \frac{1}{2} \sum_{s\in S^b \cup \{\sneu\}} (x(D, s) - m(D)z(s)).
    \end{align*}
    where the last inequality uses the fact that $\pibs(1|s) \in \{0,1/2, 1\}$ if $s\neq \sneu$, so $p_{\max} - \pibs(1|s) \geq 1/2$ for $s\in S^b$ and $p_{\max} \geq 1/2$. Therefore, 
    \[
        \norm{x(D) - m(D)\ve{z}}_2  = \sqrt{2} \sum_{s\in S^b \cup \{\sneu\}} (x(D, s) - m(D)z(s)) \leq 2\sqrt{2}\epsilon_0. 
    \]
    Because $\norm{v}_\umat \leq \lamu^{1/2}\norm{v}_2$ for any $v\in \R^{|\sspa|}$, we have proved \eqref{eq:proposition-socp-fe:intermediate-goal}, which leads to a contradiction. 
    Therefore, \eqref{eq:lp-priority-condition-1-restate} must hold. Similarly, \eqref{eq:lp-priority-condition-2-restate} also holds. 
\end{proof}

Next, we discuss choosing $\eta$ and $\errfe$ when either of the SOCPs is not feasible. We first temporarily let $\eta = \min(\eta_1, \eta_2)$ and $\errfe = 2\sqrt{2} \lamu^{1/2} (|\sempty|+1) / N$, where $\eta_1$ or $\eta_2$ is $+\infty$ if the corresponding SOCP is infeasible. We modify the values of $\eta$ and $\errfe$ if the following two cases happen, and do nothing otherwise.
\begin{itemize}
\item If \eqref{eq:solve-eta-1-begins}--\eqref{eq:solve-eta-1-ends} is not feasible, but $\big(\alpha - \max_{s\neq \sneu} \pibs(1|s)\big) \errfe / \eta \geq (|\sempty|+1) / N$, then we reset $(\eta, \errfe) = \big(|\sspa|^{-1/2} \min\{y^*(\sneu,1), y^*(\sneu,0)\}, \, |\sspa|^{-1/2}(|\sempty|+1) / N \big)$, which is guaranteed to be feasibility ensuring by \Cref{lem:feasibility-ensuring}. 
\item If \eqref{eq:solve-eta-2-begins}--\eqref{eq:solve-eta-2-ends} is not feasible, but $\big(1 - \alpha - \max_{s\neq \sneu} \pibs(0|s)\big) \errfe / \eta \geq (|\sempty|+1) / N$, then we again reset $(\eta, \errfe) = (|\sspa|^{-1/2} \min\{y^*(\sneu,1), y^*(\sneu,0)\}, |\sspa|^{-1/2}(|\sempty|+1) / N)$. 
\end{itemize}
To see why $(\eta, \errfe)$ obtained in this way is feasibility-ensuring, we focus on the situation where neither of the above two cases happen. Without loss of generality, we assume $m(D) \geq \errfe / \eta$, because otherwise \eqref{eq:feasible-ensuring-exp-restate} is always false. 
Suppose \eqref{eq:solve-eta-1-begins}--\eqref{eq:solve-eta-1-ends} is infeasible, but $\big(\alpha - \max_{s\neq \sneu} \pibs(1|s)\big) \errfe / \eta \geq (|\sempty|+1) / N$. Then for $m(D) \geq \errfe / \eta$, 
\begin{align*}
    \alpha m(D) - \sum_{s\neq \sneu} \pibs(1|s) x(D,s) &\geq (\alpha - \max_{s\neq \sneu} \pibs(1|s)) m(D) \\
    &\geq (\alpha - \max_{s\neq \sneu} \pibs(1|s)) \frac{\errfe}{\eta} \\
    &> (|\sempty|+1) / N,  
\end{align*}
where the second inequality utilizes the fact that $\alpha > \max_{s\neq \sneu} \pibs(1|s)$ when \eqref{eq:solve-eta-1-begins}--\eqref{eq:solve-eta-1-ends} is infeasible. Therefore, \eqref{eq:lp-priority-condition-1-restate} holds automatically. 
Then if \eqref{eq:solve-eta-2-begins}--\eqref{eq:solve-eta-2-ends} is also infeasible, we can apply the same argument above; if \eqref{eq:solve-eta-2-begins}--\eqref{eq:solve-eta-2-ends} is feasible, with $(\eta, \errfe) = (\eta_2, 2\sqrt{2} \lamu^{1/2} (|\sempty|+1) / N)$, by the arguments in \Cref{prop:socp-feasible-ensuring}, \eqref{eq:feasible-ensuring-exp-restate} implies \eqref{eq:lp-priority-condition-2-restate}. Therefore, $(\eta, \errfe)$ is feasibility-ensuring.

\subsubsection{Updating $\Db_t$}

With $\eta$ solved, now we implement the step of solving $\Db_t$. 
When $\slk(X_t, [N]) \geq 0$, we simply let $\Db_t = [N]$. Next, we specify the procedure for determining $\Db_t$ in two cases, based on whether $\slk(X_t, \Db_{t-1}) \geq 0$. 

\paragraph{Case 1: When $\slk(X_t, \Db_{t-1}) \geq 0$.}
In this case, we solve the following SOCP parameterized by $\errtol =  (1+2\lamu^{1/2})|\sspa| / N$:
\begin{subequations}
\begin{align}
    \underset{\ve{z},\, m}{\text{maximize}}  \mspace{12mu}& \sums z(s) \\*
    \text{subject to}\mspace{25mu}
    & \norm{\ve{z} - m \statdist}_{\umat} \leq  \eta m -  \frac{|\sempty|+1}{N} - \errtol \label{eq:slack-constraint-socp-restate-2} \\*
    & \sums z(s) = m \\*
    & X_t(\Db_{t-1}, s) \leq z(s) \leq X_t([N], s)  \quad \forall s\in\sspa.
\end{align}
\end{subequations}
Based on whether the SOCP is feasible, we define $\Db_t$ and argue that it satisfies the requirements of \Cref{alg:two-set} in different ways, as specified below. 

If the SOCP is infeasible, we choose $\Db_t = \Db_{t-1}$. To see why $\Db_t$ is $\errtol$-maximal feasible set, observe that $\slk(X_t, \Db_t) = \slk(X_t, \Db_{t-1}) \geq 0$; moreover, the infeasibility of the SOCP implies that there is no $D'\supseteq \Db_t$ such that $\slk(X_t, D') \geq \errtol$. 

If the SOCP is feasible, let the optimal solution solved from this SOCP be $(\ve{z}^*, m^*)$. We let $D_t$ consist of $\floor{Nz^*(s)}$ arms with state $s$ for each state $s\in\sspa$. 
Then we have $X_t(\Db_{t-1}, s) \leq X_t(\Db_{t}, s) \leq X_t([N], s)$, so it is equivalent to letting $\Db_t \supseteq \Db_{t-1}$. 
Next, we show that $\Db_t$ is a $\errtol$-maximal feasible set. 
We first show that $\slk(X_t, \Db_t) \geq 0$. Observe that $\norm{X_t(\Db_t) - \ve{z}^*}_1 \leq |\sspa| / N$, so 
\begin{align*}
    &\abs{\norm{X_t(\Db_t) - m(\Db_t) \statdist}_\umat - \norm{\ve{z}^* - m^*\statdist}_\umat} \\
    &\qquad \leq \norm{(X_t(\Db_t) - \ve{z}^*) - (m(\Db_t) - m^*) \statdist}_\umat \\
    &\qquad \leq \norm{(X_t(\Db_t) - \ve{z}^*)}_\umat + \abs{m(\Db_t) - m^*}\norm{\statdist}_\umat \\
    &\qquad \leq \lamu^{1/2} \norm{(X_t(\Db_t) - \ve{z}^*)}_2 + \lamu^{1/2} \abs{m(\Db_t) - m^*}\norm{\statdist}_2 \\
    &\qquad\leq \lamu^{1/2}\norm{(X_t(\Db_t) - \ve{z}^*)}_1 + \lamu^{1/2}\abs{m(\Db_t) - m^*}\norm{\statdist}_1 \\
    &\qquad \leq \frac{2\lamu^{1/2}|\sspa|}{N}. 
\end{align*}
Therefore, 
\begin{align*}
    \slk(X_t, \Db_t) &= \eta m(\Db_t) - \norm{X_t(\Db_t) - m(\Db_t)\statdist}_{\umat} - \frac{|\sempty|+1}{N}   \\
    &\geq m^* -  \frac{|\sspa|}{N} - \norm{\ve{z}^* - m^*\statdist}_{\umat} - \frac{2\lamu^{1/2}|\sspa|}{N}  - \frac{|\sempty|+1}{N} \\
    &\geq \errtol - \frac{(1+2\lamu^{1/2})|\sspa|}{N} \\
    &= 0. 
\end{align*}
To show the maximality of $\Db_t$, observe that for any $D' \supseteq \Db_t$ satisfying $\slk(X_t, D') \geq \errtol$, $(X_t(D'), m(D'))$ satisfies all the constraints of this SOCP, so $m(D') \leq m^* \leq m(\Db_t)+|\sspa|/N \leq m(\Db_t)+\errtol$.

\paragraph{Caes 2: When $\slk(X_t, \Db_{t-1})<0$.} In this case,  we solve two SOCPs. 
The first SOCP is given by 
\begin{subequations}
\begin{align}
    \underset{\ve{z},\, m}{\text{maximize}}  \mspace{12mu}& \sums z(s) \\*
    \text{subject to}\mspace{25mu}
    & \norm{\ve{z} - m \statdist}_{\umat} \leq  \eta m -  \frac{|\sempty|+1}{N} - \errtol  \label{eq:slack-constraint-socp-restate-3} \\*
    & \sums z(s) = m \\*
    &  0\leq z(s) \leq X_t(\Db_{t-1}, s) \quad \forall s\in\sspa.
\end{align}
\end{subequations}
We define an intermediate set $\Dtemp$ based on the outcome of solving this SOCP. If this SOCP is infeasible, we let $\Dtemp = \emptyset$. 
If the SOCP is feasible, let the optimal solution be $(\ve{z}^*, m^*)$. We let the set $\Dtemp$ consist of $\floor{Nz^*(s)}$ arms with state $s$, for each $s\in\sspa$. 
One can verify that the set can be chosen as a subset of $\Db_{t-1}$ because $z^*(s) \leq X_t(\Db_{t-1}, s)$. 
Also, the constraint \eqref{eq:slack-constraint-socp-restate-3} ensures that $\Dtemp$ satisfies $\slk(X_t, \Dtemp) \geq 0$, following a similar argument as before.

After obtaining $\Dtemp$, we solve the second SOCP given by
\begin{subequations}
\begin{align}
    \underset{\ve{z},\, m}{\text{maximize}}  \mspace{12mu}& \sums z(s) \\*
    \text{subject to}\mspace{25mu}
    & \norm{\ve{z} - m \statdist}_{\umat} \leq  \eta m -  \frac{|\sempty|+1}{N} - \errtol  \label{eq:slack-constraint-socp-restate-4} \\*
    & \sums z(s) = m \\*
    &  X_t(\Dtemp, s) \leq z(s) \leq X_t([N], s) \quad \forall s\in\sspa.
\end{align}
\end{subequations}
If this SOCP is infeasible, we let $\Db_t = \Dtemp$. 
We argue that $\Db_t$ is an $\errtol$-maximal feasible set:
first, because $\slk(X_t, \Dtemp) \geq 0$ or $\Dtemp=\emptyset$, the same holds for $\Db_t$; 
then observe that for any $D'$ such that $D'\supseteq \Db_t = \Dtemp$ and $\slk(X_t, D') \geq \errtol$, $(X_t(D'), m(D'))$ satisfies the constraint of the SOCP, so such $D'$ does not exist by the infeasibility, which implies that $\Dtemp$ is $\errtol$-maximal feasible. 

If this SOCP is feasible, we let its optimal solution be $(\ve{z}^*, m^*)$. We let the set $\Db_t$ consist of $\floor{Nz^*(s)}$ arms with state $s$, for each $s\in\sspa$. 
The constraint \eqref{eq:slack-constraint-socp-restate-4} ensures that $\Db_t$ satisfies $\slk(X_t, \Db_t) \geq 0$, following a similar argument as before. 
Moreover, for any set $D'\supseteq \Db_t\supseteq \Dtemp$ and such that $\slk(X_t, D_t)\geq \errtol$, $(X_t(D'), m(D'))$ satisfies the constraints of this SOCP, so $m(D') \leq m^* \leq m(\Db_t) + |\sspa| / N \leq m(\Db_t) + \errtol$. Therefore, $\Db_t$ is $\errtol$-maximal feasible.

\subsubsection{Deciding actions for the arms in $(\Db_t\cup \Da_t)^c$}
\label{app:exp-details:decide-actions-outside}
In our implementation of the two-set policy, we choose the actions for the arms in $(\Db_t\cup \Da_t)^c$ following similar ideas as the ID policy in \citep{HonXieCheWan_24}. 
Specifically, we generate a random ideal action $\syshat{A}_t(i)$ and temporarily set $A_t(i) = \syshat{A}_t(i)$ for each $i\in (\Db_t\cup \Da_t)^c$. If 
\[
    \sumN A_t(i) > \alpha N, 
\]
we change $A_t(i)$ to $0$ for $i\in(\Db_t\cup \Da_t)^c$ starting from large $i$'s to small $i$'s, until $\sumN A_t(i) = \alpha N$. 
Similarly, if 
\[
    \sumN A_t(i) < \alpha N, 
\]
we change $A_t(i)$ to $1$ for $i\in(\Db_t\cup \Da_t)^c$ starting from large $i$'s to small $i$'s, until $\sumN A_t(i) = \alpha N$. As guaranteed by \Cref{lem:two-set:subroutine-conform}, we can always achieve $\sumN A_t(i) = \alpha N$ through this process.

\subsection{Details of RB instances simulated in \hcref{sec:experiments}}
\label{app:exp-details:examples}

\subsubsection{Definitions of the instances in \hcref{fig:constructed-examples}}

Here, we provide some details on the RB instances simulated in \Cref{fig:new2-eight-state-045} and \Cref{fig:conveyor-nd}. 

The RB instance corresponding to \Cref{fig:new2-eight-state-045} is defined in \citep{HonXieCheWan_24}. We repeat its definition below for completeness. 
This RB isntance is defined by a single-armed MDP with state space $\sspa=\{0,1,\ldots,7\}$. 
Each state in $\sspa$ is labelled by a \textit{preferred action}, which is chosen to be action $1$ for states $\{0, 1, 2, 3\}$ and action $0$ for states $\{4, 5, 6, 7\}$. 
If an arm is in state~$s$ and takes the preferred action, it moves to state $(s+1) \bmod 8$ with probability $p_{s,\rightsub}$, and stays in state $s$ otherwise; if it does not take the preferred action, it moves to state $(s-1)^+$ with probability $p_{s,\leftsub}$;
here the probabilities $p_{s, \rightsub} = 0.1$ for all $s\in\sspa$, and $p_{s, \leftsub}$ is defined as
\begin{align*}
     p_{s, \leftsub} =
     \begin{cases}
         1 \quad &  \text{ if } s =0,1 \\
         0.5 - 0.1 s / 8 \quad & \text{ otherwise}.
     \end{cases}
\end{align*}
The reward function is defined as $r(7, 0) = 0.1$, $r(0,1)=1/300$, and $r(s,a)=0$ for all the other  $s\in\sspa,a\in\aspa$.  
The budget parameter $\alpha = 0.45$. 
We initialize the states to be uniformly distributed over $\{4,5,6,7\}$. 

Here we briefly comment on the rationale of constructing this example. The main purpose of this transition structure is to let the LP index policy use a suboptimal priority order. 
Specifically, the LP index policy in this example prioritizes state $3$ over state $5$, and state $5$ over state $4$. 
Thus, if at a certain time step, all arms are in states $\{3,4,5\}$ with a suitable distribution, all arms in state $3$ and $5$ will be activated. Consequently, the arms in $3$ or $5$ will return to state $4$, so all arms are stuck in these three states without getting any reward. 

Apart from the transition structure, the value of the reward function $r(0,1)$ is also specially chosen to make the problem non-indexable. The budget $\alpha$ is chosen to be less than $0.5$, so that state $0$ is the neutral state.

The RB instance in \Cref{fig:conveyor-nd} is constructed with similar ideas. The single-armed MDP of this instance has the state space $\sspa = \{0,1,2, \dots, 12\}$; the preferred actions are $1$ for states $\{0,1,\dots, 5\}$, and are $0$ for the rest of the states. The transition rules after taking the preferred (non-preferred) action are the same as the last example, with $p_{\cdot, L}$ adjusted to 
\[
     p_{s, \leftsub} =
     \begin{cases}
         1 \quad &  \text{ if } s =0,1 \\
         0.5 - 0.1 s / 12 \quad & \text{ otherwise}.
     \end{cases}
\]
We let the reward function be defined as $r(7, 0) = 0.1$ and $r(0,1)=1/300$, and $r(s,a)=0$ for all the other  $s\in\sspa,a\in\aspa$.  
We let the budget parameter $\alpha = 0.4$. 
We initialize the states to be uniformly distributed over $\{6,7,8,9,10,11\}$. 
In this example, the LP index policy prioritizes state $5$ over state $7$, and state $7$ over state $6$, so intuitively, the arms could be stuck on these three states. 

One can potentially generalize the above two examples, and construct more RB instances that satisfy Assumptions~\ref{assump:aperiodic-unichain}--\ref{assump:local-stability}, but violate UGAP under the Whittle index policy and the LP index policy.

\subsubsection{Details of generating uniformly random examples}

For the instances in \Cref{fig:uniform-1,fig:uniform-6-and-0}, all vectors $P(s,a,\cdot)$ and $r(\cdot,a)$ are sampled independently and uniformly from $\simplex(\sspa)$.

\end{document}